\documentclass[a4paper]{article}[12pts]

\usepackage[english]{babel}
\usepackage[utf8x]{inputenc}
\usepackage[T1]{fontenc}

\normalsize

\usepackage[a4paper,top=1in,bottom=1in,left=1in,right=1in,marginparwidth=1in]{geometry}

\usepackage{natbib}
\usepackage{setspace}
\usepackage{amsmath}
\usepackage{amssymb} 
\usepackage{amsthm}
\usepackage{amsfonts, mathtools}
\usepackage{algorithm,algpseudocode}
\usepackage{bm}
\usepackage{bbm}
\usepackage{comment}
\usepackage{graphicx}
\usepackage{multirow, booktabs}
\usepackage{caption}
\usepackage{subcaption}
\usepackage[colorinlistoftodos]{todonotes}
\usepackage{hyperref}
\hypersetup{colorlinks=true, allcolors=blue}
\usepackage{enumerate}
\usepackage{url}
\usepackage{framed}
\usepackage{multirow}
\usepackage{textgreek}
\usepackage{adjustbox}
\usepackage{tcolorbox}     

\usepackage{multirow}
\usepackage{textgreek}
\usepackage{adjustbox}

\theoremstyle{plain}
\newtheorem{theorem}{Theorem}[section]
\newtheorem{proposition}[theorem]{Proposition}
\newtheorem{lemma}[theorem]{Lemma}

\theoremstyle{definition}
\newtheorem{definition}[theorem]{Definition}
\newtheorem{assumption}[theorem]{Assumption}
\newtheorem{example}[theorem]{Example}
\theoremstyle{remark}
\newtheorem{remark}[theorem]{Remark}

\newcommand{\shang}[1]{{\color{red} [Shang: {#1}]}} 
\definecolor{rose}{rgb}{1.0, 0.33, 0.64}
\newcommand{\zz}[1]{{\color{rose} [ZZ: {#1}]}}
\renewcommand{\thefootnote}{\fnsymbol{footnote}}
\DeclareMathOperator*{\argmax}{arg\,max}
\DeclareMathOperator*{\argmin}{arg\,min}

\title{Towards Better Statistical Understanding of Watermarking LLMs}

\author{Zhongze Cai$^{\dagger *}$ \ \  Shang Liu$^{\dagger *}$ \ \   Hanzhao Wang$^{\ddagger *}$ \  \ Huaiyang Zhong$^{\diamond}$ \ \ \ Xiaocheng Li$^{\dagger}$}
\date{\small
$^{\dagger}$ Imperial College Business School, Imperial College London\\
$^{\ddagger}$ The University of Sydney Business School, The University of Sydney\\
$^{\diamond}$ Grado Department of Industrial and Systems Engineering, Virginia Tech }

\begin{document}
\maketitle
\onehalfspacing

\def\thefootnote{*}\relax\footnotetext{Equal contribution.}

\begin{abstract}
In this paper, we study the problem of watermarking large language models (LLMs). We consider the trade-off between model distortion and detection ability and formulate it as a constrained optimization problem based on the red-green list watermarking algorithm. We show that the optimal solution to the optimization problem enjoys a nice analytical property which provides a better understanding and inspires the algorithm design for the watermarking process. We develop an online dual gradient ascent watermarking algorithm in light of this optimization formulation and prove its asymptotic Pareto optimality between model distortion and detection ability. Such a result guarantees an averaged increased green list probability and henceforth detection ability explicitly (in contrast to previous results). Moreover, we provide a systematic discussion on the choice of the model distortion metrics for the watermarking problem. We justify our choice of KL divergence and present issues with the existing criteria of ``distortion-free'' and perplexity. Finally, we empirically evaluate our algorithms on extensive datasets against benchmark algorithms.
\end{abstract}

\section{Introduction}
As the ability of large language models (LLMs) evolves rapidly, their applications have raised concerns about LLMs' abuse, such as launching bots on social media, creating fake news, and cheating on writing school essays. The synthetic data created by the LLMs rather than real humans is also polluting the data pool and should be detected and removed to create a high-quality dataset before training \citep{radford2023robust}. Numerous attempts have been made to make the detection possible, which can mainly be classified into two categories: post hoc detection that does not modify the language model and the watermarking that encodes information in the modified content.

Post hoc detection aims to train models that label the texts without monitoring the generation process using statistical features of the LLMs. For example, when being inspected by another LLM, the statistical properties of machine-generated texts deviate from the human-generated ones in aspects such as the distributions of token log-likelihoods \citep{gehrmann2019gltr, ippolito2019automatic, zellers2019defending, solaiman2019release, tian2023gptzero, mitchell2023detectgpt}. However, post hoc ways usually rely on the fundamental assumption that machine-generated texts deviate from human-generated texts, which could be challenged in two ways. The first threat is that LLMs can mimic humans better as time moves on, leading to a situation where LLMs are more indistinguishable from humans \citep{sadasivan2023can}. Although the detection is possible when LLMs are still imperfect in learning \citep{chakraborty2023possibilities}, such a cat-and-mouse game makes the detection less reliable with new models. Another drawback is that there might not be a universal statistical property of human speech across different groups of people. For instance, \citet{liang2023gpt} points out that GPT detectors are biased against non-native speakers.

It has reached a consensus that the usage of LLMs should be regulated and companies have agreed to add restrictions \citep{paul2023top}. The watermarking scheme has become not only an academic topic \citep{aaronson2023watermarking, christ2023undetectable, kirchenbauer2023watermark, sadasivan2023can} but also a practical solution. As the transformer-based language models \citep{vaswani2017attention} have been widely adopted, the next-token-prediction generating process enables researchers to develop watermarking methods by monitoring the next-token distribution. Existing methods \citep{aaronson2023watermarking, christ2023undetectable, kuditipudi2023robust, kirchenbauer2023watermark, zhao2023provable, fernandez2023three} mainly distort the distribution via a (sequence of) random variable(s) called key(s), and the detector uses the key(s) to examine the statistics of the texts. In short, watermarking algorithms twist the distribution of texts and obtain some level of detection ability.

Among all the existing watermarking algorithms, the red-green list watermarking introduced by \citet{kirchenbauer2023watermark} has been very popular, but there is no consensus on analyzing the trade-off between the model distortion and the detection ability. For example, few characterize the \textit{price} paid to obtain such a watermark -- there is no common definition of the model distortion. For instance, \citet{aaronson2023watermarking, kuditipudi2023robust, hu2023unbiased, wu2024resilient} argue that the ``distortion-free'' criteria should be considered, and suggest the model distortion can be avoided. Some other papers measure the model distortion differently: current empirical validation mostly uses the difference of log-perplexity to measure the model distortion \citep{kirchenbauer2023watermark}, while some use the Kullback-Leibler divergence \citep{zhao2023provable}. Besides, most existing algorithms decide their hyperparameters in a heuristic way \citep{kirchenbauer2023watermark}. In this paper, we try to improve the understanding of the watermarking algorithms by answering the following questions: \textit{Statistically, what is the minimal price for the red-green list watermarking algorithm to attain a certain level of detection ability? Furthermore, what should be considered a good measurement for the price?}

In summary, we address our main contributions as follows: In Section \ref{sec:Pareto_optimal}, we study the \textit{price} paid when applying a generalized version of the soft red-green list watermarking scheme introduced in \citet{kirchenbauer2023watermark}. We formulate the trade-off between the model distortion and the detection ability as a constrained optimization problem that minimizes the Kullback-Leibler divergence of the watermarked model towards the original model subject to lower bounds on the average increased green list probability (as a proxy for the detection ability). Such a target is justified by the natural decomposition of the total KL divergence into the sums of the step-wise KL divergence (Proposition \ref{prop:KL_decompose}). We curve the trade-off and develop an online optimization watermarking algorithm (Algorithm \ref{alg:DualGA}) as an asymptotic Pareto optimum under mild assumptions (Assumption \ref{assum:iid}) as is shown in Theorem \ref{thm:dualGA_performance_iid}. In Section \ref{sec:KL_DISTORT}, we justify our choice of Kullback-Leibler divergence as the model distortion. We show that any watermarked model close to the original model is hardly detectable via an information-theoretic lower bound (Proposition \ref{prop:info_lower_bound}). We also show in Proposition \ref{prop:KL_cvx} that the ``distortion-free'' criterion focuses on ``the distance of the expected model from the original one'' rather than ``the expected distance of the model from the original one''. Such results suggest that the marginal distortion-free criterion should be reconsidered carefully. Later, we justify the choice of KL divergence against the difference of log-perplexity, where the latter could be smaller than zero under distortion. In Section \ref{sec:exp}, we present numerical experiments. We compare our main algorithm Dual Gradient Ascent (Algorithm \ref{alg:DualGA}, abbreviated as DualGA) to existing algorithms on benchmark datasets, validating that it achieves the optimal trade-off between the model distortion and the detection ability and ensures the detectable watermark individually for every prompt.

\subsection{Model Setup}
Language models (LMs) describe the probabilities of a sequence of words. The currently mainstream LMs \citep{radford2019language,brown2020language} work autoregressively by predicting the next token based on a given prompt and all previous tokens. An LM $\mathcal{M}$ is a function that treats both a prompt $x$ and previous $t-1$ tokens $y_{1}, \dots, y_{t-1}$ as \textit{context} (abbreviated as $\bm{y}_{[t-1]}$) and maps the context to a probability vector $\bm{p}_t \in \Delta(\mathcal{V})$, where $\Delta(\mathcal{V})$ denotes the set of probability distributions over the vocabulary $\mathcal{V}$. To obtain $\bm{p}_t$, a typical LM first generates a logit vector $\bm{l}_t=(l_{t,1},...,l_{t,|\mathcal{V}|})$ and applies a softmax layer to output
$p_{t,k} \coloneqq \frac{\exp(l_{t,k})}  {\sum_{k'\in \mathcal{V}} \exp(l_{t,k'})}$. 
The next token $y_t$ is sampled via $\bm{p}_t=(p_{t,1},...,p_{t,|\mathcal{V}|})^\top$. The LM repeats this procedure until a ``terminate'' symbol $v_{\text{term}}\in\mathcal{V}$ is sampled or the token sequence reaches a pre-specified maximum length $T_{\text{max}}$. It does not hurt to assume that after receiving the first $v_{\text{term}}$, the LM continues outputting the same $v_{\text{term}}$ until the maximum length is reached. We denote the whole generating process by the map $\bm{p}_{[T_{\text{max}}]}: \mathcal{X} \rightarrow \Delta(\mathcal{V}^{T_{\text{max}}})$, where $\mathcal{X}$ is the prompt space. For simplicity, we omit the subscript and use $\bm{p}$ to denote the whole mapping/probability model when there is no ambiguity.

We denote the marginal distribution of $\bm{p}$ that generates the first $t$ tokens by $\bm{p}_{[t]}(\cdot)$ and the distribution of the $t$-th token conditioned on the previous $t-1$ tokens by $\bm{p}_{t|[t-1]}(\cdot)$. 

The watermarking of LMs is a procedure that develops a \textit{watermarked} language model accompanied by a detection algorithm. A watermarked model contains an algorithm $\mathcal{A}$ and a key $\mathcal{K}$ to be a mapping from the prompt $x$ and previous tokens $\bm{y}_{[t-1]}$ to another probability vector $\bm{q}_t^{(\mathcal{A}, \mathcal{K})}$ over $\mathcal{V}$. The watermarked LM also works autoregressively but based on modified distributions. In the following, we use $\bm{q}$ to represent the probability vectors of watermarked models and sometimes omit the superscript $(\mathcal{A},\mathcal{K})$ for simplicity. A detection procedure $\mathcal{D}$ receives the detection key $\mathcal{K}$ and a token sequence as inputs and then gives a binary output as a prediction of whether the token sequence is generated from the watermarked LM, i.e., generated by $\bm{q}^{(\mathcal{A}, \mathcal{K})}$ (1 for yes and 0 for no). The two types of errors are:
\[
\alpha_{\bm{p}}(\mathcal{D}, x) \coloneqq \mathbb{P}\left(\mathcal{D}(\mathcal{K}, \bm{y}_{[T_{\text{max}}]}) = 1 \,\big| \,\bm{y}_{[T_{\text{max}}]} \sim \bm{p}(x)\right).
\]
\[
\beta(\mathcal{D}, x) \coloneqq \mathbb{P}\left(\mathcal{D}(\mathcal{K}, \bm{y}_{[T_{\text{max}}]}) = 0 \,\big| \,\bm{y}_{[T_{\text{max}}]} \sim \bm{q}^{(\mathcal{A}, \mathcal{K})}(x)\right).
\]
Here $\alpha_{\bm{p}}(\mathcal{D}, x)$ is the type I error that describes the probability of predicting a normal sequence as generated by the watermarked LM. The conditional part requires that the sequence $\bm{y}_{[T_{\text{max}}]}$ is generated by $\bm{p}$, which can be based on a human, the unwatermarked LM, or other watermarked LMs. $\beta(\mathcal{D}, x)$ is the type II error that describes the probability of predicting an LM-generated sequence as not generated by the watermarked LM. We remark that the two errors are affected by both the watermarking procedure/algorithm and the detection algorithm.

\section{Trade-off between Model Distortion and Detection Ability}
\label{sec:Pareto_optimal}

The two natural desiderata for a watermarking algorithm are (i) the watermarked LM stays close to the unwatermarked LM and (ii) the watermarked text can be easily detected. We aim to provide a precise characterization of the trade-off between these two aspects. Specifically, we adopt the Kullback-Leibler divergence to measure the closeness between the two LMs (i.e., the extent of model distortion caused by watermarking) and we consider the two types of errors defined earlier as a measurement of the detection ability. To begin with, we state a generalized version of the soft watermarking scheme proposed by \cite{kirchenbauer2023watermark}. The theoretical analysis, on the one hand, shows a nice structure of the algorithm, and on the other hand, reveals potential issues of the algorithm.

\begin{algorithm}[ht!]
\caption{Generalized soft watermarking \citep{kirchenbauer2023watermark}}
    \label{alg:kirchen}
    \begin{algorithmic}[1] 
    \Require Original LM $\bm{p}$, prompt $x$,  green list ratio $\gamma$, pseudorandom function $f$, perturbations $\{\delta_{t,k}>0\}_{t\in [T_{\text{max}}],k\in \mathcal{V}}$, random seed $\mathcal{K}$
    \Ensure $y_1, \dots, y_{T_{\text{max}}}$
    \State Initialize $t\gets1$, $y_0\gets $ the last token of $x$
    \While{$t \leq T_{\text{max}}$ and $y_{t-1} \neq v_{\text{term}}$}
        
        \State Input the seed $\mathcal{K}$ and (a  hash of) $y_{t-1}$ into a pseudorandom function $f$ to randomly partition the vocabulary into a green list with size $\gamma |\mathcal{V}|$ and a red list with size $(1-\gamma) |\mathcal{V}|$
        \State\label{line:Original_Logit_Vector} Compute the logit vector generated by the original LM $\bm{p}$ and modify it to produce a probability vector
        \[
        q_{t,k} \gets 
        \frac{\exp(l_{t,k} + \delta_{t,k} \cdot \mathbbm{1}\{k \text{ green at }t\})}{\sum_{k^\prime=1}^{|\mathcal{V}|} \exp(l_{t,k^\prime} + \delta_{t,k'} \cdot \mathbbm{1}\{k^\prime \text{ green at }t\})}
        \]
        \State Sample a token
        \[y_t \sim \left(q_{t,1},\ldots, q_{t,|\mathcal{V}|}\right),\] 
        \State $t\gets t+1$
    \EndWhile
    \If{$t<T_{\text{max}}$}
        \State Set all remaining $y_{t+1}, \dots, y_{T_{\text{max}}}$ to be $v_{\text{term}}$
    \EndIf
    \end{algorithmic}
\end{algorithm}


Algorithm \ref{alg:kirchen} states a generalized soft watermarking algorithm of \cite{kirchenbauer2023watermark}. At each step $t$, the algorithm randomly partitions the vocabulary into green words (with ratio $\gamma$) and red words (with ratio $1-\gamma$). The partition is based on the previous word $y_{t-1}$ and some random function $f$. Then, the watermarks are inserted into the generated token sequence by increasing the probabilities of the green words. Equivalently, one adds a positive number $\delta_{t,k}$ to the logit value where $t$ is the index of the token location in the sequence, and $k$ is the word index in the vocabulary. The algorithm is stated in a more general way than \cite{kirchenbauer2023watermark} in that it allows different values of $\delta_{t,k}$ while \cite{kirchenbauer2023watermark} keep all of them constant, i.e., $\delta_{t,k}\equiv \delta$.

\subsection{Constrained Optimization Formulation}


The design of the algorithm increases the probability of green words appearing in the generated text, which is vital for the detection. Specifically, the detection algorithm uses the key $\mathcal{K}$ to recover the green list and the red list at each token and thus can identify whether each token in the sequence is green or red. If there are significantly more green tokens than red ones over a (sub)sequence of tokens, it indicates a watermark. In this light, the following quantity is an important indicator of the intensity of the watermarking and is closely related to the detection ability. As noted earlier, we use $\bm{q}$ to denote the watermarked LM and $\bm{p}$ to denote the original LM and we define \textit{the difference of the green word probability} (DG) at the $t$-th token by
\begin{equation}\label{eq:define_DG}
\mathrm{DG}_t(\bm{q}_t) \coloneqq \sum_{k = \text{green},  k \in \mathcal{V}} q_{t,k} - \sum_{k = \text{green},  k \in \mathcal{V}} p_{t,k}.
\end{equation}
The quantity DG measures the change in terms of the green word probability from the watermarked LM to the original LM. The larger the value of DG, the easier the generated text can be detected. We note that this quantity is a random variable that relies on the context up to token $t-1$, i.e. the prompt $x$ and $y_{[t-1]}$. However, we note that the value of DG can be fully controlled by the parameters $\delta_{t,k}$'s as we can see from Step \ref{line:Original_Logit_Vector} of Algorithm \ref{alg:kirchen}.

To quantify how far a watermarked distribution deviates from the
original model, we use the \emph{Kullback--Leibler} (KL) divergence defined as follows.
\begin{definition}[KL divergence]\label{def:KL}
Let \(\mathcal{Q}\) and \(\mathcal{P}\) be probability measures on the
same measurable space with \(\mathcal{Q}\!\ll\!\mathcal{P}\).
Their \emph{Kullback--Leibler divergence} is
\[
  D_{\mathrm{kl}}\!\bigl(\mathcal{Q}\,\|\,\mathcal{P}\bigr)
  \;:=\;
  \int_{\Omega}
  \log\!\Bigl(
      \frac{\mathrm{d}\mathcal{Q}}{\mathrm{d}\mathcal{P}}
  \Bigr)\,
  \mathrm{d}\mathcal{Q}.
\]
A value of \(D_{\mathrm{kl}}=0\) indicates \(\mathcal{Q}=\mathcal{P}\);
larger values correspond to greater distortion.
\end{definition}
In the definition, we use the notation 
\(
  \mathcal{Q}\!\ll\!\mathcal{P}
\)
to denote that  the measure \(\mathcal{Q}\) is
\emph{absolutely continuous} with respect to \(\mathcal{P}\):  
 for any measurable set \(A\), if \(\mathcal{P}(A)=0\)
then \(\mathcal{Q}(A)=0\).
By the Radon-Nikodým theorem, this condition guarantees the existence of a density function
\(\frac{\mathrm{d}\mathcal{Q}}{\mathrm{d}\mathcal{P}}\) such that
\(\mathrm{d}\mathcal{Q}
     =\bigl(\frac{\mathrm{d}\mathcal{Q}}{\mathrm{d}\mathcal{P}}\bigr)
      \mathrm{d}\mathcal{P}\).

At each step $t$, the original LM generates the token with the distribution $\bm{p}_t$, and the watermarked LM generates the token with the distribution $\bm{q}_t$. The KL divergence between these two quantifies the extent to which the watermarking algorithm twists the $t$-th token. As the key interest is always measuring the distance against the original LM $\bm{p}_t$, we abbreviate and define
$D_{\mathrm{kl},t}(\bm{q}_t)\coloneqq D_{\mathrm{kl}}(\bm{q}_{t} \| \bm{p}_{t}).$

We note from Algorithm \ref{alg:kirchen} and the discussion above that both the difference of green word probability DG and the KL divergence can be fully controlled by the perturbs $\delta_{t,k}$'s under the generalized soft watermarking scheme. Thus we can rewrite both quantities as a function of $\delta_{t,k}$'s. That is, we denote the DG function and KL divergence by $\mathrm{DG}_t(\delta_{t,1},\dots, \delta_{t,|\mathcal{V}|})$ and
$D_{\mathrm{kl},t}(\delta_{t,1},\dots, \delta_{t,|\mathcal{V}|})$.

Consider the following constrained optimization problem:
\begin{align}
\mathrm{OPT}(\Delta)\coloneqq \min_{\delta_{t,k}} \quad \frac1T \sum_{t=1}^T D_{\mathrm{kl},t}(\delta_{t,1},\dots,\delta_{t,|\mathcal{V}|})\quad
\text{s.t. } \frac1T \sum_{t=1}^T \mathrm{DG}_t(\delta_{t,1},\dots,\delta_{t,|\mathcal{V}|}) \geq \Delta,
\label{opt:free_t_k}
\end{align}
where the decision variables are $\delta_{k,t}$'s. The right-hand side of the constraint $\Delta>0$ imposes a condition that the average change of the green word probability should exceed a threshold. We will see shortly that the quantity $\Delta$ is closely related to the detection ability. In this light, the optimization problem searches from a minimal twisted LM that achieves a certain level of detection ability. It explicitly trade-offs the two desiderata we mentioned at the beginning of this section. In particular, if we extend the domain  $\delta_{t,k} \in [-\infty, \infty]$, the optimization formulation also covers the hard watermarking scheme \citep{kirchenbauer2023watermark}.

Two closely related optimization problems are 
\begin{align}
\mathrm{OPT}_{\text{static-}k}(\Delta) \coloneqq \min_{\delta_{t}} \quad \frac1T \sum_{t=1}^T D_{\mathrm{kl},t}(\delta_{t})\quad
\text{s.t. } \frac1T \sum_{t=1}^T \mathrm{DG}_t(\delta_{t}) \geq \Delta,
\label{opt:free_t}
\end{align}
where we restrict $\delta_{t,1} = \dots = \delta_{t,|\mathcal{V}|} = \delta_t$ for all $t$ and abbreviate $D_{\mathrm{kl},t}(\delta_{t},\dots,\delta_{t})$ and $\mathrm{DG}_t(\delta_{t},...,\delta_{t})$ as $D_{\mathrm{kl},t}(\delta_{t})$ and $\mathrm{DG}_t(\delta_{t})$, respectively.
\begin{align}
\mathrm{OPT}_{\text{static-}t,k}(\Delta) \coloneqq \min_{\delta} \quad \frac1T \sum_{t=1}^T D_{\mathrm{kl},t}(\delta)\quad
\text{s.t. } \frac1T \sum_{t=1}^T \mathrm{DG}_t(\delta) \geq \Delta.
\label{opt:nofree}
\end{align}
where we restrict $\delta_{t,k}=\delta$ for all $t$ and $k$. 

These optimization problems also correspond to different versions of the algorithm. For example, \eqref{opt:free_t_k} corresponds to Algorithm \ref{alg:kirchen} while \eqref{opt:nofree} corresponds to the original soft watermarking algorithm in \cite{kirchenbauer2023watermark}.


It is easy to see that
$\mathrm{OPT}(\Delta) \leq \mathrm{OPT}_{\text{static }k}(\Delta) \leq \mathrm{OPT}_{\text{static }t,k}(\Delta).$
In the next, we will show in Proposition \ref{prop:free_k} and Proposition \ref{prop:free_t} that $\mathrm{OPT}(\Delta) = \mathrm{OPT}_{\text{static }k}(\Delta) = \mathrm{OPT}_{\text{static }t,k}(\Delta)$
and discuss the implications.

\subsection{Analytical Properties of the Optimization Problem}
\label{subsec:Pareto_Optimality}

In the above optimization problems, the objective function calculates the single-step KL divergence for each step and takes the summation. Indeed, we show that this also corresponds to the KL divergence between the two LMs $\bm{p}$ and $\bm{q}$. 

\begin{definition}
\label{def:conditional_KL}
$\mathcal{Q}$ and $\mathcal{P}$ are two (joint) distributions of $(U, V)$ on the space $\mathcal{U} \times \mathcal{V}$. The conditional KL divergence is defined by $D_{\mathrm{kl}}(\mathcal{Q}_{V|U} \| \mathcal{P}_{V|U} | \mathcal{Q}_{U}) \coloneqq \mathbb{E}_{u \sim \mathcal{Q}_U}\big[D_{\mathrm{kl}}(\mathcal{Q}_{V|U=u} \| \mathcal{P}_{V|U=u})\big]$,
where the expectation is taken with respect to $u \sim \mathcal{Q}_U$.
\end{definition}

This definition gives us a tool to represent the KL divergence by the summation of single-step KL divergences as the following proposition.

\begin{proposition}
\label{prop:KL_decompose}
For an LM $\bm{p}$, prompt $x \in \mathcal{X}$ and the watermarked LM $\bm{q}$ watermarked by algorithm $\mathcal{A}$ and key $\mathcal{K}$, we have $D_{\mathrm{kl}}\Big(\bm{q}(x) \big\| \bm{p}(x)\Big)  
 = \sum_{t=1}^{T_{\text{max}}} D_{\mathrm{kl}}\Big(\bm{q}_{t|[t-1]}(x) \big\| \bm{p}_{t|[t-1]}(x)\Big| \bm{q}_{[t-1]}(x)\Big)$.
\end{proposition}

The left-hand side treats the two LMs as distributions over the space of sequences of $T_{\max}$ tokens and measures their KL divergence. The right-hand side is decomposed into a summation where each term is the single-step conditional KL divergence. In the context of watermarking, the left-hand side captures the total amount of distortion between the original LM and the watermarked LM. And the proposition states that this total distortion is equal to the summation of token-wise distortion at each token/time step. Thus the optimization objectives are related to the proposition in the following way
\begin{align*}
   \mathbb{E}\left[\sum_{t=1}^{T_{max}} D_{\mathrm{kl},t}(\delta_{t,1},\dots,\delta_{t,|\mathcal{V}|})\right]=  \sum_{t=1}^{T_{max}} D_{\mathrm{kl}}\Big(\bm{q}_{t|[t-1]}(x) \big\| \bm{p}_{t|[t-1]}(x)\big| \bm{q}_{[t-1]}(x)\Big) = D_{\mathrm{kl}}\Big(\bm{q}(x) \big\| \bm{p}(x)\Big).
\end{align*}
To see this, the objective function is a realization of the conditional KL divergence which replaces the expectation with respect to $\bm{q}_{[t-1]}(x)$ for the first $t-1$ tokens with the realized sequence. In other words, the objective function for the optimization problems can be seen as an unbiased estimator of the KL divergence between the original LM and the watermarked LM.


We remark that this decomposition relationship does not hold generally for other divergence measures between two distributions. Apart from this theoretical structure, the optimization problem's objective function also has a nice analytical form: the single-step KL divergence can be expressed in a closed form in terms of $\delta_{k,t}$'s. This enables us to derive the following results on the optimal solution.


\begin{proposition}
\label{prop:free_k}
Suppose the problem \eqref{opt:free_t_k} is feasible. The optimal solution satisfies
$
\delta_{t,1}^* = \cdots = \delta_{t,|\mathcal{V}|}^*$ for every $t \in [T].$ 
In other words, the optimal solution of \eqref{opt:free_t_k} shares the same form as that of \eqref{opt:free_t}.
\end{proposition}

\begin{proposition}
\label{prop:free_t}
Suppose the problem \eqref{opt:free_t} is feasible. The optimal solution of \eqref{opt:free_t} satisfies
$\delta_{1}^* = \cdots = \delta_{T}^*.$
In other words, the optimal solution of \eqref{opt:free_t} shares the same form as the program \eqref{opt:nofree}.
\end{proposition}

These two propositions state that under mild feasibility assumptions, the three optimization problems \eqref{opt:free_t_k}, \eqref{opt:free_t}, and \eqref{opt:nofree} share the same optimal solution. From an algorithm viewpoint, it means the generalization of Algorithm \ref{alg:kirchen} over the original algorithm of \cite{kirchenbauer2023watermark} does not bring additional benefits. We make several remarks on the implications for watermarking LLMs: 
\begin{itemize}
\item The optimal value of $\delta$ for the optimization problems above depends on the prompt $x$, the realized sequence $\bm{y}_{[T_{\text{max}}]}$, and the original LM $\bm{p},$ precisely, on the conditional distribution $\bm{p}(\cdot|x)$. For example, if we want to achieve the same level of detection ability (as the constraint of the optimization problems) for two prompts: (i) ``Where is the capital of U.K.?'' and (ii) ``What is your favorite color?'', we should use different $\delta$. For different prompts, i.e., for different $\bm{p}(\cdot|x)$, it corresponds to different optimal solution $\delta^*$ for the optimization problem. Reversely, if we use a fixed $\delta$ for all the generations as the original design in \citep{kirchenbauer2023watermark}, it will result in a different level of detection ability and KL-divergence (from the original LM) for every prompt.
\item The optimization problems above can only be solved in an online fashion as it depends on the roll-out of the sequence. When we implement the watermarking algorithm, it generates the tokens one by one, and at time $t$, we cannot foresee the future terms (from $t+1$ to $T$) in the optimization problems, which disallows solving them in an offline manner. 
\item The proposition above states the structure of the optimal solution of the optimization problems. If we view minimizing model distortion and maximizing detection ability as a multi-objective optimization problem, one subtle point is that the fixed-$\delta$ version of Algorithm \ref{alg:kirchen} \citep{kirchenbauer2023watermark} does not result in a \textit{Pareto optimum}. This is because although the optimal $\delta^*$ remains fixed over time, it depends on the prompt, the roll-out of the sequence, the LM, and also the green word ratio $\gamma$. Even if we forget about the aspects of the prompt and the LM, it still requires careful coordination between the values of $\gamma$ and $\delta$ to make the resultant generation on the Pareto optimal curve. Our numerical experiments also verify this point. 
\end{itemize}

These discussions motivate our development of an online algorithm for the problem which enables a uniform level of detection ability or 
KL-divergence across all the generations. As we will see later, it also brings other benefits to the problem.

\subsection{Online Algorithm with Adaptive $\delta$}

Motivated by the discussions above, we develop an algorithm that not only ensures Pareto optimality (of model distortion and detection ability) but also achieves a pre-specified detection ability ($\Delta$ in \eqref{opt:nofree}) for every generated token sequence. This requires (i) an adaptive choice of $\delta$ based on the prompt, the LM, and $\gamma$ and (ii) an online implementation that learns the optimal $\delta$ on the fly.

Algorithm \ref{alg:DualGA} gives our online algorithm for adaptive choosing $\delta$. For the optimization problem \eqref{opt:nofree}, the optimal solution $\delta^*$ depends on (i) the prompt $x$, (ii) the LM $\bm{p}$, and (iii) the realized generation of the tokens. Algorithm \ref{alg:DualGA} uses a different $\delta_t$ over time, and ideally, it aims to have $\delta_t\rightarrow \delta^*$ quickly. In this way, it learns the optimal $\delta^*$ online and adaptively (to the prompt). The idea of Algorithm \ref{alg:DualGA} is to perform an online gradient ascent algorithm on the Lagrangian dual function for the dual variable $\lambda_t$, and uses $\lambda_t$ to approximate the optimal dual variable $\lambda^*.$ To see this, we first introduce  the Lagrangian of the optimization problem \eqref{opt:nofree} as follows where the dual variable is denoted by $\lambda \geq 0$,
\begin{align}
L(\delta, \lambda) \coloneqq & \frac1T \sum_{t=1}^T D_{\mathrm{kl},t}(\delta)- \lambda \cdot\left(\frac1T \sum_{t=1}^T \mathrm{DG}_t(\delta) - \Delta\right) \nonumber \\ 
= & \frac1T \sum_{t=1}^T D_{\mathrm{kl},t}(\delta)- \frac1T \sum_{t=1}^T \lambda \cdot \left( \mathrm{DG}_t(\delta) - \Delta\right).
\label{eq:Lagrangian}
\end{align}
Denote the corresponding primal function as
$
f(\delta) \coloneqq \sup_{\lambda \geq 0} \ L(\delta, \lambda),
$
and the dual function as
$
g(\lambda) \coloneqq \inf_{\delta} \ L(\delta, \lambda).
$
The following lemma states some key properties of the optimization problem \eqref{opt:nofree} (which also applies to the other two optimization problems \eqref{opt:free_t_k} and \eqref{opt:free_t}).

\begin{lemma}
\label{lemma:Lagrangian_dual}
\begin{itemize}
\item[(a)] The infimum that defines $g(\lambda)$ can always be achieved by setting $\delta=\lambda$. That is $g(\lambda) = L(\lambda, \lambda)$.
\item[(b)] The Lagrangian dual $g(\lambda)$  can be decomposed token-wise: if we define $L_t(\delta, \lambda) \coloneqq D_{\mathrm{kl},t}(\delta)- \lambda \cdot ( \mathrm{DG}_t(\delta) - \Delta)$ and henceforth $f_t$ and $g_t$ accordingly, then (i) Part (a) also holds for each $g_t(\lambda)$, and (ii) we have
$
g(\lambda) = \frac1T \sum_{t=1}^T g_t(\lambda)
$
and each $g_t$ is concave with $\frac{\mathrm{d} g_t}{\mathrm{d} \lambda} = \Delta - \mathrm{DG}_t(\lambda)$.
\item[(c)] Suppose the primal problem \eqref{opt:nofree} is feasible. Then the strong duality holds
\begin{equation}
\inf_{\delta} \ f(\delta) = \sup_{\lambda \geq 0} \ g(\lambda).
\label{eq:strong_duality}
\end{equation}
with $\delta^* = \lambda^*$, where $\lambda^*$ is the optimal choice of $\lambda$ for maximizing the Lagrangian dual $g(\lambda)$. 
\end{itemize}
\end{lemma}

Lemma \ref{lemma:Lagrangian_dual} justifies the design of Algorithm \ref{alg:DualGA}. The choice of $\delta_t \leftarrow \lambda_t$ in Step 4 of the algorithm serves two purposes: First, if we have $\lambda_t$ converging to $\lambda^*$, then we will also have a converging $\delta_t$ to $\delta^*$ according to Part (c) of Lemma \ref{lemma:Lagrangian_dual}. Second, only the choice of $\delta_t=\lambda_t$ can result in that $\Delta - \mathrm{DG}_t(\delta_t)$ is the gradient of $g_t$ and thus it ensures that the algorithm's update of $\lambda_{t+1}$ performs online gradient ascent in the dual space.  The online gradient ascent procedure will converge because of the concavity of each $g_t(\lambda)$, and this exactly dictates how we update the dual variable in Step 7 and Step 8 of Algorithm \ref{alg:DualGA}, where the hyper-parameter $\eta$ is the step size of the gradient-based updates.

The existing practice of the watermarking algorithm of \cite{kirchenbauer2023watermark} usually involves a hyper-parameter tuning procedure that finds a proper combination of $\gamma$ and $\delta.$ Yet we note that the procedure is conducted on a population level -- once $\gamma$ and $\delta$ are selected, they are applied to all the generations (for all the prompts). And we will see in the numerical experiment that this population-level choice prevents the chosen combination from staying on the Pareto optimality curve (of model distortion and detection ability). In contrast, for our algorithm, we choose $\delta$ adaptively for each prompt $x$. The choice is also adaptive to the LM $\bm{p}$ and the realization of the generation.

Compared with the standard soft red list watermarking algorithm of \citet{kirchenbauer2023watermark}, Algorithm \ref{alg:DualGA} adds only a modest computational overhead, introduced by the adaptive selection of $\delta_t$ at each time step $t$. As detailed in Steps 7 and 8 of Algorithm \ref{alg:DualGA}, the extra cost is evaluating $\mathrm{DG}_t(\delta_t)$, which requires access to the conditional next-token probability vector both before and after the green probability boost. Concretely, given the same logits from the model forward pass, Algorithm \ref{alg:kirchen} forms only the boosted distribution (applying one softmax / normalization over the vocabulary $\mathcal{V}$), whereas Algorithm \ref{alg:DualGA} also needs the original pre‑boost distribution to compute $\mathrm{DG}_t(\delta_t)$, adding at most one additional normalization over $\mathcal{V}$.  Hence, the watermarking step in Algorithm \ref{alg:DualGA} performs no more than twice the vector operations of Algorithm \ref{alg:kirchen} (two normalizations instead of one), and the overall watermarking complexity remains  $O(|\mathcal{V}|\cdot T_{\max})$ (with $|\mathcal{V}|$ denoting the vocabulary size and $T_{\max}$ the maximum sentence length). In practice, the difference in the runtime of Algorithm \ref{alg:DualGA} and Algorithm \ref{alg:kirchen} is almost negligible, primarily due to that the common step of sampling the next token is typically performed on the CPU, which dominates the runtime for both algorithms. Meanwhile, the watermarking step just involves logit manipulations, which are simple vector operations executed on the GPU and especially optimized for batch processing. As a result, the additional computational time is typically negligible.  The empirical computational costs of DualGA, along with those of the other benchmark algorithms for comparison, are provided in Appendix \ref{subapx:computational_table}.



\begin{algorithm}[ht!]
\caption{Dual Gradient Ascent for Soft Watermarking}
    \label{alg:DualGA}
    \begin{algorithmic}[1] 
    \Require Original LM $\bm{p}$, prompt $x$, pseudorandom function $f_{\gamma}$ to determine the green list with ratio $\gamma$, random seed $\mathcal{K}$, DG constraint $\Delta$, step size $\eta$, initial dual variable $\lambda_1$
    \Ensure $y_1, \dots, y_{T_{\text{max}}}$
    \State Initialize the dual variable $\lambda_1$, $t\gets1$, $y_0\gets $ the last token of $x$
    \While{$t \leq T_{\text{max}}$ and $y_{t-1} \neq v_{\text{term}}$}
        \State Input the seed $\mathcal{K}$ and (a  hash of) $y_{t-1}$ into a pseudorandom function $f$ to randomly partition the vocabulary into a green list with size $\gamma |\mathcal{V}|$ and a red list with size $(1-\gamma) |\mathcal{V}|$. Observe the green list probability
        \[G_t \gets \sum_{k = \text{green},  k \in \mathcal{V}} p_{t,k}\]
        \State Set ${\delta}_t$ according to the Lagrangian dual
        \[{\delta}_t \gets \lambda_t\]
        \State Compute the logit vector generated by the original LM $\bm{p}$ and modify it to produce a probability vector
        \[
        q_{t,k} \gets 
        \frac{\exp(l_{t,k} + {\delta}_t \cdot \mathbbm{1}\{k \text{ green at }t\})}{\sum_{k^\prime=1}^{|\mathcal{V}|} \exp(l_{t,k^\prime} + {\delta}_t \cdot \mathbbm{1}\{k^\prime \text{ green at }t\})}
        \]
        \State Sample out a token
        \[y_t \sim \bm{q}_{t} = \left(q_{t,1},...,q_{t,|\mathcal{V}|}\right)\]
        \State Compute the online gradient of the dual function
        \[\mathrm{gd}_t \gets \Delta - \mathrm{DG}_t(\delta_t)\]
        \State Update the dual variable via online gradient ascent
        \[\lambda_{t+1} \gets \lambda_{t} + \eta \cdot \mathrm{gd}_t\]
        \State $t\gets t+1$
    \EndWhile
    \If{$t<T_{\text{max}}$}
        \State Set all remaining $y_{t+1}, \dots, y_{T_{\text{max}}}$ to be $v_{\text{term}}$
    \EndIf
    \end{algorithmic}
\end{algorithm}

\subsection{Algorithm Analysis}

Now we provide some theoretical analysis of the algorithm. The techniques are not new, and we simply aim to generate more intuitions for this watermarking algorithm. 

\begin{assumption}
\label{assum:iid}
We assume that 

\begin{itemize}
\item[(a)] The green word probabilities under the original LM $G_t \coloneqq \sum_{k = \text{green},  k \in \mathcal{V}} p_{t,k}\in [0,1]$ are i.i.d. random variables for all $t \in [T]$.
\item[(b)] The optimal dual variable $\lambda^*$ is bounded with a known upper bound $M$:
$\lambda^* \in [0, M].$
Furthermore, we assume the initial $\lambda_1$ and the final $\lambda_{T+1}$ are within this range.
\end{itemize}
\end{assumption}

For the assumptions, Part (a) states that the green word probabilities are i.i.d., and it is much milder than to enforce all the conditional probabilities $\bm{p}_{t|[t-1]}(\cdot)$ to be i.i.d. (which is quite unrealistic). We will also relax this assumption shortly. Part (b) imposes a mild condition on the boundedness of the dual optimal solution.


\textbf{Remark. } We note that all the three optimization problems \eqref{opt:free_t_k}, \eqref{opt:free_t}, and \eqref{opt:nofree} are feasible if $G_t \in (0,1)$ for all $t \in [T]$ and $\sum_{t=1}^T (1-G_t) \geq T\Delta$. Under Assumption \ref{assum:iid}, both requirements are satisfied with probability at least $1-\varepsilon$ if $\Delta < 1 - \gamma - \tilde{O}(1/\sqrt{T})$ where $\tilde{O}$ hides the poly-logarithmic factors. With a moderate choice of $\Delta$, we can always see the optimization problems as feasible.

\begin{theorem}
\label{thm:dualGA_performance_iid}
Under Assumption \ref{assum:iid}, with the step size $\eta = \Theta(1/\sqrt{T})$, Algorithm \ref{alg:DualGA} satisfies
\[
\frac{1}{T} \sum_{t=1}^T \mathbb{E}[D_{\mathrm{kl}, t}(\delta_t)] \leq \mathbb{E}\big[\mathrm{OPT}(\Delta)\big] + O(1/\sqrt{T})
\]
and
\[
\frac{1}{T} \sum_{t=1}^T \mathbb{E}[\mathrm{DG}_t(\delta_t)] \geq \Delta - O(1/\sqrt{T})\]
where the expectations are taken with respect to the randomness of the algorithm and the language model jointly.
\end{theorem}

The theorem states that the expected optimality gap and the expected constraint violation are all on the order of $1/\sqrt{T}$. Specifically, as the number of tokens $T$ generated by the watermarked language model increases, the expected difference between the performance of Algorithm \ref{alg:DualGA} and the theoretical optimal performance decreases at the rate of $O(1/\sqrt{T})$. Similarly, the degree to which the algorithm violates the constraint (of ensuring a certain detection ability) also diminishes at the rate of $O(1/\sqrt{T})$, meaning that the algorithm increasingly adheres to the required detection capability. Note that the algorithm does not require any prior knowledge or any hyper-parameter tuning; it can adaptively fit into whatever prompt and LM and achieve a similar level of near-optimality. The proof is standard in the literature of online stochastic programming and online resource allocation \citep{agrawal2014fast, li2020simple}.

Part (a) of Assumption \ref{assum:iid} is not critical. For Algorithm \ref{alg:DualGA}, we can explicitly relate the sub-optimality with the extent of non-i.i.d.ness. Specifically, we can define a global non-stationarity measure (as the bandits with knapsacks literature \citep{liu2022non})
\[
W_{T,1} \coloneqq \sum_{t=1}^T \|D_{\mathrm{kl}, t} - \overline{D_{\mathrm{kl}}}\|_{\infty}, \ \ \ \ 
W_{T,2} \coloneqq \sum_{t=1}^T \|\mathrm{DG}_{t} - \overline{\mathrm{DG}}\|_{\infty},
\]
where $\overline{D_{\mathrm{kl}}}$ and $\overline{\mathrm{DG}}$ denote the average of the KL divergence and the DG function, respectively. Then the bounds in Theorem \ref{thm:dualGA_performance_iid} accordingly become
\[
\frac{1}{T} \sum_{t=1}^T \mathbb{E}[D_{\mathrm{kl}, t}(\delta_t)] \leq \mathbb{E}\big[\mathrm{OPT}(\Delta)\big] + O(1/\sqrt{T}) + O(W_{T,1}/T),
\]
\[
\frac{1}{T} \sum_{t=1}^T \mathbb{E}[\mathrm{DG}_t(\delta_t)] \geq \Delta - O(1/\sqrt{T}) - O(W_{T,2}/T).\]
In addition to the sub-optimality gap $O(1/\sqrt{T})$, there is an additional term related to the non-i.i.d.ness of the online process. We note that the bounds can be conservative in that in practice, we observe the algorithms perform much better than the bounds prescribe. But at least, these bounds assure that Algorithm \ref{alg:DualGA} still gives a stable performance even if the i.i.d. part of Assumption \ref{assum:iid} does not hold.

\textbf{Remark.} We note that Algorithm \ref{alg:DualGA} is not the only algorithm that works for this online setting. We choose this mainly for its simplicity and nice empirical performance. As the optimization problems \eqref{opt:free_t} and \eqref{opt:nofree} share the same optimum and both of them can be transformed into equivalent convex programs, we can apply the algorithms in the literature of the online convex optimization with constraints to obtain theoretical guarantees without the i.i.d. assumption at all. For example, \citet{neely2017online} propose a virtual queue (also known as backpressure) algorithm that achieves both $O(1/\sqrt{T})$ guarantees for regret and constraint violation without the i.i.d. assumption. Numerically, we find Algorithm \ref{alg:DualGA} performs better than the backpressure algorithm and it is easier to implement as well.

\section{Discussions on Model Distortion and KL Divergence}
\label{sec:KL_DISTORT}

In this section, we provide more discussions on the choice of KL divergence and compare it against the other two popular criteria for the watermarking problem -- perplexity and marginal distortion-free. We hope the discussion calls for more reflections on the rigorousness and properness of the measurements for the watermarked LMs' model distortion. We believe in comparison with these two criteria, the KL divergence is a better one in quantifying the distortion of a watermarked LM and these discussions also further justify our choice of the objective function in our optimization problems.


\subsection{Distortion as a Price for Watermarking}

In Section \ref{subsec:Pareto_Optimality}, we establish a connection between the objective function of the optimization problems and the KL divergence between the original LM $\bm{p}$ and the watermarked LM $\bm{q}.$ The following proposition further connects the two types of errors with the model distortion under the KL divergence. 

\begin{proposition}
\label{prop:info_lower_bound}
For any prompt $x$ and any watermarking scheme $\bm{q}$, we have
\[\inf_{\mathcal{D}} \alpha_{\bm{p}}(\mathcal{D}, x)+\beta(\mathcal{D},x) \geq 1- \sqrt{1-\exp\big(-D_{\mathrm{kl}}(\bm{q}(x) \| \bm{p}(x))\big)}.\]
In particular, when the KL divergence tends to zero, the sum of two types of errors under any detection algorithm $\mathcal{D}$ cannot be better than a random decision.
\end{proposition}
While the above proposition seems similar to the results in \citet{sadasivan2023can, chakraborty2023possibilities}, the intuition behind it is quite different. Specifically, the existing results \citep{sadasivan2023can, chakraborty2023possibilities} characterize the possibility/impossibility of distinguishing AI-generated texts from humans. For the watermarking problem, the setting is quite different in that the watermarking algorithm is allowed to twist the original LM. In this sense, the KL divergence can be viewed as a minimum price for distinguishing the watermarked LM $\bm{q}$ against an arbitrary distribution $\bm{p}.$

The optimization problems \eqref{opt:free_t_k}, \eqref{opt:free_t}, and \eqref{opt:nofree} all concern the minimum model distortion (the objective) to achieve a certain level of detection ability (the constraint). The above proposition further reinforces such a relationship that one has to distort the model (in terms of a positive KL divergence) to be able to distinguish the watermarked LM from all the other LMs.  

\subsection{Comparison with Other Criteria}

The distribution distortion caused by watermarking is mainly characterized in two ways by the existing literature: (i) the distance between the marginal distributions that generate each token $y_t$ and (ii) the difference of expected logarithm perplexity. The former is often used in theoretical analysis \citep{aaronson2023watermarking, christ2023undetectable, kuditipudi2023robust} and the latter in empirical evaluation \citep{tian2023gptzero, liang2023gpt, kirchenbauer2023watermark, wouters2023optimizing}. In this subsection, we compare the KL divergence with those two popular criteria. 

\subsubsection{Comparison with Log-Perplexity}

The notion of perplexity (PPL) has been long used as a performance measure for evaluating language models, and it is also commonly used for the problem detecting AI-generated texts from humans \citep{wallach2009evaluation, beresneva2016computer, tian2023gptzero}. Formally, the perplexity is defined as follows.

\begin{definition}
For an evaluation language model $\bm{p}_{\text{eval}}$, a prompt $x$, and a token sequence $\bm{y}_{[T]}$, the perplexity (PPL) is defined as
\[\mathrm{PPL}(\bm{y}_{[T]} | \bm{p}_{\text{eval}}, x) \coloneqq \exp\Big(-\frac{1}{T}\sum_{t=1}^T \log\big(\bm{p}_{\text{eval}}(x, \bm{y}_{[t-1]})_{y_t}\big) \Big).\]
In practice, the perplexity is often used taking the logarithm, resulting in the logarithm of the perplexity (LoP).
Furthermore, the expected logarithm of the perplexity of a language model $\bm{p}$ can be defined as
\[\mathrm{LoP}(\bm{p} | \bm{p}_{\text{eval}}, x) \coloneqq \mathbb{E}_{\bm{y}\sim \bm{p}}\big[\mathrm{LoP}(\bm{y} | \bm{p}_{\text{eval}}, x)\big],\]
and the model distortion of $\bm{q}$ is measured via the difference of the expected logarithm of the perplexity (DLoP)
\[
\mathrm{DLoP}(\bm{q}, \bm{p}|\bm{p}_{\text{eval}}, x) \coloneqq \mathrm{LoP}(\bm{q} | \bm{p}_{\text{eval}}, x) - \mathrm{LoP}(\bm{p} | \bm{p}_{\text{eval}}, x).
\]
\end{definition}


A line of research \citep{kirchenbauer2023watermark, wouters2023optimizing} in the watermarking field utilizes DLoP as a measurement of model distortion in the belief that the model distortion is small if the DLoP is small. In such a measurement the evaluation LM $\bm{p}_{\text{eval}}$ is often chosen to be the original LM or a larger LM. Despite the properness of the perplexity in other NLP tasks, its usage in evaluating model distortion in watermarking can be misleading. The reason is that it may lead to an awkward situation where one could severely twist the language model while making the perplexity even smaller. Such an intuition is formalized as follows.

\begin{proposition}
\label{prop:negative_DLoP}
Consider generating a fixed length of $T$ tokens with a prompt $x$ and a language model $\bm{p}$ and using a (known) evaluation language model $\bm{p}_{\text{eval}}$. Suppose there is no tie in $\mathcal{V}^{T}$ for $\bm{p}_{\text{eval}}(\cdot|x)$ and $|\mathcal{V}^{T}| \geq 3$. Then there exists $\tilde{\bm{q}}$, s.t. 
\[
\mathrm{DLoP}(\tilde{\bm{q}}, \bm{p}|\bm{p}_{\text{eval}}, x) < 0,
\]
while $
D_{\mathrm{kl}}(\tilde{\bm{q}} \|\bm{p}) > 0.$
\end{proposition}

The LM $\tilde{\bm{q}}$ can be constructed by a deterministic LM that always predicts the word with the largest probability under the evaluation LM $\bm{p}_{\text{eval}}.$ The assumption of knowing the evaluation LM in the construction of such a model $\tilde{\bm{q}}$ is not restrictive: a common practice of training the language model $\bm{p}$ is to distill the larger evaluation model $\bm{p}_{\text{eval}}$ into some smaller model so that $\bm{p} \approx \bm{p}_{\text{eval}}$, which means that the counter-example can be constructed by knowing $\bm{p}$ as a proxy. This $\tilde{\bm{q}}$ even decreases the perplexity against the original LM but it is apparently not a good LM in that (i) it is a deterministic one and (ii) it has a positive KL-divergence from the original LM $\bm{p}.$ Therefore, Proposition \ref{prop:negative_DLoP} warns of a potential issue of using the perplexity to measure the performance of a watermarked LM. 


\subsubsection{Comparison with Marginal Distortion-Free}
Another popular criterion for defining a \textit{good} watermarking algorithm is to check if the watermarked LM is marginal distortion-free \citep{aaronson2023watermarking, christ2023undetectable, kuditipudi2023robust, hu2023unbiased, wu2024resilient}. A watermarking algorithm is called \textit{marginal distortion-free} if the average distribution of the next token prediction under the watermarked LM exactly matches the original LM, where the average is taken for the key $\mathcal{K}$. We formally define it as follows.

\begin{definition}
\label{def:margin_distort_free}
A watermarking algorithm $\mathcal{A}$ is called as \textit{marginally distortion-free} if for any language model $\bm{p}$, the following holds for all $ t \leq T_{\text{max}}$ and prompt $x \in \mathcal{X}$
\[\mathbb{E}_{\mathcal{K}}\big[\bm{q}_{t|[t-1]}(x)\big] = \bm{p}_{t|[t-1]}(x).\]
\end{definition}

From this definition, we believe a rigorous name for this property should be \textit{marginally} distortion-free, while the existing literature simply calls it distortion-free. Calling it distortion-free gives a misconception that the watermarking algorithm does not make or need any change to the underlying LM, which could be misleading in some cases. One can think of this marginal distortion-free property as a property of the watermarking algorithm on the \textit{population} level. If we could aggregate all the LMs watermarked by different keys (statistically, marginalizing over the keys), it would match the original LM. However, each watermarked LM is twisted. For example, consider the case where the key is public: the watermarking service can only be provided with some LM watermarked by only one key so that the key is made public and shared between the examiners and the provider. That model is drastically different from the original LM. A guarantee of marginally distortion-free says nothing about any single LM's performance as it only concerns an averaging experience for all the keys. In comparison, the KL-divergence that we study throughout our paper and in all the above optimization problems, ensure the model distortion in a sample-path manner. That is, it ensures every realization of the watermarked LM behaves closely to the original LM. 
 

\begin{proposition}
\label{prop:KL_cvx} It holds that $
D_{\mathrm{kl}}\big(\mathbb{E}_{\mathcal{K}}[\bm{q}] \big\| \bm{p}\big)\leq \mathbb{E}_{\mathcal{K}}\big[D_{\mathrm{kl}}(\bm{q} \| \bm{p})\big].$
\end{proposition}

The above proposition characterizes the relationship between the two measurements of model distortion. Our measurement of the KL divergence can be viewed as an upper bound of the marginal distortion of the LM (on the left-hand side). The marginal distortion-free property requires $D_{\mathrm{kl}}\big(\mathbb{E}_{\mathcal{K}}[\bm{q}] \big\| \bm{p}\big)=0$.

To give some concrete examples, we provide a result on two popular marginal distortion-free algorithms: exponential minimum sampling \citep{aaronson2023watermarking, christ2023undetectable, kuditipudi2023robust,li2025statistical} and inverse transform sampling \citep{kuditipudi2023robust,li2025statistical}. We show that they yield an expected KL divergence whose quantity is the same as the entropy of the original model.\\

\noindent\begin{minipage}[t]{\textwidth}
\begin{example} Exponential minimum sampling \citep{aaronson2023watermarking, christ2023undetectable, kuditipudi2023robust,li2025statistical} and inverse transform sampling \citep{kuditipudi2023robust,li2025statistical} are not distortion-free.
\label{eg:ems} We have $D_{\mathrm{kl}}\Big(\mathbb{E}_{\mathcal{K}}\big[\bm{q}_{t|[t-1]}\big] \Big\| \bm{p}_{t|[t-1]}\Big) = 0,$
but $\mathbb{E}_{\mathcal{K}}\Big[D_{\mathrm{kl}}\big(\bm{q}_{t|[t-1]} \big\| \bm{p}_{t|[t-1]}\big)\Big] = \mathrm{Ent}_{y_t}\big(\bm{p}_{t|[t-1]} \big),$
where $\mathrm{Ent}_{y_t}(\cdot)$ denotes the entropy with respect to the $t$-th token $y_t$.\\

\end{example}
\end{minipage}

To interpret this example, it means the watermarked LM received by one user $\bm{q}_{t|[t-1]}$ has a KL divergence towards the original LM $\bm{p}_{t|[t-1]}$ that is as large as the entropy of the original LM. This resolves the seeming contradiction between our impossibility result in Proposition \ref{prop:info_lower_bound} and the existence of ``distortion-free'' algorithms. The marginal distortion-free criterion hides the necessary price paid for the watermarking. Such a criterion itself does not necessarily imply a practical failure, but it does not rule out those possibilities. For example, \citet{kuditipudi2023robust} notice a phenomenon in practice when applying the watermarking scheme of \citet{aaronson2023watermarking} that the watermarked model would be easily trapped in a tautology (that is, constantly repeating the same sentence). The above proposition and example explain such an undesirable phenomenon.

\section{Experiments}
\label{sec:exp}
In this section, we evaluate our Dual Gradient Ascent (DualGA) algorithm (Algorithm \ref{alg:DualGA}).
First, we evaluate the algorithm’s detection capability and its degree of model distortion in comparison to existing benchmarks. Our empirical results demonstrate that it achieves an optimal trade-off between these two aspects, consistently maintaining strong detection performance across all prompts. Second, we investigate the algorithm’s behavior under varying language models (LMs) and prompts. We also assess its robustness against adversarial watermark-removal attacks. Additional experiments are presented in Appendix \ref{subapx:additional_exp},  including examples of generated responses (Appendix \ref{apxsub:Example_response}), the use of the dual variable as a signal for monitoring text repetition, which is a common issue in hashing-based watermarking methods (Appendix \ref{subsec:body_text_repetition_detection}), and the sensitivity to hyperparameter choices (Appendix \ref{appsec:robust_hyper}), among others.

\subsection{Experiment Setup}

\textbf{Language Models.} We employ several language models to generate the watermarked sentences, including Llama-3-8B \citep{grattafiori2024llama}, Llama-7B \citep{touvron2023llama},  and OPT-2.7B \citep{zhang2022opt}. Unless otherwise specified, Llama-3-8B is used as the default model for most experiments. 


\textbf{Datasets.} The experiments test the algorithms in two datasets: C4 \citep{raffel2020exploring} and LFQA \citep{fan2019eli5}. The C4 dataset is an extensive collection of English-language texts sourced from the public Common Crawl web scrape. The LFQA dataset, designed for long-form question-answering, is compiled by picking questions with their longest answers from Reddit. More details about the dataset are provided in Appendix \ref{subapx:data_construct}.

\textbf{Benchmark Algorithms.} We consider the Soft Red List (SRL) watermarking algorithm \citep{kirchenbauer2023watermark} (which keeps $\delta_{t,k}\equiv\delta$ in Algorithm~\ref{alg:kirchen}), the Exponential Minimum Sampling (EMS) \citep{aaronson2023watermarking}, the Unbiased Watermark (UW) \citep{hu2023unbiased}, and the Adaptive Text Watermark (ATW) \citep{liu2024adaptive}. For each algorithm, we report results under multiple hyperparameter settings. Due to space constraints, full implementation details are deferred to Appendix~\ref{subapx:algs_implement}.


\textbf{Our Algorithm.} The hyperparameters of the DualGA algorithm include the green list ratio $\gamma$, the DG constraint $\Delta$, the step size $\eta$, and the initial dual variable $\lambda_1$. Among these, the most important parameter that governs the algorithm's behavior is $\Delta$ (in Appendix \ref{subapx:choose_green_ratio}, we explain how $\gamma$ is selected based on $\Delta$). Further implementation details of DualGA are provided in Appendix \ref{subapx:algs_implement}.

\subsection{Detection Ability, Distortion, and Pareto Optimality}

In this subsection, we evaluate the detection ability and distortion level of different algorithms. To test the detection ability, we control the false positive rate (FPR), i.e., the probability of predicting an unwatermarked text as watermarked,  at pre-specified levels ($10^{-4}$ or $10^{-6}$)
and quantify the detection ability by the true positive rates (TPR), i.e., the probability of successfully detecting a watermarked text. The model distortion is measured by the KL-divergence (KL) (averaged over tokens) between the watermarked and unwatermarked text distributions, as defined in Section \ref{sec:KL_DISTORT}. We also provide more discussions on the detection ability metrics in Appendix \ref{subapx:detection_ps}.

Table \ref{Tab:exp_performance} shows the evaluation of the detection ability (TPR) alongside the model distortion (KL). For each algorithm family we set multiple configurations, with each of them tested on two datasets with two significance levels ($10^{-4}$ and $10^{-6}$). 
Each configuration of these algorithms corresponds to a different combination of the detection ability and the model distortion. In comparison, our algorithm of DualGA exhibits a better trade-off between the two aspects; for a certain level of model distortion, our algorithm achieves a higher TPR, and conversely, for a certain level of TPR, our algorithm exhibits a smaller model distortion. All algorithms in Table \ref{Tab:exp_performance} are evaluated using Llama-3-8B as the text generation model, and the advantage of our algorithm also holds for other models such as Llama-7B (the results are provided in Appendix \ref{subapx:addi_detectVSdistort}).

\begin{table}[h!]
  \centering
\resizebox{\textwidth}{!}{
  \begin{tabular}{ll|ccc|ccc}
    \toprule
    \multirow{3}{*}{Method} & \multirow{3}{*}{Configuration}&\multicolumn{3}{c|}{C4} &\multicolumn{3}{c}{LFQA}  \\
    & & \multicolumn{2}{c}{TPR $\uparrow$}& \multirow{2}{*}{KL$\downarrow$} & \multicolumn{2}{c}{TPR$\uparrow$}& \multirow{2}{*}{KL$\downarrow$} \\ 
    \cmidrule(lr){3-4} \cmidrule(lr){6-7}
& &FPR$<10^{-4}$ &FPR$<10^{-6}$& &FPR$<10^{-4}$ &FPR$<10^{-6}$& \\
    \midrule
\multirow{9}{*}{SRL}   
& $\delta=1, \gamma=0.1$&0.08 & 0.02 & 0.02 & 0.17 & 0.06 & 0.02 \\
 &  $\delta=1,\gamma=0.5$&0.27 & 0.09 & 0.05 & 0.31 & 0.11 & 0.05 \\
 & $\delta=1,\gamma=0.7$& 0.17 & 0.03 & 0.04 & 0.18 & 0.07 & 0.03 \\
 
& $\delta=5,\gamma=0.1$&1.0 & 0.99 & 1.2 & 1.0 & 1.0 & 1.2 \\
   & $\delta=5,\gamma=0.5$&1.0 & 1.0 & 0.65 & 1.0 & 1.0 & 0.65 \\
   & $\delta=5,\gamma=0.7$&0.99 & 0.97 & 0.35 & 0.99 & 0.96 & 0.35 \\
& $\delta=10,\gamma=0.1$&0.99 & 0.99 & 1.4 & 1.0 & 0.99 & 1.4 \\
& $\delta=10,\gamma=0.5$&1.0 & 1.0 & 1.2 & 1.0 & 0.99 & 1.2 \\
& $\delta=10,\gamma=0.7$&0.98 & 0.93 & 0.57 & 0.98 & 0.92 & 0.54 \\
    \midrule
 \multirow{3}{*}{EMS} 
 & $\tau=0.5$&0.38 & 0.16 & 0.71 & 0.22 & 0.14 & 0.54 \\
 & $\tau=1$&0.90 & 0.82 & 1.03 & 0.92 & 0.87 & 0.46 \\
 & $\tau=1.5$&1.0 & 1.0 & 2.6 & 1.0 & 1.0 & 2.7 \\
    \midrule
\multirow{2}{*}{UW} 
 & $\delta$-reweight & 1.0 & 1.0 & 2.3 & 1.0 & 1.0 & 2.1 \\
 & $\gamma$-reweight & 1.0 & 1.0 & 1.1 & 1.0 & 0.99 & 1.2 \\
 \midrule
 \multirow{2}{*}{ATW} 
 & $\delta=1,\alpha=2$ & 0.51 & 0.29 & 0.07 & 0.56 & 0.33 & 0.07\\
 & $\delta=5,\alpha=2$ & 0.99 & 0.92 & 0.36 & 1.0 & 0.96 & 0.36 \\
\midrule 
 \multirow{4}{*}{DualGA} 
 & $\Delta=0.2$&0.92 & 0.91 & 0.28 & 0.89 & 0.87 & 0.27 \\
 & $\Delta=0.3$&0.99 & 0.98 & 0.34 & 0.98 & 0.98 & 0.31 \\
 & $\Delta=0.4$&1.0 & 1.0 & 0.57 & 1.0 & 0.99 & 0.59 \\
 & $\Delta=0.5$ &1.0 & 1.0 & 1.1 & 1.0 & 1.0 & 0.97 \\
    \bottomrule
  \end{tabular}}
  \caption{The detection ability (TPR) and model distortion (KL) of watermarking algorithms across different hyperparameter configurations. The TPR under different FPR thresholds measures the detection ability (the higher the better). The KL measures the distortion of the watermarked text (the lower the better).}\label{Tab:exp_performance}
\end{table}

Figure \ref{fig:LFQA_Pareto} visualizes the trade-off between the detection ability and the model distortion. For two language models, Llama-3-8B (left) and OPT-2.7B (right), evaluated on LFQA,  Figure \ref{fig:LFQA_Pareto} plots respectively the realized DG against the KL on the population level, where each point represents one algorithm configuration's averaged performance of 500 test prompts. Here, the realized DG is calculated as the realized difference in the green word probability (defined in \eqref{eq:define_DG} and computed by the average over $t\in [1, T_{\mathrm{max}}]$ for each prompt).

Both plots in Figure \ref{fig:LFQA_Pareto} show that our algorithm achieves a Pareto optimality for whatever choice of the constraint budget parameter $\Delta$. In other words, the parameter $\Delta$ gives a handle to balance the trade-off but does not affect the algorithm's Pareto optimality. While for SRL, although it can also achieve Pareto optimality under specific choices of hyperparameters, it exhibits suboptimal performance when the parameters $\delta$ and $\gamma$ are not well-aligned (e.g., $\gamma=0.5, \delta=5$). Moreover, even when a particular SRL configuration achieves an optimal trade-off for a specific language model (LM), there is no guarantee that the same configuration remains optimal for a different LM. For instance, the configuration $\gamma = 0.5, \delta = 2$ is nearly optimal for Llama-3-8B (see the left plot of Figure \ref{fig:LFQA_Pareto}), but not for OPT-2.7B (see the right plot of Figure \ref{fig:LFQA_Pareto}), whereas $\gamma=0.7, \delta=5$ is nearly optimal for OPT-2.7B but not for Llama-3-8B. This observation reinforces the theoretical result in Section \ref{subsec:Pareto_Optimality}, which highlights that the optimal choice of $\delta$ depends on multiple factors, including the selection of LM and the value of $\gamma$.

\begin{figure}[!htb]
\centering
\includegraphics[width=0.7\linewidth]{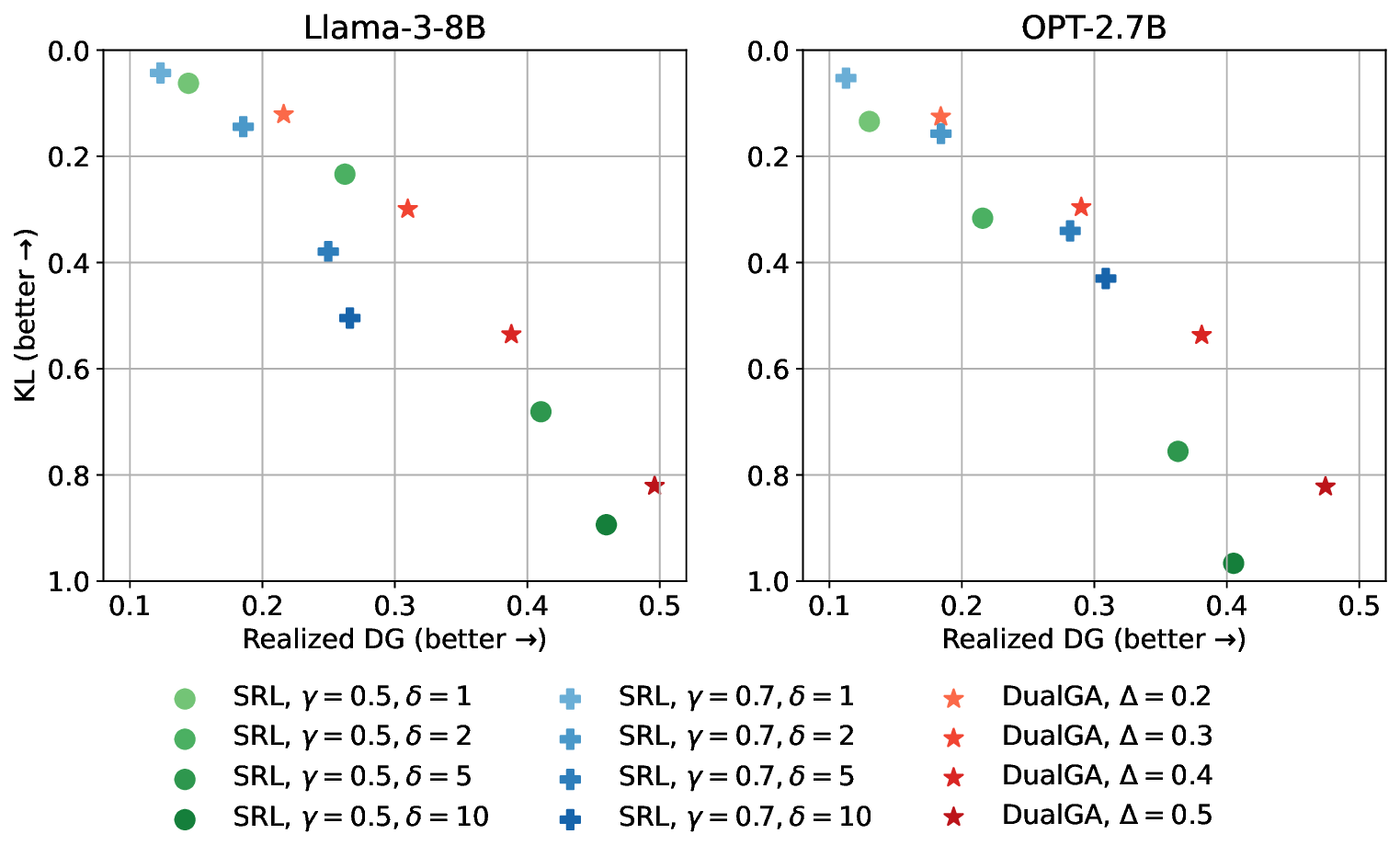}
\caption{Detection ability (Realized DG) and distortion (KL) on the LFQA dataset by Llama-3-8B (left) and OPT-2.7B (right). Each point represents the population-level values (i.e., averaged over all samples) of DG and KL for a given algorithm configuration. In both plots, several SRL configurations and all DualGA configurations lie on the Pareto-optimal curve, indicating an effective balance between detection ability and distortion. However, not all SRL configurations achieve this optimal trade-off: a configuration that is Pareto-optimal for one language model (e.g., SRL with $\gamma=0.5, \delta=2$ in the left plot) may fall notably below the curve for another model (e.g., the same configuration in the right plot).}\label{fig:LFQA_Pareto}
\end{figure}

\subsection{Consistent Performance across Diverse Scenarios}
\label{sec:Consistent_Performance_exp}
A major strength of our algorithm lies in its adaptiveness. For different generation tasks and LMs, it automatically achieves Pareto optimality. In this subsection, we empirically demonstrate that DualGA demonstrates consistent performance across diverse deployment scenarios, requiring only minimal computational overhead to search for the optimal hyperparameters.

Figure \ref{fig:scatter_4} visualizes the performance of these algorithms at the individual prompt level. Each plot represents a specific combination of dataset and LM, with each point indicating the outcome from a single prompt. In every experimental scenario, the average performance of our algorithm lies on the Pareto-optimal curve, which agrees with the results shown in Figure \ref{fig:LFQA_Pareto}. Moreover, when tracking the performance of a fixed algorithm configuration across the four plots, we observe that for any fixed value of $\Delta$, DualGA maintains consistent performance across different datasets and language models. In contrast, SRL (also under a fixed parameter configuration) shows much less stability across different experimental settings. At the individual prompt level, DualGA exhibits a very concentrated DG level (near the target $\Delta$) for different prompts, while the SRL algorithm has a highly varied level of model distortion and detection ability for each prompt even under the same parameter configuration of $\delta$ and $\gamma$. These results suggest that DualGA demonstrates a more consistent ability to ensure uniform detection across individual prompts and a diverse range of experiment scenarios.

Another interesting phenomenon in Figure \ref{fig:scatter_4} is that, for our DualGA algorithm, the variation of its distortion (KL) and the realized DG exhibit a positive correlation. We note that the algorithm minimizes model distortion subject to a specific detection ability, and this suggests that prompts that are inherently more challenging to watermark (e.g., factual questions like ``Where is the capital of the U.K.?'') require a greater distortion change to achieve the same level of detection abilities compared to inherently easier prompts (e.g., opinion questions like ``What is your favorite color?''). For a small detection target $\Delta$, both types of prompts may only need some minor modification to the original response to meet the detection target and thus both allow low distortion. However, to achieve higher detection ability, more complex prompts necessitate stronger watermarking efforts, potentially skewing the correct answer or incorporating redundant tokens, whereas easier prompts allow more flexibility for embedding watermarks and thus still a small distortion. Additional comparisons with the EMS algorithm are provided in Appendix \ref{subapx:stability_apx}.

\begin{figure}[!htb]
\centering
\includegraphics[width=1\linewidth]{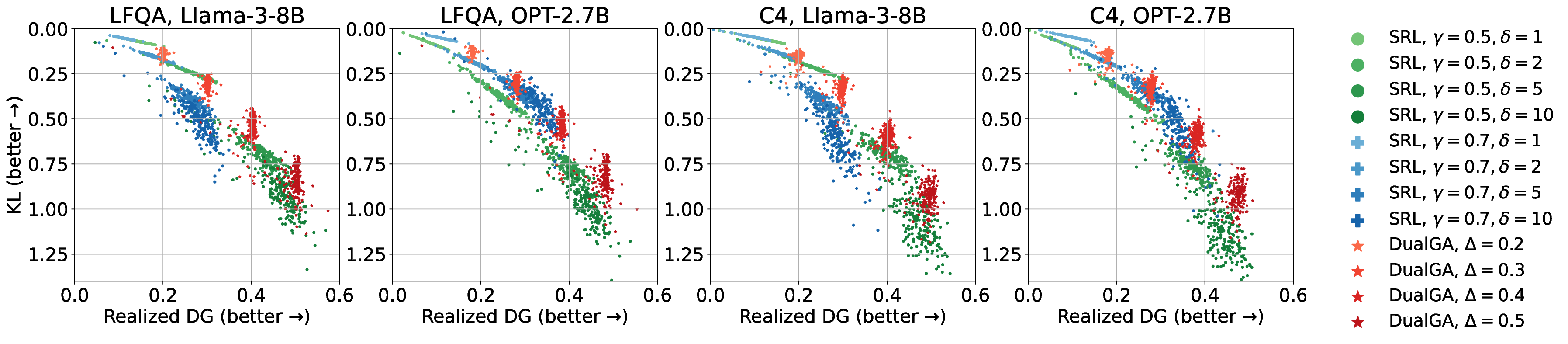}
\caption{Detection ability (Realized DG) and distortion (KL) at the individual prompt level. Each point represents an algorithm's outcome on a single prompt, and each plot corresponds to a specific dataset–LM pairing. Red points indicate the DualGA method, while blue and green points represent the SRL baselines. Across all four plots, the shape and position of the DualGA point clusters remain largely consistent, whereas some of the SRL clusters shift and reshape significantly. This suggests that DualGA consistently achieves uniform detection across different prompts, tasks, and LMs. 
}\label{fig:scatter_4}
\end{figure}

\subsection{Robustness under Attacks}\label{subsec:body_text_robust}

We evaluate the robustness of DualGA and several benchmark methods against three standard watermark-removal attacks: deletion, insertion, and substitution, as well as three more advanced attacks: paraphrasing, entropy-based substitution, and color-based substitution. In paraphrasing, GPT‑4 is prompted to rewrite each sentence by replacing a specified percentage of words; entropy-based substitution first identifies high-entropy positions and then substitutes tokens in these positions; color-based substitution prompts GPT‑4 to generate multiple paraphrases that replace only the ``green'' tokens, and retains the version achieving the most green-to-red substitutions. The methodology and implementation details for each attack type are provided in Appendix \ref{subapx:robustness_attack}. Attack rate is defined as the percentage of tokens affected. The performance of all considered algorithms under varying attack types and attack rates is shown in Figure \ref{fig:robustness_6_attacks}, where detection ability is quantified using $p$-values (defined in Appendix \ref{subapx:detection_ps}), with lower values indicating stronger detection.

Some of the key observations include that: (i) Under mild attack scenarios (Attack Rate $\leq 30\%$), most methods maintain strong performance, as reflected by their median $p$-values below $10^{-4}$ (the black dashed line in the figures). (ii) Typically, enhanced robustness appears with the increased distortion, as evidenced by higher KL values noted in the legend. (iii) Specifically for DualGA, we observe an improvement in robustness as the parameter $\Delta$ increases. Since $\Delta$ is the target DG and directly controls the watermarking strength, texts watermarked by DualGA with higher $\Delta$ are easier to detect and more robust to withstand attacks. (iv) Compared to attacks on randomly selected tokens (the first four types), non-random attacks (entropy-based substitution and color-based substitution) cause a more significant reduction in detection ability at the same attack rate.

\begin{figure}[!h]
\centering
\includegraphics[width=0.8\linewidth]{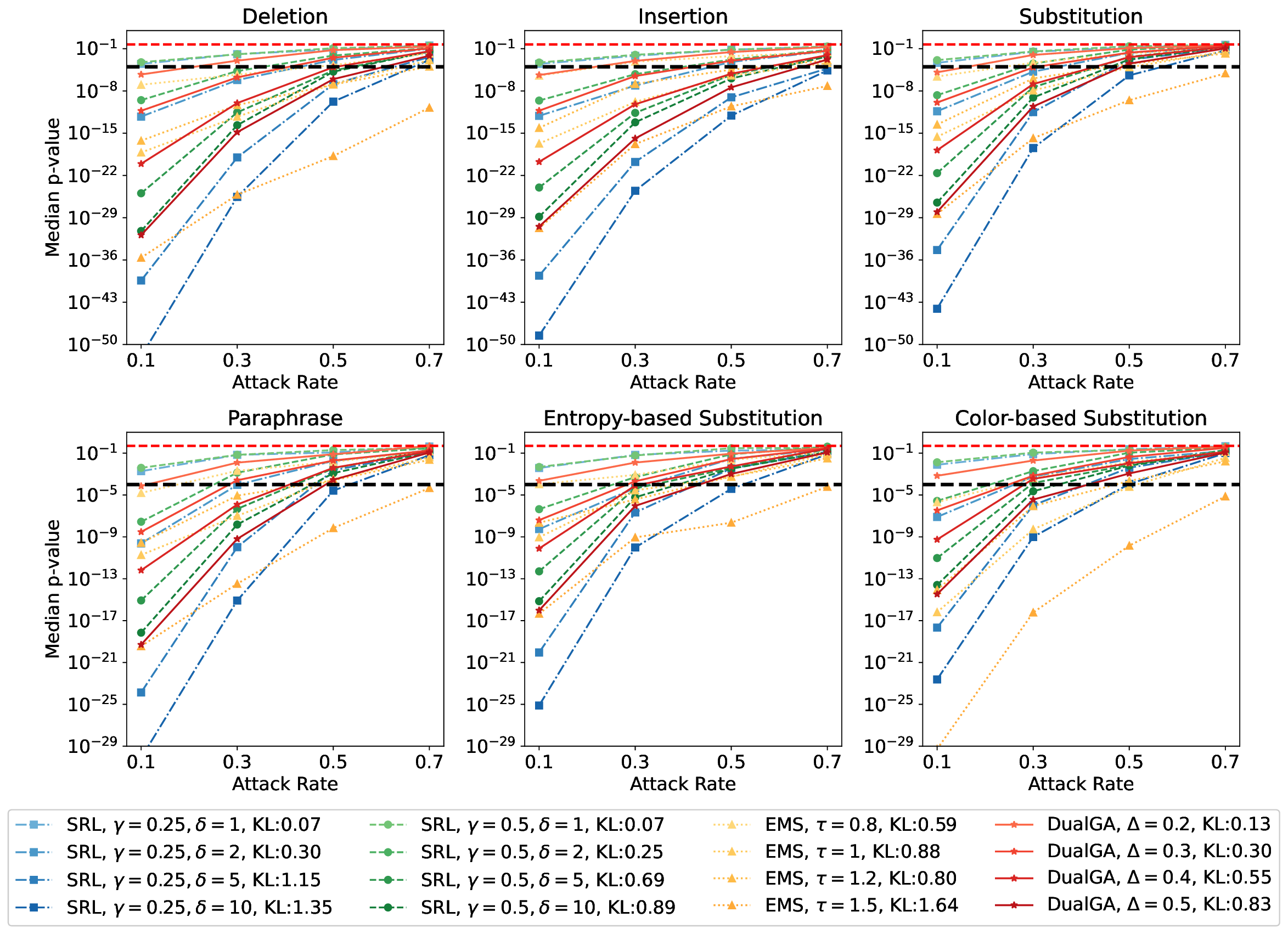}
\caption{The robustness of the detection ability across different types of attack on the LFQA dataset. The black dashed horizontal line is $p=10^{-4}$, and the red dashed line is $p=0.5$. Attack rate is the proportion of tokens that are deleted from, inserted into, or substituted/paraphrased from the original text. All algorithms lose their detection ability as the attack rate gets larger, but generally, watermarking configurations that induce heavier model distortion are more robust to attacks.}
\label{fig:robustness_6_attacks}
\end{figure}

\section{Conclusion}
We study watermarking for large language models through a statistical lens, casting the core trade-off between detection ability and model distortion as a constrained optimization problem, and showing that a simple, adaptive dual–gradient ascent procedure (DualGA) achieves near–Pareto-optimal performance. Our analysis further justifies why KL divergence is a more appropriate measure of model distortion than others. Empirically, DualGA delivers a stronger distortion–detection frontier than popular baselines across datasets and language models, maintains uniform per-prompt detection ability, and exhibits competitive robustness under common attack models.

\noindent \textbf{Conflict of Interest} The authors report there are no competing interests to declare.

\bibliographystyle{informs2014}
\bibliography{main}

@inproceedings{fernandez2023three,
  title={Three bricks to consolidate watermarks for large language models},
  author={Fernandez, Pierre and Chaffin, Antoine and Tit, Karim and Chappelier, Vivien and Furon, Teddy},
  booktitle={2023 IEEE International Workshop on Information Forensics and Security (WIFS)},
  pages={1--6},
  year={2023},
  organization={IEEE}
}

@article{li2025statistical,
  title={A statistical framework of watermarks for large language models: Pivot, detection efficiency and optimal rules},
  author={Li, Xiang and Ruan, Feng and Wang, Huiyuan and Long, Qi and Su, Weijie J},
  journal={The Annals of Statistics},
  volume={53},
  number={1},
  pages={322--351},
  year={2025},
  publisher={Institute of Mathematical Statistics}
}

@article{fan2019eli5,
  title={ELI5: Long form question answering},
  author={Fan, Angela and Jernite, Yacine and Perez, Ethan and Grangier, David and Weston, Jason and Auli, Michael},
  journal={arXiv preprint arXiv:1907.09190},
  year={2019}
}

@article{yellott1977relationship,
  title={The relationship between Luce's choice axiom, Thurstone's theory of comparative judgment, and the double exponential distribution},
  author={Yellott Jr, John I},
  journal={Journal of Mathematical Psychology},
  volume={15},
  number={2},
  pages={109--144},
  year={1977},
  publisher={Elsevier}
}

@inproceedings{agrawal2014fast,
  title={Fast algorithms for online stochastic convex programming},
  author={Agrawal, Shipra and Devanur, Nikhil R},
  booktitle={Proceedings of the twenty-sixth annual ACM-SIAM symposium on Discrete algorithms},
  pages={1405--1424},
  year={2014},
  organization={SIAM}
}

@article{li2020simple,
  title={Simple and fast algorithm for binary integer and online linear programming},
  author={Li, Xiaocheng and Sun, Chunlin and Ye, Yinyu},
  journal={Advances in Neural Information Processing Systems},
  volume={33},
  pages={9412--9421},
  year={2020}
}

@article{touvron2023llama,
  title={LLaMA: Open and efficient foundation language models},
  author={Touvron, Hugo and Lavril, Thibaut and Izacard, Gautier and Martinet, Xavier and Lachaux, Marie-Anne and Lacroix, Timoth{\'e}e and Rozi{\`e}re, Baptiste and Goyal, Naman and Hambro, Eric and Azhar, Faisal and others},
  journal={arXiv preprint arXiv:2302.13971},
  year={2023}
}

@article{raffel2020exploring,
  title={Exploring the limits of transfer learning with a unified text-to-text transformer},
  author={Raffel, Colin and Shazeer, Noam and Roberts, Adam and Lee, Katherine and Narang, Sharan and Matena, Michael and Zhou, Yanqi and Li, Wei and Liu, Peter J},
  journal={The Journal of Machine Learning Research},
  volume={21},
  number={1},
  pages={5485--5551},
  year={2020},
  publisher={JMLRORG}
}

@article{radford2019language,
  title={Language models are unsupervised multitask learners},
  author={Radford, Alec and Wu, Jeffrey and Child, Rewon and Luan, David and Amodei, Dario and Sutskever, Ilya and others},
  journal={OpenAI blog},
  volume={1},
  number={8},
  pages={9},
  year={2019}
}

@article{brown2020language,
  title={Language models are few-shot learners},
  author={Brown, Tom and Mann, Benjamin and Ryder, Nick and Subbiah, Melanie and Kaplan, Jared D and Dhariwal, Prafulla and Neelakantan, Arvind and Shyam, Pranav and Sastry, Girish and Askell, Amanda and others},
  journal={Advances in neural information processing systems},
  volume={33},
  pages={1877--1901},
  year={2020}
}

@article{kuditipudi2023robust,
  title={Robust distortion-free watermarks for language models},
  author={Kuditipudi, Rohith and Thickstun, John and Hashimoto, Tatsunori and Liang, Percy},
  journal={arXiv preprint arXiv:2307.15593},
  year={2023}
}

@article{christ2023undetectable,
  title={Undetectable Watermarks for Language Models},
  author={Christ, Miranda and Gunn, Sam and Zamir, Or},
  journal={arXiv preprint arXiv:2306.09194},
  year={2023}
}

@misc{aaronson2023watermarking,
  title={Watermarking of Large Language Models},
  author={Aaronson, Scott},
  url={https://simons.berkeley.edu/talks/scott-aaronson-ut-austin-openai-2023-08-17},
  year={2023},
  urldate={2024-03-17}
}

@article{kirchenbauer2023watermark,
  title={A watermark for large language models},
  author={Kirchenbauer, John and Geiping, Jonas and Wen, Yuxin and Katz, Jonathan and Miers, Ian and Goldstein, Tom},
  journal={arXiv preprint arXiv:2301.10226},
  year={2023}
}

@article{liang2023gpt,
  title={GPT detectors are biased against non-native English writers},
  author={Liang, Weixin and Yuksekgonul, Mert and Mao, Yining and Wu, Eric and Zou, James},
  journal={arXiv preprint arXiv:2304.02819},
  year={2023}
}

@article{wouters2023optimizing,
  title={Optimizing watermarks for large language models},
  author={Wouters, Bram},
  journal={arXiv preprint arXiv:2312.17295},
  year={2023}
}

@misc{tian2023gptzero,
  title={Gptzero update v1},
  author={Tian, Edward},
  year={2023},
  url={https://gptzero.substack.com/p/gptzero-update-v1}
}

@article{sadasivan2023can,
  title={Can ai-generated text be reliably detected?},
  author={Sadasivan, Vinu Sankar and Kumar, Aounon and Balasubramanian, Sriram and Wang, Wenxiao and Feizi, Soheil},
  journal={arXiv preprint arXiv:2303.11156},
  year={2023}
}

@article{chakraborty2023possibilities,
  title={On the possibilities of ai-generated text detection},
  author={Chakraborty, Souradip and Bedi, Amrit Singh and Zhu, Sicheng and An, Bang and Manocha, Dinesh and Huang, Furong},
  journal={arXiv preprint arXiv:2304.04736},
  year={2023}
}

@inproceedings{wallach2009evaluation,
  title={Evaluation methods for topic models},
  author={Wallach, Hanna M and Murray, Iain and Salakhutdinov, Ruslan and Mimno, David},
  booktitle={Proceedings of the 26th annual international conference on machine learning},
  pages={1105--1112},
  year={2009}
}

@inproceedings{beresneva2016computer,
  title={Computer-generated text detection using machine learning: A systematic review},
  author={Beresneva, Daria},
  booktitle={Natural Language Processing and Information Systems: 21st International Conference on Applications of Natural Language to Information Systems, NLDB 2016, Salford, UK, June 22-24, 2016, Proceedings 21},
  pages={421--426},
  year={2016},
  organization={Springer}
}

@book{le2012asymptotic,
  title={Asymptotic methods in statistical decision theory},
  author={Le Cam, Lucien},
  year={2012},
  publisher={Springer Science \& Business Media}
}

@article{bretagnolle1978estimation,
  title={Estimation des densit{\'e}s: risque minimax},
  author={Bretagnolle, Jean and Huber, Catherine},
  journal={S{\'e}minaire de probabilit{\'e}s de Strasbourg},
  volume={12},
  pages={342--363},
  year={1978}
}

@article{polyanskiy2014lecture,
  title={Lecture notes on information theory},
  author={Polyanskiy, Yury and Wu, Yihong},
  journal={Lecture Notes for ECE563 (UIUC) and},
  volume={6},
  number={2012-2016},
  pages={7},
  year={2014},
  publisher={Citeseer}
}

@inproceedings{zinkevich2003online,
  title={Online convex programming and generalized infinitesimal gradient ascent},
  author={Zinkevich, Martin},
  booktitle={Proceedings of the 20th international conference on machine learning (icml-03)},
  pages={928--936},
  year={2003}
}

@article{liu2022non,
  title={Non-stationary bandits with knapsacks},
  author={Liu, Shang and Jiang, Jiashuo and Li, Xiaocheng},
  journal={Advances in Neural Information Processing Systems},
  volume={35},
  pages={16522--16532},
  year={2022}
}

@article{neely2017online,
  title={Online convex optimization with time-varying constraints},
  author={Neely, Michael J and Yu, Hao},
  journal={arXiv preprint arXiv:1702.04783},
  year={2017}
}

@book{BrockwellTimSeries,
author = {Brockwell, Peter J and Davis, Richard A},
title = {Time series: theory and methods},
year = {1986},
isbn = {0387964061},
publisher = {Springer-Verlag},
address = {Berlin, Heidelberg}
}

@article{hoeffding1994probability,
  title={Probability inequalities for sums of bounded random variables},
  author={Hoeffding, Wassily},
  journal={The collected works of Wassily Hoeffding},
  pages={409--426},
  year={1994},
  publisher={Springer}
}

@book{boyd2004convex,
  title={Convex optimization},
  author={Boyd, Stephen P and Vandenberghe, Lieven},
  year={2004},
  publisher={Cambridge university press}
}

@inproceedings{radford2023robust,
  title={Robust speech recognition via large-scale weak supervision},
  author={Radford, Alec and Kim, Jong Wook and Xu, Tao and Brockman, Greg and McLeavey, Christine and Sutskever, Ilya},
  booktitle={International Conference on Machine Learning},
  pages={28492--28518},
  year={2023},
  organization={PMLR}
}

@article{zellers2019defending,
  title={Defending against neural fake news},
  author={Zellers, Rowan and Holtzman, Ari and Rashkin, Hannah and Bisk, Yonatan and Farhadi, Ali and Roesner, Franziska and Choi, Yejin},
  journal={Advances in neural information processing systems},
  volume={32},
  year={2019}
}

@article{ippolito2019automatic,
  title={Automatic detection of generated text is easiest when humans are fooled},
  author={Ippolito, Daphne and Duckworth, Daniel and Callison-Burch, Chris and Eck, Douglas},
  journal={arXiv preprint arXiv:1911.00650},
  year={2019}
}

@article{gehrmann2019gltr,
  title={Gltr: Statistical detection and visualization of generated text},
  author={Gehrmann, Sebastian and Strobelt, Hendrik and Rush, Alexander M},
  journal={arXiv preprint arXiv:1906.04043},
  year={2019}
}

@article{solaiman2019release,
  title={Release strategies and the social impacts of language models},
  author={Solaiman, Irene and Brundage, Miles and Clark, Jack and Askell, Amanda and Herbert-Voss, Ariel and Wu, Jeff and Radford, Alec and Krueger, Gretchen and Kim, Jong Wook and Kreps, Sarah and others},
  journal={arXiv preprint arXiv:1908.09203},
  year={2019}
}

@inproceedings{mitchell2023detectgpt,
  title={Detectgpt: Zero-shot machine-generated text detection using probability curvature},
  author={Mitchell, Eric and Lee, Yoonho and Khazatsky, Alexander and Manning, Christopher D and Finn, Chelsea},
  booktitle={International Conference on Machine Learning},
  pages={24950--24962},
  year={2023},
  organization={PMLR}
}

@article{paul2023top,
 author  = {Paul, Kari and Bhuiyan, Johana and Rushe, Dominic},
 year    = {2023},
 title   = {Top tech firms commit to AI safeguards amid fears over pace of change},
 journal = {The Guardian},
 urldate = {2024-03-17}
}

@inproceedings{atallah2001natural,
  title={Natural language watermarking: Design, analysis, and a proof-of-concept implementation},
  author={Atallah, Mikhail J and Raskin, Victor and Crogan, Michael and Hempelmann, Christian and Kerschbaum, Florian and Mohamed, Dina and Naik, Sanket},
  booktitle={Information Hiding: 4th International Workshop, IH 2001 Pittsburgh, PA, USA, April 25--27, 2001 Proceedings 4},
  pages={185--200},
  year={2001},
  organization={Springer}
}

@inproceedings{atallah2002natural,
  title={Natural language watermarking and tamperproofing},
  author={Atallah, Mikhail J and Raskin, Victor and Hempelmann, Christian F and Karahan, Mercan and Sion, Radu and Topkara, Umut and Triezenberg, Katrina E},
  booktitle={International workshop on information hiding},
  pages={196--212},
  year={2002},
  organization={Springer}
}

@article{meral2009natural,
  title={Natural language watermarking via morphosyntactic alterations},
  author={Meral, Hasan Mesut and Sankur, B{\"u}lent and {\"O}zsoy, A Sumru and G{\"u}ng{\"o}r, Tunga and Sevin{\c{c}}, Emre},
  journal={Computer Speech \& Language},
  volume={23},
  number={1},
  pages={107--125},
  year={2009},
  publisher={Elsevier}
}

@inproceedings{chiang2004natural,
  title={Natural language watermarking using semantic substitution for Chinese text},
  author={Chiang, Yuei-Lin and Chang, Lu-Ping and Hsieh, Wen-Tai and Chen, Wen-Chih},
  booktitle={Digital Watermarking: Second International Workshop, IWDW 2003, Seoul, Korea, October 20-22, 2003. Revised Papers 2},
  pages={129--140},
  year={2004},
  organization={Springer}
}

@inproceedings{topkara2006natural,
  title={Natural language watermarking: Challenges in building a practical system},
  author={Topkara, Mercan and Riccardi, Giuseppe and Hakkani-T{\"u}r, Dilek and Atallah, Mikhail J},
  booktitle={Security, Steganography, and Watermarking of Multimedia Contents VIII},
  volume={6072},
  pages={106--117},
  year={2006},
  organization={SPIE}
}

@article{vaswani2017attention,
  title={Attention is all you need},
  author={Vaswani, Ashish and Shazeer, Noam and Parmar, Niki and Uszkoreit, Jakob and Jones, Llion and Gomez, Aidan N and Kaiser, {\L}ukasz and Polosukhin, Illia},
  journal={Advances in neural information processing systems},
  volume={30},
  year={2017}
}

@article{zhao2023provable,
  title={Provable robust watermarking for ai-generated text},
  author={Zhao, Xuandong and Ananth, Prabhanjan and Li, Lei and Wang, Yu-Xiang},
  journal={arXiv preprint arXiv:2306.17439},
  year={2023}
}

@book{johnson2001information,
  title={Information hiding: steganography and watermarking-attacks and countermeasures: steganography and watermarking: attacks and countermeasures},
  author={Johnson, Neil F and Duric, Zoran and Jajodia, Sushil},
  volume={1},
  year={2001},
  publisher={Springer Science \& Business Media}
}

@article{devlin2018bert,
  title={Bert: Pre-training of deep bidirectional transformers for language understanding},
  author={Devlin, Jacob and Chang, Ming-Wei and Lee, Kenton and Toutanova, Kristina},
  journal={arXiv preprint arXiv:1810.04805},
  year={2018}
}

@article{kirchenbauer2023reliability,
  title={On the reliability of watermarks for large language models},
  author={Kirchenbauer, John and Geiping, Jonas and Wen, Yuxin and Shu, Manli and Saifullah, Khalid and Kong, Kezhi and Fernando, Kasun and Saha, Aniruddha and Goldblum, Micah and Goldstein, Tom},
  journal={arXiv preprint arXiv:2306.04634},
  year={2023}
}

@article{huang2023towards,
  title={Towards optimal statistical watermarking},
  author={Huang, Baihe and Zhu, Banghua and Zhu, Hanlin and Lee, Jason D and Jiao, Jiantao and Jordan, Michael I},
  journal={arXiv preprint arXiv:2312.07930},
  year={2023}
}

@article{li2024statistical,
  title={A Statistical Framework of Watermarks for Large Language Models: Pivot, Detection Efficiency and Optimal Rules},
  author={Li, Xiang and Ruan, Feng and Wang, Huiyuan and Long, Qi and Su, Weijie J},
  journal={arXiv preprint arXiv:2404.01245},
  year={2024}
}

@article{hu2023unbiased,
  title={Unbiased watermark for large language models},
  author={Hu, Zhengmian and Chen, Lichang and Wu, Xidong and Wu, Yihan and Zhang, Hongyang and Huang, Heng},
  journal={arXiv preprint arXiv:2310.10669},
  year={2023}
}

@inproceedings{wu2024resilient,
  title={A Resilient and Accessible Distribution-Preserving Watermark for Large Language Models},
  author={Wu, Yihan and Hu, Zhengmian and Guo, Junfeng and Zhang, Hongyang and Huang, Heng},
  booktitle={Forty-first International Conference on Machine Learning},
  year={2024}
}

@article{huo2024token,
  title={Token-Specific Watermarking with Enhanced Detectability and Semantic Coherence for Large Language Models},
  author={Huo, Mingjia and Somayajula, Sai Ashish and Liang, Youwei and Zhang, Ruisi and Koushanfar, Farinaz and Xie, Pengtao},
  journal={arXiv preprint arXiv:2402.18059},
  year={2024}
}

@article{grattafiori2024llama,
  title={The llama 3 herd of models},
  author={Grattafiori, Aaron and Dubey, Abhimanyu and Jauhri, Abhinav and Pandey, Abhinav and Kadian, Abhishek and Al-Dahle, Ahmad and Letman, Aiesha and Mathur, Akhil and Schelten, Alan and Vaughan, Alex and others},
  journal={arXiv preprint arXiv:2407.21783},
  year={2024}
}

@article{zhang2022opt,
  title={Opt: Open pre-trained transformer language models},
  author={Zhang, Susan and Roller, Stephen and Goyal, Naman and Artetxe, Mikel and Chen, Moya and Chen, Shuohui and Dewan, Christopher and Diab, Mona and Li, Xian and Lin, Xi Victoria and others},
  journal={arXiv preprint arXiv:2205.01068},
  year={2022}
}

@article{liu2024adaptive,
  title={Adaptive text watermark for large language models},
  author={Liu, Yepeng and Bu, Yuheng},
  journal={arXiv preprint arXiv:2401.13927},
  year={2024}
}

\newpage
\appendix
\section{Related Literature}
\paragraph{Watermarking Natural Languages.}
Unlike the computer vision tasks where the underlying data is \textit{continuous}, watermarking natural languages is deemed more difficult due to its discrete nature \citep{johnson2001information}. Traditional watermarking schemes try to transform the text into a whole new text to encode information, such as syntactic structure restructuring \citep{atallah2001natural}, paraphrasing \citep{atallah2002natural}, semantic substitution \citep{chiang2004natural}, synonym substitution \citep{topkara2006natural}, morphosyntactic alterations \citep{meral2009natural}, and so on. Recent advancements in modern language models such as Transformer \citep{vaswani2017attention} enable researchers to construct new approaches based on monitoring the token generation. \citet{aaronson2023watermarking, christ2023undetectable} propose an exponential minimum sampling algorithm that maintains the marginal distribution per token the same. When generating the $t$-th token, they use the previously observed $n$ tokens as a seed for the pseudo-random function to generate a sequence of uniformly distributed random variables $r_{t, k}$ over $[0, 1]$ for each $k \in |\mathcal{V}|$, and select the token to minimize $-\frac{\text{log}(r_{t,k})}{p_{t,k}}$. By showing that the marginal distribution of each token is identical to the original distribution when averaged across the whole space of $r_{t,k}$, the authors argue that such a property is desirable when conducting watermarking. Such a criterion is formulated in Definition \ref{def:margin_distort_free}. \citet{kuditipudi2023robust} later derive a watermarking scheme by assuming a sequence of secret keys so that each token can be uniquely assigned a pseudo-random function, resulting in a more robust watermarking at the cost of higher requirements of sharing the secret keys. \citet{kirchenbauer2023watermark} develop a soft watermarking scheme based on pseudo-randomly dividing the vocabulary into a red list and a green list and increasing the probabilities in the green list. \citet{zhao2023provable} propose fitting the red-green list across the whole sequence, leading to a provably more robust watermarking in comparison to the original one in \citet{kirchenbauer2023watermark}. \citet{fernandez2023three} point out that the $z$-score test used in previous research \citep{kirchenbauer2023watermark} relies on a fundamental assumption of i.i.d. tokens, which may not hold in some cases, leading to an underestimated false positive rate. As for the marginal distortion-free criterion (Definition \ref{def:margin_distort_free}), different follow-up works \citet{hu2023unbiased, wu2024resilient} are developed for the red-green list watermarking scheme of \citet{kirchenbauer2023watermark}.

\paragraph{Optimal Condition of Watermarking Language Models.} 
Different attempts have been made to certify the optimal condition of watermarking LLMs. The most straightforward way is to find the optimal trade-off between the type I and the type II errors \citep{huang2023towards}. The uniformly most powerful (UMP) watermark specified in \citet{huang2023towards} is a Neyman-Pearson type likelihood ratio test and requires the knowledge of the distribution of the full sequence output, which is impractical since the service provider does not generate the full distribution but in an autoregressive way and the detector cannot know the rejection region associated with some texts generated at some arbitrary time points. Pursuing the UMP watermarking algorithm is more of a theoretical interest rather than a practical one, suggesting that achieving the optimal Type I/Type II error trade-off may be difficult in practice. Some literature takes one step back by not designing the watermarking algorithm directly but the detection rules. \citet{li2024statistical} provides a framework for computing the asymptotic efficiency and examines the optimal condition of the detection rules under two watermarking algorithms \citet{aaronson2023watermarking, kuditipudi2023robust}. For the red-green list watermarking scheme, \citet{wouters2023optimizing} sets the goal to be maximizing the difference of the green list probability (DG) under the same intuition as ours that the larger the DG, the better the Type I/Type II error trade-off one can achieve. \citet{wouters2023optimizing} curves the Pareto optimum of the DG against the difference of the log-perplexity, arguing that the optimal strategy is to apply no watermarking at all when the expected change of log-perplexity is greater than some threshold while applying the hard watermarking in the opposite cases. While our Pareto optimum result measured by the KL divergence suggests that one should set the strength of the watermarking equal across the whole sequence, the proposal by \citet{wouters2023optimizing} is optimized against the difference of log-perplexity, which could be even negative and should not be regarded as a good measurement of distortion for the watermarking problem (See Section \ref{sec:KL_DISTORT}). \citet{huo2024token} proposes another objective of maximizing the semantic coherence as higher-level preservation of the original model. The algorithm minimizes the semantic loss against maximizing the DG yet lacks theoretical guarantees.

\paragraph{Post Hoc Detection.}
The difference between the statistical properties of machine-generated texts and human-generated texts has been noticed for a long time, and researchers are making various attempts to distinguish the machine-generated texts using those features. For instance, \citet{gehrmann2019gltr} build a model named GLTR to use the information of the histograms of the log-likelihoods to detect machine-generated texts. \citet{solaiman2019release} make a simpler proposal by inspecting the total log-likelihood of the whole sequence. Researchers also attempt to use another language model to distinguish the texts: \citet{ippolito2019automatic} employ a fine-tuned BERT \citep{devlin2018bert} to classify the texts; \citet{zellers2019defending} train a model called Grover to generate texts given titles and use the Grover model itself to detect the texts. A more recent example is \citet{tian2023gptzero} which detects the abnormally low variation in perplexity when evaluated by corresponding LLMs. \citet{mitchell2023detectgpt} observe that the outputs of LLMs tend to stay in the negative curvature regions of the log-likelihood functions and derive a classifier based on that observation.

\paragraph{Possibility/Impossibility Results.}
There are also some debates on whether the detection of machine-generated texts is possible. \citet{sadasivan2023can} prove that if the total variation distance (denoted by $\mathrm{TV}$) between the language model's distribution and the human's language distribution vanishes, then any detector cannot get a result better than a random decision with respect to AUROC. Their impossibility result shares the same spirit of Le Cam's method with our Proposition \ref{prop:info_lower_bound}, which could also be regarded as another proof supporting the argument that no detection ability could be obtained without model distortion. Later, \citet{chakraborty2023possibilities} mitigates the impossibility result by showing the term $1-\mathrm{TV}$ degrades exponentially fast as the number of tokens increases under the i.i.d. assumption. This implies the possibility of detection by increasing the sequence length under the circumstances that the language model deviates from humans by only a small margin per token. However, their result involving the total variation distance requires the unnatural i.i.d. assumption over the distribution per token. In contrast, our Proposition \ref{prop:KL_decompose} decomposes the joint KL divergence token-wise without any specifying assumption, supporting the natural choice of the KL divergence as model distortion measurement.

\section{Detection Ability and Robustness}
\label{sec:more_analyses}

In this section, we develop error guarantees for our algorithm/the optimization problems and discuss the robustness aspect of the algorithm.

\subsection{Type I and Type II Error}

When we present the optimization problems, the constraint is interpreted as a proxy for detection ability. Now we formally establish the connection between the cumulative difference in green word probability (DG) and the detection error. The result shows that such an optimization formulation and Algorithm \ref{alg:DualGA} provide not only a better theoretical guarantee but also an explicit handle for controlling the \textit{individual} guarantee for both types of error of the watermarked LM \textit{regardless of} the received prompt in contrast to previous algorithms.

As noted in \citet{fernandez2023three}, the $z$-test used in \citet{kirchenbauer2023watermark} only serves as an asymptotic approximation of the true positive rate. Thus, we follow the Beta distribution way in \citet{fernandez2023three} in practice. For theoretical simplicity, we present two types of error bounds based on setting an explicit threshold for the $z$-score on the green token numbers. We define the $z$-score as
\begin{equation}\label{eq:z_score}
z \gets \frac{|\bm{y}_{2:T}|_G - \gamma (T-1)}{\sqrt{(T-1)\gamma(1-\gamma)}},
\end{equation}
where $|\bm{y}_{2:T}|_G$ is defined by the count of green words from the 2-nd token to the $T$-th token. Then we have the following bound.

\begin{proposition}
\label{prop:type_I&II}
Suppose that the constraint of the optimization problem \eqref{opt:nofree} is met. Under Assumption \ref{assum:iid}, for any prompt $x$ and any decision threshold $z_0$ that a sequence is labeled as generated by the watermarked LM if $z \geq z_0$, we have
\[
\alpha_{\bm{p}}(\mathcal{D}, x) \leq \exp\big(-2(\gamma (1-\gamma) z_0^2)\big), \ \ \ 
\beta(\mathcal{D}, x) \leq \exp\big(-2((\Delta \sqrt{T-1} - \sqrt{\gamma (1-\gamma)} z_0)^2)\big).
\]
\end{proposition}

Proposition \ref{prop:type_I&II} states that if we fix the type I error rate $\alpha$, the type II error rate decreases exponentially fast in terms of $\Delta^2 T$. Importantly, this result entails an individual control of the actual $\Delta$, i.e, to meet the constraint of the optimization problem for each prompt, LM, and $\gamma$. This requires solving the optimization program \eqref{opt:nofree} in hindsight. While this is impossible generally, our algorithm ensures the constraint is violated up to the order of $\frac{1}{\sqrt{T}}$. Comparatively, this logic also explains why 
\citet{kirchenbauer2023watermark}'s heuristic and fixed choice of $\delta$ does not lead to any type II error rate \textit{until} it has witnessed all the realized tokens; such an intuition is also shown in the original type II error analysis of \citet{kirchenbauer2023watermark}, which relies on the average entropy of the token sequence and cannot be guaranteed until the full sequence is generated. 


\begin{figure}[ht!]
\begin{adjustbox}{valign=t}
\begin{minipage}{0.9\textwidth}
\centering
\includegraphics[width=0.8\textwidth]{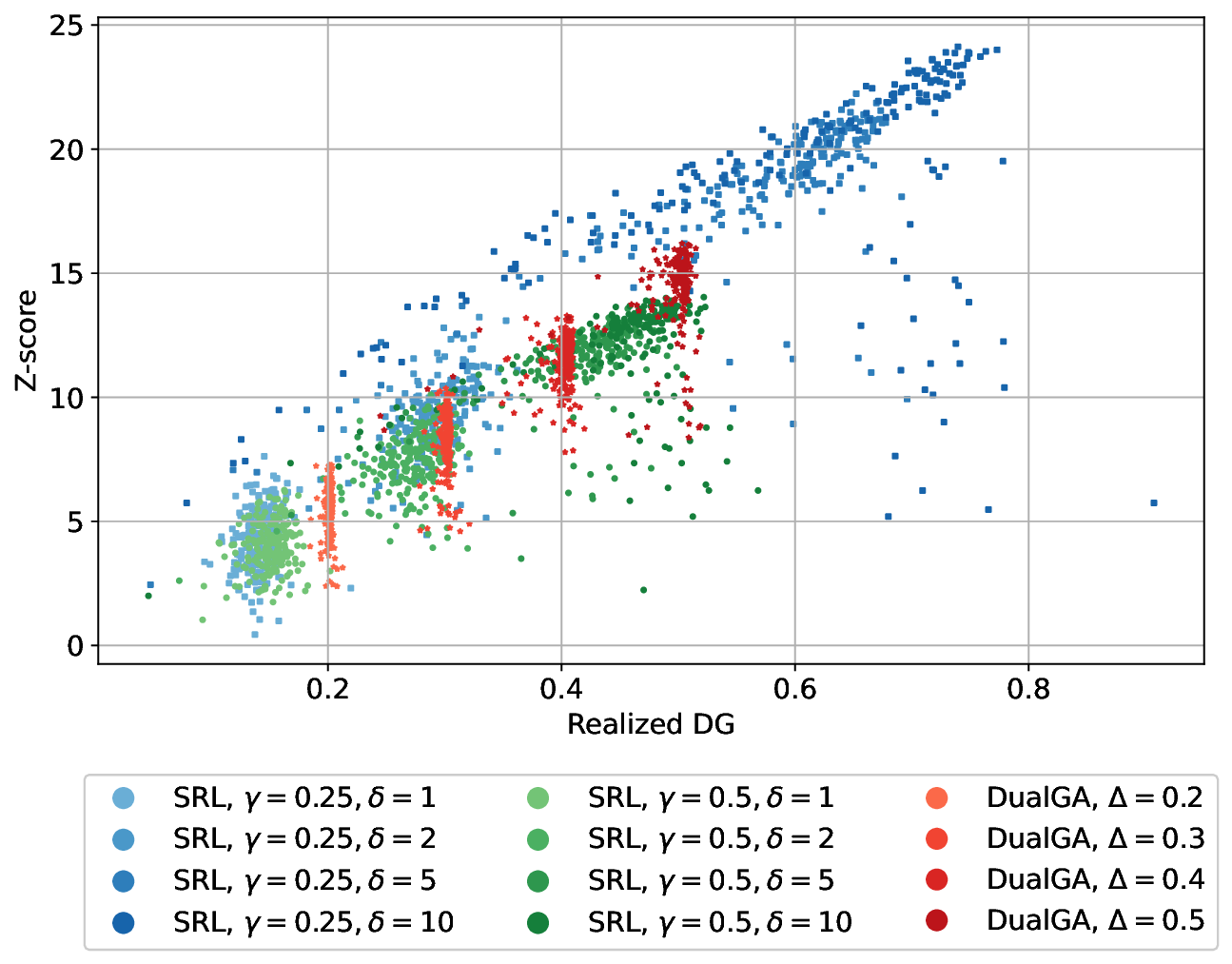}
\caption{The scatter plot of $z$-score v.s. realized DG for different algorithms. SRL stands for the algorithm in \cite{kirchenbauer2023watermark} and DualGA stands for our Algorithm \ref{alg:DualGA} under different parameter combinations. Each point represents one generated sequence, and for each algorithm, 200 sequences are generated.}
    \label{fig:LFQA_scatter}
\end{minipage}
\end{adjustbox}
\end{figure}

Figure \ref{fig:LFQA_scatter} provides a visualization of the relationship between the realized DG and the $z$-score. First, we observe that a larger realized DG corresponds to a larger $z$-score which justifies the choice of using DG as the constraint of the optimization problem. Second, we observe that our algorithm, though implemented in an online fashion, meets the (offline) constraint well (that the realized DG $\geq \Delta$). Third, our adaptive choice of $\delta$ gives a more precise control of the realized DG and the $z$-score. For any combination of $\gamma$ and $\delta$, the sequences generated by SRL have a quite dispersed DG and realized $z$-score.  




\subsection{Robustness}

The following proposition states the robustness of our watermarked algorithm. Specifically, it says how the detection ability changes under different adversarial attacks that aim to undermine the effectiveness of the watermarking algorithm.

\begin{proposition}
\label{prop:adv_robust}
Consider the following three adversarial attacks:

\begin{itemize}
\item Deletion: One adversarially deletes a certain length of the token sequence from the sequence generated by the watermarked LM of Algorithm \ref{alg:DualGA}. Let the length of the deleted sequence being $T^\prime = l \cdot T$, and define
\[
\Delta^\prime = \frac{\Delta - l + l\gamma}{1-l}.
\]
\item Insertion: One adversarially inserts a certain length of token sequence into the sequence generated by the watermarked LM of Algorithm \ref{alg:DualGA}. Let the inserted sequence length $T^\prime = l \cdot T$, and define
\[
\Delta^\prime = \frac{\Delta - l \gamma - \frac1T}{1+l}.
\]
\item  Edit: One adversarially edits a certain length of the token sequence generated by the watermarked LM of Algorithm \ref{alg:DualGA}. Let the edited sequence length $T^\prime = l \cdot T$, and define
\[
\Delta^\prime = \Delta - l - \frac1T.
\]
\end{itemize}
The error bound in Proposition \ref{prop:type_I&II} holds for $\Delta'$ instead of $\Delta$ and new sequence length correspondingly.
\end{proposition}

The proof of the proposition follows the same logic as that of Proposition \ref{prop:type_I&II}. It implies that the watermarked text is robust against an adversarial attack up to linearly many changes to the generated text.

\section{Proofs and Discussions in Section \ref{sec:Pareto_optimal}}
\label{apd:proof_sec3}
\subsection{Proof of Proposition \ref{prop:free_k}}
\begin{proof}
We consider the following one-step optimization program with respect to $\delta_{t,1}, \dots, \delta_{t,|\mathcal{V}|}$ where we omit the subscript of $t$ for notation simiplicity:
\begin{align}
\begin{split}
\min_{\delta_1, \dots, \delta_{|\mathcal{V}|} \in [-\infty, +\infty]} \quad & D_{\mathrm{kl}}(\delta_1,\dots,\delta_{|\mathcal{V}|}) \\
\text{s.t. }& \mathrm{DG}(\delta_1,\dots,\delta_{|\mathcal{V}|}) \geq \Delta.
\end{split}
\label{opt:one_step_k}
\end{align}
By checking the Lagrangian dual, we have the following lemma.
\begin{lemma}
\label{lemma:one_step_k}
Denote the optimal solution of the optimization program \eqref{opt:one_step_k} by $\delta_1^*, \dots, \delta_{|\mathcal{V}|}^*$. Assume the program is feasible with $\Delta \leq \sum_{k \text{ is red}} p_k$. Then
\[
\delta_1^* = \dots = \delta_{|\mathcal{V}|}^*.
\]
\end{lemma}
With the above lemma in mind, we shall easily see why the optimal solution must be in the form of
\[\delta_{t,1}^* = \cdots = \delta_{t,|\mathcal{V}|}^*, \quad \forall t \in [T].\]
The reason is straightforward: if there were an optimal solution $\delta^\prime$ deviating from the above form, say, $\exists t_0 \in [T]$, such that the above equality does not hold for $t_0$. Then, one can easily get another feasible solution $\delta''$ by setting
\[
\delta''_{t,k} \coloneqq
\begin{cases}
\tilde{\delta}_{t_0,k}, \quad &\text{if } t=t_0,\\
\delta^\prime_{t,k}, \quad &\text{otherwise},
\end{cases}
\]
where $\tilde{\delta}_{t_0}$ is the optimal solution of the following program
\begin{align*}
\min_{\delta_1, \dots, \delta_{|\mathcal{V}|}} \quad & D_{\mathrm{kl},t_0}(\delta_1,\dots,\delta_{|\mathcal{V}|}) \\
\text{s.t. }& \mathrm{DG}_{t_0}(\delta_1,\dots,\delta_{|\mathcal{V}|}) \geq \mathrm{DG}_{t_0}(\delta^\prime_{t_0,1},\dots,\delta^\prime_{t_0,|\mathcal{V}|}).
\end{align*}
From Lemma \ref{lemma:one_step_k}, we know that $\delta^\prime_{t_0, 1},\dots,\delta^\prime_{t_0, |\mathcal{V}|}$ cannot be the optimal solution of the above program. Thus, the newly derived solution $\delta''$ is strictly better than $\delta^\prime$, which contradicts the assumption.
\end{proof}

\begin{proof}[Proof of Lemma \ref{lemma:one_step_k}]
First, we transform the optimization problem \eqref{opt:one_step_k} into another equivalent form, where each probability $p_k$ is watermarked to be $r_k p_k$.
\begin{align}
\begin{split}
\min_{r_1,\dots,r_{|\mathcal{V}|}} \quad & \sum_{k=1}^{|\mathcal{V}|} r_k p_k \log(r_k)\\
\text{s.t. }& \sum_{k\in [|\mathcal{V}|], k \text{ is green}} (r_k - 1) p_k \geq \Delta,\\
& \sum_{k=1}^{|\mathcal{V}|} (r_k - 1) p_k = 0.
\end{split}
\label{opt:transformed_one_step_k}
\end{align}
We write down the Lagrangian of the optimization program \eqref{opt:transformed_one_step_k} as
\begin{align*}
&L(r_1, \dots, r_{|\mathcal{V}|}, \lambda, \mu) \\ 
& \coloneqq \sum_{k} r_k p_k \log(r_k) + \lambda (\Delta - \sum_{k \text{ is green}} (r_k - 1) p_k) + \mu (\sum_{k} (r_k - 1) p_k)\\
& = \sum_{k \text{ is green}} r_k p_k (\log(r_k) - \lambda + \mu) + \sum_{k \text{ is red}} r_k p_k (\log(r_k) + \mu) + \lambda(\Delta + \sum_{k \text{ is green}} p_k) - \mu \sum_{k} p_k\\
& = G \cdot r_k(\log(r_k) -\lambda + \mu) + R \cdot r_k(\log(r_k) + \mu) + \lambda (\Delta + G) - \mu (G+R),
\end{align*}
where we have denoted $\sum_{k \text{ is green}} p_k$ by $G$ and $\sum_{k \text{ is red}}$ by $R$ (since we assume feasibility, we automatically have $G+R = 1$ and $\Delta \leq R$).
Then we can write the Lagrangian dual function as
\[
\phi(\lambda, \mu) \coloneqq \inf_{r_1, \dots, r_{|\mathcal{V}|}} L(r_1, \dots, r_{|\mathcal{V}|}, \lambda, \mu),
\]
where $\lambda \geq 0$ and $\mu \in \mathbb{R}$.

We derive the first-order partial derivative of Lagrangian to be
\[\frac{\partial L}{\partial r_k} = 
\begin{cases}
p_k\big(\log(r_k) + 1 - \lambda + \mu \big), \quad & k \text{ green},\\
p_k\big(\log(r_k) + 1 + \mu \big), \quad & k \text{ red}.
\end{cases}\]
Then the supremum (note that the supremum is taken for unconstrained $r_k \geq 0$ now) is reached when the first-order condition is satisfied, that is
\[
r_k^* = 
\begin{cases}
\exp( -1 + \lambda - \mu), \quad & k \text{ green},\\
\exp( -1 - \mu), \quad & k \text{ red}.
\end{cases}
\]
Substituting into the Lagrangian, we have the Lagrangian dual function as
\[
\phi(\lambda, \mu) = -G e^{-1+\lambda-\mu} - R e^{-1-\mu} + \lambda(\Delta + G) - \mu(G+R).
\]
Note that the dual function is always concave, thus we can solve the dual problem
\begin{align*}
\max_{\lambda \geq 0, \mu \in \mathbb{R}} \phi(\lambda, \mu)
\end{align*}
by inspecting the first-order condition
\[
\begin{cases}
e^{-1+\lambda-\mu} G + e^{-1-\mu} R - (G+R) = 0,\\
-e^{-1+\lambda-\mu} + \Delta + G = 0.
\end{cases}
\]
The dual optimum is attained at
\[
\begin{cases}
\lambda^* = \log\big(\frac{(\Delta+G)R}{(R-\Delta)G}\big),\\
\mu^* = -1+\log\big(\frac{R}{R-\Delta}\big),
\end{cases}
\]
where the optimum is
\[\phi(\lambda^*, \mu^*) = (R-\Delta)\log\left(\frac{R-\Delta}{R}\right) + (\Delta+G)\log\left(\frac{\Delta+G}{G}\right).\]

To prove the strong duality, one just needs to notice that
\[\phi(\lambda^*,\mu^*) = \sum_{k=1}^{|\mathcal{V}|} r^*_k p_k \log(r^*_k).\]

Therefore, the optimal solution of the optimization program \eqref{opt:transformed_one_step_k} satisfies $r_1^* = \dots = r_{|\mathcal{V}|}^*$, which concludes the proof.
\end{proof}
\begin{remark}
A quick note on the optimization problem \eqref{opt:transformed_one_step_k} is that the primal program itself is a convex program, yet Slater's condition does not necessarily hold if $\Delta = R$. This case corresponds to the original program's case of $\delta_{k} = +\infty$ for some green $k$. Thus, if we rule out the possibility of DG constraint $\Delta$ to be no less than $R$, the strong duality can be directly derived from Slater's condition.
\end{remark}

\subsection{Proof of Proposition \ref{prop:free_t}}
\label{subapd:proof_free_t}
\begin{proof}
Similar to the notations in the proof of Proposition \ref{prop:free_k}, we write down the optimization program as
\begin{align*}
\min_{\delta_1,\dots,\delta_T} \quad & \frac{1}{T} \sum_{t=1}^T \Big[\frac{G_t \delta_t \cdot e^{\delta_t}}{G_t e^{\delta_t} - G_t + 1} - \log(G_t e^{\delta_t} - G_t + 1)\Big]\\
\text{s.t. } & \frac{1}{T} \sum_{t=1}^T \frac{G_t(1-G_t)(e^{\delta_t} - 1)}{G_t e^{\delta_t} - G_t + 1} \geq \Delta.
\end{align*}
The Lagrangian of the program is
\[
L(\delta_1, \dots, \delta_T, \lambda) \coloneqq \sum_{t=1}^T \frac{G_t \delta_t \cdot e^{\delta_t}}{G_t e^{\delta_t} - G_t + 1} - \log(G_t e^{\delta_t} - G_t + 1) + \lambda\cdot \left(\Delta - \frac{1}{T} \sum_{t=1}^T \frac{G_t(1-G_t)(e^{\delta_t} - 1)}{G_t e^{\delta_t} - G_t + 1}\right).
\]
The next step is to compute the Lagrangian dual
\[
\phi(\lambda) \coloneqq \inf_{\delta_1, \dots, \delta_T \in [-\infty, +\infty]} L(\delta_1,\dots,\delta_T,\lambda).
\]
We first check the first-order derivative
\[
\frac{\partial L}{\partial \delta_t} = \frac{1}{T} \left[\frac{\partial D_{\mathrm{kl},t}}{\partial \delta_t} - \lambda \frac{\partial \mathrm{DG}_t}{\partial \delta_t}\right].
\]
Since the exponential transformation keeps the monotonicity, we transform the partial derivative into the form of $\nu_t = e^{\delta_t}$ as
\begin{align*}
\frac{\partial L}{\partial \nu_t} & = \frac{1}{T} \left[ \frac{\partial D_{\mathrm{kl},t}}{\partial \nu_t} - \lambda \frac{\partial \mathrm{DG}_t}{\partial \nu_t} \right] \\
& = \frac{1}{T} \cdot \frac{G_t(1-G_t)}{(G_t \nu_t- G_t + 1)^2} \cdot (\log(\nu_t) - \lambda).
\end{align*}
Thus we can see explicitly that the first-order condition suffices to reach the minimum with $\nu_t^* = e^\lambda$, equal to
\[
\delta_t^* = \lambda.
\]
We substitute it into the Lagrangian to derive the expression of the Lagrangian dual:
\begin{align*}
\phi(\lambda) & = L(\lambda, \dots, \lambda, \lambda) \\
& =\lambda \cdot \left(\frac{1}{T} \sum_{t=1}^T G_t + \Delta \right) - \frac{1}{T} \sum_{t=1}^T \log(G_t e^{\lambda} - G_t + 1).
\end{align*}
To derive the supremum of the Lagrangian dual, we calculate its first-order derivative as
\begin{align*}
\frac{\partial \phi}{\partial \lambda} & = \Delta - \frac{1}{T} \sum_{t=1}^T \frac{G_t(1-G_t)(e^{\lambda} - 1)}{G_t e^{\lambda} - G_t + 1} \\
& = \Delta - \frac{1}{T} \sum_{t=1}^T \mathrm{DG}_t(\lambda).
\end{align*}
Since we have assumed the feasibility of the primal program, we have
\[
 \frac{1}{T} \sum_{t=1}^T \mathrm{DG}_t(\lambda)\ge \Delta.
\]
Combining it with the fact that $\mathrm{DG}_t$ is a monotonically increasing function, we conclude that the unique supremum of the dual function is reached when $\lambda = \lambda^* > 0$ such that
\[
\frac{1}{T} \sum_{t=1}^T \mathrm{DG}_t(\lambda^*) = \Delta.
\]
Substituting it into the expression of the Lagrangian dual, we have
\[
\phi(\lambda^*) = \sum_{t=1}^T D_{\text{kl}, t} (\lambda^*).
\]
Note that if we set all $\delta_1, \dots, \delta_T$ to be $\lambda^*$, then we have the primal target function to be exactly the same as $\phi(\lambda^*)$. We now conclude that the strong duality holds. Therefore, the optimum of the primal program is reached at
\[
\delta_1^* = \dots = \delta_T^* = \lambda^*.
\]
\end{proof}
\begin{remark}
We note that the strong duality can also be derived in a similar way to that in the proof of Proposition \ref{prop:free_k}. To see that, we can transform the program with respect to $\delta$ to the probability ratio $r$, forming a convex program. Then if we assume $\Delta < \frac1T \sum_{t=1}^T R_t$, we can directly get the strong duality from Slater's condition.
\end{remark}

\subsection{Proof of Lemma \ref{lemma:Lagrangian_dual}}
\begin{proof}[Proof of Part (a)]
We first examine the first-order condition of $L(\delta, \lambda)$ w.r.t. $\delta$. Since $\nu \coloneqq e^{\delta}$ is monotonically increasing w.r.t. $\delta$, we can check the first-order condition w.r.t. $\nu$ instead:
\begin{align*}
\frac{\partial L}{\partial \nu} & = \frac1T \sum_{t=1}^T \frac{\partial D_{\mathrm{kl},t}}{\partial \nu} - \lambda \cdot \frac1T \sum_{t=1}^T \frac{\partial \mathrm{DG}_t}{\partial \nu}\\
& = \frac1T \sum_{t=1}^T \frac{(1-G_t)G_t}{(G_t \nu - G_t + 1)^2} \cdot (\log(\nu) - \lambda)\\
& = \frac1T \sum_{t=1}^T \frac{(1-G_t)G_t}{(G_t \nu - G_t + 1)^2} \cdot (\delta - \lambda).
\end{align*}
Since $G_t \in [0, 1]$ and $\nu \geq 0$, we have the conclusion that the minimum of $L$ w.r.t. $\delta$ is taken when setting $\delta \leftarrow \lambda$, implying
\[
g(\lambda) = \inf_{\delta} L(\delta, \lambda) = L(\lambda, \lambda).
\]
\end{proof}

\begin{proof}[Proof of Part (b)]
The token-wise conclusion (for example, for $t_0$) can be easily verified by setting all $D_{\mathrm{kl}, t}$ and $\mathrm{DG}_t$ to be one specific token case (for example, replacing all $D_{\mathrm{kl}, t}$ and $\mathrm{DG}_t$ by $D_{\mathrm{kl}, t_0}$ and $\mathrm{DG}_{t_0}$). The decomposition automatically holds if we set all $\delta_t$ to be one specific value of $\lambda$.
\end{proof}

\begin{proof}[Proof of Part (c)]
By straightforward calculation,
\[
g_t(\lambda) = G_t \lambda + \Delta \lambda - \log(G_t e^{\lambda} - G_t + 1).
\]
One can check the first-order derivative by
\begin{align*}
\frac{\mathrm{d} g_t}{\mathrm{d} \lambda} & = G_t + \Delta - \frac{G_t e^{\lambda}}{G_t e^{\lambda} - G_t + 1}\\
& = \Delta + \frac{-(1-G_t)G_t e^{\lambda} + (1-G_t)G_t}{G_t e^{\lambda} - G_t + 1}\\
& = \Delta - \mathrm{DG}_t(\lambda).
\end{align*}
The concavity can be derived by the fact that $\mathrm{DG}_t(\cdot)$ is a monotonically increasing function for each $G_t \in [0, 1]$.
\end{proof}

\begin{proof}[Proof of Part (d)]
The strong duality can actually be derived from the proof of Proposition \ref{prop:free_t} (see Appendix \ref{subapd:proof_free_t}). However, we present the whole proof here for completeness.
From the conclusions of Part (b) and Part (c), we know that
\[
\frac{\mathrm{d} g}{\mathrm{d} \lambda} = \Delta - \frac1T \sum_{t=1}^T \mathrm{DG}_t(\lambda).
\]
If the primal problem \eqref{opt:nofree} is feasible, then $\Delta$ should be moderately small so that $\frac1T \sum_{t=1}^T \mathrm{DG}_t(\infty) > \Delta$. In that case, there always exists one single $\lambda^*$ that achieves the supremum of $g(\lambda)$ with
\[
\Delta - \frac1T \sum_{t=1}^T \mathrm{DG}_t(\lambda^*) = 0.
\]
Substituting the above into the original definition of Lagrangian \eqref{eq:Lagrangian}, we have
\begin{align*}
g(\lambda^*) & = \frac1T \sum_{t=1}^T \big( (G_t + \Delta) \cdot \lambda^* - \log(G_t e^{\lambda} - G_t + 1) \big) \\
& = \frac1T \sum_{t=1}^T \big( (G_t + \mathrm{DG}_t(\lambda^*)) \cdot \lambda^* - \log(G_t e^{\lambda} - G_t + 1) \big) \\
& = \frac1T \sum_{t=1}^T \big( \frac{G_t e^{\lambda}}{G_t e^{\lambda} - G_t + 1} \cdot \lambda^* - \log(G_t e^{\lambda} - G_t + 1) \big)\\
& = \frac1T \sum_{t=1}^T D_{\mathrm{kl}, t}(\lambda^*) \\
& \geq \frac1T \sum_{t=1}^T D_{\mathrm{kl}, t}(\delta^*) \\
& \geq \frac1T \sum_{t=1}^T f_t(\delta^*) \\
& = \inf_{\delta}\ f(\delta).
\end{align*}
Combining the inequality with the weak duality \citep{boyd2004convex} such that $\sup_{\lambda \geq 0} g(\lambda) \leq \inf_{\delta} f(\delta)$, we have the conclusion of strong duality. The equality $\delta^* = \lambda^*$ easily follows from Part (a).
\end{proof}

\subsection{Proof of Theorem \ref{thm:dualGA_performance_iid}}
\begin{proof}
Recall that we have defined the Lagrangian dual function as $\phi(\lambda)$. We also define the step-wise Lagrangian dual function
\[
\phi_t(\lambda) \coloneqq \inf_{\delta}\{D_{\mathrm{kl},t}(\delta)+\lambda(\Delta-\mathrm{DG}_t(\delta))\},
\]
where we directly have
\[
\phi(\lambda) = \frac1T \sum_{t=1}^T \phi_t(\lambda),
\]
and all infimums are taken at $\delta^*(\lambda) = \lambda$.
At every step, since we are setting $\delta_t = \lambda = \delta^*(\lambda)$, we have
\begin{equation}
D_{\mathrm{kl}, t}(\delta_t) = \phi_t(\lambda_t) + \lambda_t(\mathrm{DG}_t(\delta_t) - \Delta).
\label{eq:update_equality}
\end{equation}
Since we are running gradient ascent on the Lagrangian dual function $\phi_t$, we have
\[
\lambda_{t+1} = \lambda_t + \eta (\Delta - \mathrm{DG}_t(\delta_t)),
\]
which is identical to running a gradient descent algorithm on
\[\xi_t(\lambda) \coloneqq \lambda(\mathrm{DG}_t(\delta_t) - \Delta).\]
Since we have chosen $\eta = \Theta(1/\sqrt{T})$, by standard Online Gradient Descent analysis \citep{zinkevich2003online}, we have for any $\lambda^\prime \in [0, M]$,
\begin{align*}
\frac1T\sum_{t=1}^T \xi_t(\lambda_t) &\leq \frac1T\sum_{t=1}^T \xi_t(\lambda^\prime) + O(1/\sqrt{T})\\
& = \lambda^\prime\Big(\frac1T\sum_{t=1}^T \mathrm{DG}_t(\delta_t) - \Delta\Big) + O(1/\sqrt{T}).
\end{align*}
Setting $\lambda^\prime = 0$, we have
\begin{equation}
\frac1T\sum_{t=1}^T \xi_t(\lambda_t) \leq O(1/\sqrt{T}).
\label{eq:OGD_sqrt_T}
\end{equation}
Hence (defining $\bar{\phi} \coloneqq \mathbb{E}[\phi]$ and $\bar{\lambda}^* \coloneqq \argmax_{\lambda \in [0, M]} \bar{\phi}(\lambda)$)
\allowdisplaybreaks
\begin{align*}
\frac1T \sum_{t=1}^T \mathbb{E}[D_{\mathrm{kl},t}(\delta_t)] & = \frac1T \sum_{t=1}^T \mathbb{E}[\phi_t(\lambda_t)] + \frac1T \sum_{t=1}^T \mathbb{E}[\xi_t(\lambda_t)] \\
& \leq \frac1T \sum_{t=1}^T \mathbb{E}[\phi_t(\lambda_t)] + O(1/\sqrt{T})\\
& = \frac1T \sum_{t=1}^T \mathbb{E}\big[\mathbb{E}[\phi_t(\lambda_t)|\mathcal{F}_{t-1}]\big] + O(1/\sqrt{T})\\
& = \frac1T \sum_{t=1}^T \mathbb{E}\big[\bar{\phi}(\lambda_t)\big] + O(1/\sqrt{T})\\
& \leq \mathbb{E}\bigg[\bar{\phi}\Big(\frac1T \lambda_t\Big)\bigg] + O(1/\sqrt{T})\\
& \leq \bar{\phi}(\bar{\lambda}^*) + O(1/\sqrt{T})\\
& = \mathbb{E}[\phi(\bar{\lambda}^*)] + O(1/\sqrt{T})\\
& \leq \mathbb{E}[\phi(\lambda^*)] + O(1/\sqrt{T})\\
& = \frac1T \mathbb{E}\Big[\sum_{t=1}^T D_{\mathrm{kl},t}(\delta^*)\Big] + O(1/\sqrt{T})\\
& = \mathbb{E}\big[\mathrm{OPT}(\Delta)\big] + O(1/\sqrt{T}),
\end{align*}
where the first equality comes from the linearity of the expectation and \eqref{eq:update_equality}, the first inequality from the OGD analysis \eqref{eq:OGD_sqrt_T}, the second equality from the fact that $\lambda_t$ is adapted to $\mathcal{F}_{t-1}$ (that is, $\lambda_t$ is determined by previous $t-1$ steps' outcome), the third equality from that $\phi_t$ is uniquely determined by $G_t$ that is adapted to $\mathcal{F}_{t}$ and independent of previous $t-1$ steps' outcome by Assumption \ref{assum:iid}, the second inequality from the fact that $\phi_t$'s are i.i.d. according to Assumption \ref{assum:iid} and the concavity of the Lagrangian dual function, the third inequality from the optimality of $\bar{\lambda}^*$ w.r.t. $\mathbb{E}[\phi] = \mathbb{E}[\phi_t] = \bar{\phi}$ according to our definition, the fourth equality again from the i.i.d. assumption, the fourth inequality from the optimality of $\lambda^*$ w.r.t. $\phi = \frac1T \sum_{t=1}^T \phi_t$, the fifth equality from the strong duality of the program \eqref{opt:nofree}, and the last equality from the definition of $\mathrm{OPT}$ and Proposition \ref{prop:free_k}, \ref{prop:free_t}.

As for the constraint violation, we notice that
\[
\lambda_{t+1} - \lambda_t = \eta (\Delta - \mathrm{DG}_t(\delta_t)), \quad \forall t \in [T].
\]
Summing the above from $t=1$ to $T$, we have
\begin{align*}
M \geq \lambda_{T+1} - \lambda_1 & = \eta \sum_{t=1}^T (\Delta - \mathrm{DG}_t(\delta_t))\\
& = \Theta(\sqrt{T}) \sum_{t=1}^T (\Delta - \mathrm{DG}_t(\delta_t)),
\end{align*}
which verifies the proof.
\end{proof}

\subsection{Discussions on Choosing Green List Ratio $\gamma$}
\label{subapx:choose_green_ratio}
The Pareto optimality of the KL divergence-difference of green probability trade-off \eqref{opt:nofree} holds for the case in regards to any sequence of $\{G_t\}_{t=1}^T$. However, if the green list ratio $\gamma$ (that is, the expectation of $G_t$) is also considered as a decision variable, then the heuristic way of choosing $\gamma$ \citep{kirchenbauer2023watermark} is no more Pareto optimal. To derive the universal Pareto optimum, one needs to solve the optimization problem \eqref{opt:nofree} for each possible sample path $\{G_t\}_{t=1}^T$ and minimize its expectation to find the optimal $\gamma$. But it is generally a hard problem to precisely find the optimal $\gamma$ as the sequence of $\{G_t\}_{t=1}^T$ is random and cannot be foreseen in hindsight. Thus, we consider the \textit{certainty equivalent} version of the optimization problem \eqref{opt:nofree} where each $G_t$ is treated as $\gamma$ by taking the expectation, and derive the (approximate) optimal $\gamma$ as follows. To give the full details, we first define the certainty equivalent problem of \eqref{opt:nofree} as
\begin{align}
\begin{split}
\mathrm{OPT}_{\text{CE}}(\Delta, \gamma) \coloneqq \ \min_{\delta} \quad & \frac{\gamma \delta e^{\delta}}{\gamma e^{\delta} - \gamma + 1} - \log(\gamma e^{\delta} - \gamma + 1)\\
\text{s.t. } & \frac{(1-\gamma)\gamma (e^{\delta} - 1)}{\gamma e^{\delta} - \gamma + 1} \geq \Delta.
\end{split}
\label{opt:certainty_equivalent}
\end{align}
The optimal delta $\delta^*(\Delta, \gamma)$ of the certainty equivalent problem \eqref{opt:certainty_equivalent} can be calculated straightforwardly by considering the binding constraint such that
\[
\frac{(1-\gamma)\gamma (e^{\delta^*} - 1)}{\gamma e^{\delta^*} - \gamma + 1} = \Delta,
\]
where $\delta^*(\Delta, \gamma) = \log(1+\frac{\Delta}{\gamma}) + \log(1+\frac{\Delta}{1-\gamma-\Delta})$. Substituting that into the KL divergence as the objective, we have
\[
D_{\mathrm{kl}}(\delta^*(\Delta, \gamma)) = (\Delta + \gamma) \big[\log(\gamma + \Delta) - \log(\gamma)\big] + (1 - \Delta - \gamma)\big[\log(1 - \gamma - \Delta) - \log(1 - \gamma)\big].
\]
Calculating its first-order partial derivative w.r.t. $\gamma$, we have
\[
\frac{\partial D_{\mathrm{kl}}(\delta^*(\Delta, \gamma))}{\partial \gamma} = -\frac{\Delta}{\gamma} - \frac{\Delta}{1-\gamma} + \log(\gamma + \Delta) - \log(\gamma) - \log(1-\gamma-\Delta) + \log(1-\gamma).
\]
The second-order partial derivative w.r.t. $\gamma$ is
\[
\frac{\partial^2 D_{\mathrm{kl}}(\delta^*(\Delta, \gamma))}{\partial \gamma^2} = \frac{\Delta^2}{\gamma^2(\gamma+\Delta)} + \frac{\Delta^2}{(1-\gamma)^2 (1-\gamma-\Delta)} > 0.
\]
Hence, the optimal $\gamma^*$ to minimize the KL divergence is derived from the first-order condition s.t.
\begin{equation}
-\frac{\Delta}{\gamma^*} - \frac{\Delta}{1-\gamma^*} + \log(\gamma^* + \Delta) - \log(\gamma^*) - \log(1-\gamma^*-\Delta) + \log(1-\gamma^*) = 0.
\label{eq:certainty_equivalent}
\end{equation}
In practice, we numerically solve the above equation \eqref{eq:certainty_equivalent} and guide our choice of $\gamma$. We find our choice of $\gamma$ performs better than the heuristic choice in \citet{kirchenbauer2023watermark} (for example, see Figure \ref{fig:LFQA_Pareto}).

\section{Proofs in Section \ref{sec:KL_DISTORT}}
\label{apd:proof_sec2}
\subsection{Proof of Proposition \ref{prop:KL_decompose}}
\begin{proof}
Notice that we just need to prove the following inequality for any two distributions $\mathcal{Q}$ and $\mathcal{P}$ on $\mathcal{U} \times \mathcal{V}$ with $\mathcal{Q} \ll \mathcal{P}$:
\[
D_{\mathrm{kl}}(\mathcal{Q} \| \mathcal{P}) = D_{\mathrm{kl}}(\mathcal{Q}_{U} \| \mathcal{P}_{U}) + D_{\mathrm{kl}}(\mathcal{Q}_{V|U} \| \mathcal{P}_{V|U} | \mathcal{Q}_U).
\]
Such an equality is the direct result of the following computation
\allowdisplaybreaks
\begin{align*}
D_{\mathrm{kl}}(\mathcal{Q}_{U} \| \mathcal{P}_{U}) 
+ D_{\mathrm{kl}}(\mathcal{Q}_{V|U} \| \mathcal{P}_{V|U} | \mathcal{Q}_U) 
& = D_{\mathrm{kl}}(\mathcal{Q}_{U} \| \mathcal{P}_{U}) 
+ \mathbb{E}_{u \sim \mathcal{Q}_{U}} [D_{\mathrm{kl}}(\mathcal{Q}_{V|U=u} \| \mathcal{P}_{V|U=u})]\\
& = \int_{\mathcal{U}} \log\Big(\frac{\mathrm{d}\mathcal{Q}_{U}}{\mathrm{d}\mathcal{P}_{U}}\Big)\mathrm{d}\mathcal{Q}_{U}
+ \int_{\mathcal{U}} \mathrm{d}\mathcal{Q}_{U}(u) \int_{\mathcal{V}} \log\Big(\frac{\mathrm{d}\mathcal{Q}_{V|U=u}}{\mathrm{d}\mathcal{P}_{V|U=u}}\Big)\mathrm{d}\mathcal{Q}_{V|U=u}\\
& = \int_{\mathcal{U}}\mathrm{d}\mathcal{Q}_{U}(u) \int_{\mathcal{V}} \mathrm{d}\mathcal{Q}_{V|U=u}(v) \log\Big(\frac{\mathrm{d}\mathcal{Q}_{U}(u)}{\mathrm{d}\mathcal{P}_{U}(u)}\cdot \frac{\mathrm{d}\mathcal{Q}_{V|U=u}(v)}{\mathrm{d}\mathcal{P}_{V|U=u}(v)}\Big) \\
& = \int_{\mathcal{U}}\mathrm{d}\mathcal{Q}_{U}(u) \int_{\mathcal{V}} \mathrm{d}\mathcal{Q}_{V|U=u}(v) \log\Big(\frac{\mathrm{d}\mathcal{Q}(u,v)}{\mathrm{d}\mathcal{P}(u,v)}\Big) \\
& = \int_{\mathcal{U}\times \mathcal{V}} \mathrm{d}\mathcal{Q}(u, v) \log\Big(\frac{\mathrm{d}\mathcal{Q}(u,v)}{\mathrm{d}\mathcal{P}(u,v)}\Big)\\
& = D_{\mathrm{kl}}(\mathcal{Q} \| \mathcal{P}),
\end{align*}
where the first equality comes from the definition of conditional KL divergence, the second and the sixth from the definition of KL divergence, the third from the linearity of integrals, and the fourth and the fifth from the definition of conditional distribution.

Repeat the decomposition for $T_{\text{max}} - 1$ times and we shall finish the proof.
\end{proof}

\subsection{Proof of Proposition \ref{prop:info_lower_bound}}
Before we state the proof, we first define the total variation distance as
\begin{definition}
For any two distributions $\mathcal{Q}$ and $\mathcal{P}$ over the measurable space $(\Omega, \mathcal{F})$, the total variation distance is defined as
\[
\mathrm{TV} \coloneqq \sup_{A \in \mathcal{F}} \{|\mathcal{Q}(A) - \mathcal{P}(A)|\}.
\]
\end{definition}

\begin{proof}
Proposition \ref{prop:info_lower_bound} is a direct consequence of the following two lemmas.
\begin{lemma}[Le Cam's Lemma \citep{le2012asymptotic}]
\label{lemma:le_cam}
For any two distributions $\mathcal{Q}$ and $\mathcal{P}$ over the space $(\Omega, \mathcal{F})$, and denote $\psi$ as a measurable function from $\Omega$ to $\{0,1\}$. Then
\[
\inf_{\psi} \big\{\mathcal{Q}(\psi(\omega) = 0) + \mathcal{P}(\psi(\omega) = 1) \big\} = 1 - \mathrm{TV}(\mathcal{Q}, \mathcal{P}).
\]
Furthermore, such an infimum is met with the following function
\[
\psi^*(s) \coloneqq \mathbbm{1}\Big\{\frac{\mathrm{d}\mathcal{Q}}{\mathrm{d}\mathcal{P}}(s) \geq 1\Big\}.
\]
\end{lemma}
\begin{lemma}[Bretagnolle-Huber's Inequality \citep{bretagnolle1978estimation}]
\label{lemma:BH_ineq}
For any two distributions $\mathcal{Q}$ and $\mathcal{P}$, we have
\[
\mathrm{TV}(\mathcal{Q}, \mathcal{P}) \leq \sqrt{1 - e^{-D_{\mathrm{kl}}(\mathcal{Q}\|\mathcal{P})}}.
\]
\end{lemma}
\end{proof}

\subsection{Proof of Proposition \ref{prop:negative_DLoP}}
\begin{proof}
With the knowledge of the full generating process of $\bm{p}_{\text{eval}}(x)$, we can find the token sequence with the largest likelihood as
\[
\bm{y}^* \coloneqq \ \argmax_{\bm{y} \in \mathcal{V}^T} \prod_{t=1}^T \bm{p}_{\text{eval}, t|[t-1]}(x, \bm{y}_{[t-1]})_{y_t},
\]
where we assume there is no tie for simplicity. Then we can construct a modified model $\tilde{\bm{q}}$, s.t.
\[
\bm{q}(x,\bm{y}_{[t-1]})_k = 
\begin{cases}
1, \quad & \text{if } k = \bm{y}^*_t,\\
0, \quad & \text{otherwise}.
\end{cases}
\]
Then 
\begin{align}
\mathrm{LoP}(\tilde{\bm{q}}|\bm{p}_{\text{eval}},x) & = -\frac1T \max_{\bm{y} \in \mathcal{V}^T} \sum_{t=1}^T \log(\bm{p}_{\text{eval}, t|[t-1]}(x, \bm{y}_{[t-1]})_{y_t}) \nonumber \\
& \leq -\frac1T \mathbb{E}_{\bm{y}\sim \bm{p} }\big[\sum_{t=1}^T \log(\bm{p}_{\text{eval}, t|[t-1]}(x, \bm{y}_{[t-1]})_{y_t})\big] \nonumber \\
& = \mathrm{LoP}(\bm{p}|\bm{p}_{\text{eval}},x),
\label{eq:DLoP_nonpositve}
\end{align}
where the equality only holds if $\bm{p} = \tilde{\bm{q}}$.

By the assumption that $|\mathcal{V}^T| \geq 3$ and there is no tie in $\bm{p}_{\text{eval}}(x, \bm{y})$, we shall see that $\bm{p}(x,\cdot)$ cannot be a Dirac delta function while $\tilde{\bm{q}}(x,\cdot)$ itself is a Dirac delta function, implying
\[\bm{p} \neq \tilde{\bm{q}}.\]
Thus, we can conclude the proof with \eqref{eq:DLoP_nonpositve} and the fact that KL divergence is only zero if the two distributions are identical.
\end{proof}

\subsection{Proof of Proposition \ref{prop:KL_cvx}}
\begin{proof}
The proposition is the direct result of the following convexity lemma and Jensen's inequality.
\begin{lemma}[Convexity of KL divergence, e.g., Theorem 4.1 of \citet{polyanskiy2014lecture}]
\label{lemma:kl_cvx}
Kullback-Leibler divergence $D_{\mathrm{kl}}(\mathcal{Q} \| \mathcal{P})$ is convex for the joint argument $(\mathcal{Q}, \mathcal{P})$.
\end{lemma}
\end{proof}

\subsection{Computation of Example \ref{eg:ems}}
\begin{itemize}
    \item \textbf{Exponential minimum sampling.} The exponential minimum sampling generates a sequence of $r_{t,k} \sim \mathrm{Unif}[0, 1]$ random variables as the key. For any sequence of $r_{t,k}$'s, the next word is deterministically chosen as
\[
y_{t} = \argmin_{k \in [|\mathcal{V}|]} -\frac{\log(r_{t,k})}{p_{t,k}}.
\]
The intuition behind this sampling procedure is that the larger $p_{t,k}$ is, the larger the chance $k$ is chosen at the $t$-th token. 
\item \textbf{Inverse transform sampling.} The inverse transform sampling generates a random variable $r_t\sim \mathrm{Unif}[0, 1]$ as the key. Then for any $r_t$, the next word is deterministically chosen as 
$$y_t=\min\left\{k\in [|\mathcal{V}|]: \sum_{k'\in [|\mathcal{V}|]}p_{t,k'}\mathbbm{1}_{\{k'\leq k\}}\geq r_t\right\}.$$
The intuition behind this sampling procedure is that the $k$'th word is chosen only when $r_t$ is located in the interval $(\sum_{k'<k}p_{t,k'},\sum_{k'\leq k}p_{t,k'}]$, which has the probability $p_{t,k}$. We present a simplified version of inverse transform sampling for illustration. In the original algorithm of \citet{kuditipudi2023robust}, an arbitrary permutation over the vocabulary may also be applied. However, such a permutation does not affect the calculation presented below.
\end{itemize}

\citet{aaronson2023watermarking, kuditipudi2023robust, christ2023undetectable,li2025statistical} show that both exponential minimum sampling and inverse transform sampling are marginally distortion-free with respect to the key $r=r_{t,k}$ or $r=r_t$:
\begin{equation}
\mathbb{E}_{r}[\bm{q}_t] = \bm{p}_t
\label{eq:marginal_distort_free}
\end{equation}
While the distortion of the marginal model is zero, the expectation of its model distortion is never zero on any trajectory of the $r_{t,k}$ (or $r_t$). The watermarked model is now a delta distribution whose KL divergence is
\[
D_{\mathrm{kl},t} = \sum_{k\in [|\mathcal{V}|]} -\log(p_{t,k}) \mathbbm{1}\{k \text{ is chosen at }t\},
\]
which directly leads the our conclusion when combined with \eqref{eq:marginal_distort_free}:
\[
\mathbb{E}_{r_{t,k}} \big[D_{\mathrm{kl},t} \big] = \sum_{k\in [|\mathcal{V}|]} -\log(p_{t,k}) p_{t,k} = \mathrm{Ent}_{y_t}\big(\bm{p}_t\big).
\]
The above computation can be applied to any marginally distortion-free watermarking algorithm whose next-word sampling is deterministic if the (random) key $r$ is given.

\section{Proofs in Section \ref{sec:more_analyses}}
\subsection{Proof of Proposition \ref{prop:type_I&II}}
\begin{proof}
Note that the random variables $\mathbbm{1}\{\text{$t$-th token is green}\}$'s are independent and Bernoulli random variables. Thus, those two inequalities are the direct consequences of Hoeffding's inequality (which is stated below for completeness).
\begin{lemma}[Hoeffding's inequality \citep{hoeffding1994probability}]
Suppose $X_i$'s are $n$ independent random variables which are almost surely bounded by $[a_i, b_i]$'s. Then
\[
\mathbb{P}\Big(\frac1n \sum_{i=1}^n X_i - \frac1n \sum_{i=1}^n \mathbb{E}[X_i] \geq \varepsilon \Big) \leq \exp\Big(-2\cdot\frac{n^2\varepsilon^2}{\sum_{i=1}^n (b_i - a_i)^2}\Big).
\]
\end{lemma}
\end{proof}

\subsection{Proof of Proposition \ref{prop:adv_robust}}
\begin{proof}
Since the event that each token is green or not is a Bernoulli random variable and is uniformly bounded no matter which adversarial attack is taken, the high probability guarantee in Proposition \ref{prop:type_I&II} still holds as long as we can lower bound the expectation of the difference of the number of green tokens for the new sequence. We thus analyze the average difference of green list probability.
\begin{itemize}
\item Deletion. For the watermarked sequence, the expected number of green tokens is $(\Delta + \gamma) T$, while the expected number of red tokens is $(1-\Delta-\gamma)T$. Suppose the worst case happens where the deleted $T^\prime = l\cdot T$ tokens are all green. Then the expected green token ratio in the modified sequence is $\frac{\Delta + \gamma - l}{1 - l}$. Subtracting the green token ratio in the unwatermarked sequence (which is $\gamma$ on expectation), we complete the proof.
\item Insertion. Similar to the deletion case, the worst case of insertion happens when all the inserted $T^\prime = l \cdot T$ are red tokens. The next token after the insertion could also be adversarially changed to be a red token. Then the expected green token ratio in the modified sequence is at least $\frac{\Delta + \gamma + \frac1T}{1 + l}$. Subtracting the term $\gamma$ completes the proof.
\item Substitution. Similar to the insertion case, the worst case of substitution takes place if one substitutes $T^\prime = l \cdot T$ green tokens to be red and an additional green token is changed to red right after the substituted substring. The expected green token ratio in the modified sequence is at least $\Delta + \gamma - l - \frac1T$, leading to our conclusion.
\end{itemize}
\end{proof}

\section{Experiments Details}
\label{apx:experiments}
\subsection{Experiment Settings}
\label{subapx:exp_settings}
In this section, we provide more details on the experiment setting. 

\subsubsection{Dataset Construction}
\label{subapx:data_construct}
We extract the prompt data from the following two datasets:
\begin{itemize}
    \item Colossal Common Crawl Cleaned corpus \citep[abbreviated as C4]{raffel2020exploring}: Following the data processing method in \cite{kirchenbauer2023watermark}, we randomly select 500 samples from the ``text'' field of the ``realnewslike'' subset of the C4 dataset, under the condition that each text has more than 250 tokens. From each selected text, we create a prompt by trimming a fixed number of $200$ tokens from the end of the text and using the left tokens as the prompt. 
    \item Long-Form Question Answering \citep[abbreviated as LFQA]{fan2019eli5}:  We generate prompts by randomly selecting $500$ questions  from the ``prefix'' field of the dataset. The ``gold\_completion'' is used as the baseline completion. This generation method is the same as the one used in \cite{kirchenbauer2023reliability}.
\end{itemize}

\subsubsection{Watermarking Algorithms}
\label{subapx:algs_implement}
The implementation details of watermarking algorithms are presented below.
\begin{itemize}
    \item Soft Red List \citep[abbreviated as SRL]{kirchenbauer2023watermark}: As given in Algorithm \ref{alg:kirchen} (with $\delta_{t,k}\equiv \delta$ for some hyperparameter $\delta$), the SRL algorithm partitions the vocabulary list into a green/red list with ratio $\gamma/(1-\gamma)$ when sampling the next token $y_t$, where $\gamma$ is the hyperparameter deciding the green list ratio. The partition scheme is determined pseudo-randomly by a random number generator seeded by a hash of token $y_{t-1}$, where we use the same hash function employed in \citet{kirchenbauer2023watermark} in the experiments. The algorithm adds $\delta$ to the ``green-list'' logits and skews the sampling towards the green list, where $\delta$ is a pre-specified hyperparameter representing the watermarking strength. The next token $y_t$ is sampled from the softmax distribution of this ``skewed'' logit vector. We set the hash to be determined by the previous 1 token (also called 1-gram) as is done in \citet{kirchenbauer2023watermark}. We also conduct the beam search with 1 beam.
    \item Exponential Minimum Sampling 
    \citep[abbreviated as EMS]{aaronson2023watermarking}: EMS generates the next token by $y_t=\argmin_{k\in \mathcal{V}} -ln(r_{t,k})/p_{t,k}$, where $p_{t,k}$ is the original probability of token $k$ following the language model, and $\{r_{t,k}\}_{k=1}^{|\mathcal{V}|}$ are generated pseudo-randomly as $\mathrm{Uniform}(0,1)$ variables from a random generator whose seed is from the hash of the previous token $y_{t-1}$. The Gumbel-trick \citep{yellott1977relationship} guarantees that if $\{r_{t,k}\}_{k=1}^{|\mathcal{V}|}$ are indeed i.i.d. random (instead of pseudo-random) variable from $\mathrm{Uniform}(0,1)$ then $\argmin_{k\in \mathcal{V}} -ln(r_{t,k})/p_{t,k}$ marginally recovers the original distribution of $p_{t,k}$. As suggested in \cite{fernandez2023three}, we use the generating temperature $\tau$ (to adjust the actual logits used for generation) as a hyperparameter of EMS. Intuitively, a higher temperature flattens the original probability $p_{t,k}$ and thus reduces its power in deciding the final watermarked token. In this way, the watermarked token depends more on the generated key $r_{t,k}$ and thus is easier to detect.
    \item Unbiased Watermark (\citet{hu2023unbiased}, abbreviated as UW): Similar to EMS, the UW draws a pseudo‑random variable at every step $t$, seeded by a secret key $\mathcal{K}$ and the current context. It then reweights the model’s next‑token distribution in a way that guarantees marginally distortion-free, i.e., the next-token distribution is identical to the original language model when marginalized over the randomness of $\mathcal{K}$. UW proposes two reweighting schemes, both driven by the pseudo-random $\text{Uniform}(0,1)$ variable $E$ at each step: $\delta$‑reweight, which treats $E$ as a point on the unit interval and performs inverse‑CDF sampling, and the chosen token receives all probability mass for that step; and $\gamma$‑reweight, which uses $E$ to generate a permutation of the vocabulary, then discards the tokens corresponding to the left half of cumulative probability mass under that ordering, and only samples from the remaining tokens. For detection, UW computes a likelihood‑based score, so the verifier must recreate the original next‑token probabilities of the generating LM. In other words, effective detection assumes access to the underlying model logits.

    \item Adaptive Text Watermark (\citet{liu2024adaptive}, abbreviated as ATW): ATW is a green/red‑list watermarking scheme that introduces two key innovations: (1) Entropy‑gated insertion, which computes the entropy of the model’s next‑token distribution and inserts a watermark only when the entropy exceeds a pre-defined threshold $\alpha$. Intuitively, ATW only watermarks the high-uncertainty positions; (2) Semantic hashing for the green/red split, which, instead of using the raw context tokens to determine the vocabulary partition, ATW uses the semantic information extracted from the context. In the experiments, we implement their best-performing hyperparameter configurations: $\delta=1, \alpha=2$ and $\delta=5, \alpha=2$. 
    
    \item Dual Gradient Ascent (Algorithm \ref{alg:DualGA}, abbreviated as DualGA):  We use the same hash function as in SRL, and select the step size $\eta$ and the initial dual variable $\lambda_1$ based on the best-performing values on a small validation set of 16 samples. In practice, we find the DualGA algorithm is resilient to a relatively large step size (in the range $\eta\in [0.1, 1]$), and a small step size ($\eta<0.05$) makes the overall performance similar to the vanilla SRL. The $\lambda_1$ parameter is typically around $10\times \Delta$. The ratio $\gamma$ is chosen to minimize the certainty equivalent problem \eqref{opt:certainty_equivalent} given some $\Delta$ by numerically solving \eqref{eq:certainty_equivalent}. The experiments manually set the $\Delta$ values for demonstration purposes. However, in practical applications, $\Delta$ can be chosen to adapt to the specific generation length constraints and task requirements either theoretically or empirically.
\end{itemize}
\subsubsection{Detection Ability Metrics}
\label{subapx:detection_ps}
Detection of the watermark can be formulated as a hypothesis testing problem -- to test the null hypothesis $\mathcal{H}_0$: ``the text is unwatermarked'', against $\mathcal{H}_1$: ``the text is watermarked''. Following this formulation, all the algorithms share a similar detection paradigm: a score statistic is calculated from the text to be tested, and then a $p$-value is calculated based on the score. Here the $p$-value represents the probability of observing a score at least as extreme as the observed score under the null hypothesis and will be elaborated on for each algorithm.  A detection threshold set to match the false positive rate is applied. If the $p$-value falls below this threshold, the null hypothesis is rejected and the text is classified as watermarked; if not, it is classified as unwatermarked. Different watermarking algorithms adopt different score statistics.
\begin{itemize}
    \item SRL and DualGA: The original SRL \citep{kirchenbauer2023watermark} uses the $z$-score as the test statistic. Under the null hypothesis, the event whether $y_t$ belongs to the green list is i.i.d., implying that the $z$-score defined in \eqref{eq:z_score} asymptotically follows a standard normal distribution. Consequently, for a text yielding a $z$-score of $z_0$, its corresponding $p$-value is calculated by
    \begin{equation}\label{eq:normal_z_pcal}
        p_0 = \mathbb{P}(Z > z_0), \quad Z\sim \mathcal{N}(0, 1).
    \end{equation}
   However, \citet{fernandez2023three} challenges the definition which is built for the asymptotic behavior, by presenting evidence that the empirical FPRs significantly exceed theoretical predictions. The gap between the theoretical guarantees and the performance in practice suggests that the $z$-score might be further improved. To mitigate the gap, they use the number of tokens in the green list $S=|\bm{y}|_G$ as the score statistic. Under the null hypothesis, $S$ follows a binomial distribution with $T$ of trials and $\gamma$ as the success probability (denoted by $\mathcal{B}(T, \gamma)$). Thus, a text with score $|\bm{y}|_G=s_0$ has a $p$-value given by
    \begin{equation}\label{eq:three_brick_p}
         p_0 = \mathbb{P}(S>s_0),\quad S\sim \mathcal{B}(T, \gamma).
    \end{equation}
    In our experiments, $p$-values are calculated following \eqref{eq:three_brick_p}.
    
    \item EMS: Detecting an EMS watermark needs to rerun the generating process of $\{r_{t,k}\}_{k=1}^{|\mathcal{V}|}$ with the given random seed for each step $t$, and calculate $-\log (1-r_{t, y_t})$ for each observed token $y_t$. Under the null hypothesis, variables $\{r_{t, y_t}\}_{t=1}^T$ can be treated as i.i.d. samples from the uniform distribution $\mathrm{Uniform}(0,1)$, thus $\{-\log (1-r_{t, y_t})\}_{t=1}^T$ are i.i.d. sampled from an exponential distribution with mean $1$. The originally proposed $z$-score for EMS is given by
    \begin{equation}\label{eq:ems_zscore}
        Z=\dfrac{S_T- T}{\sqrt{T}}, \quad \text{with } S_T \coloneqq \sum_{t=1}^T -\log (1-r_{t, y_t}).
    \end{equation}
    By the central limit theorem, the $z$-score defined in \eqref{eq:ems_zscore} asymptotically follows a standard normal distribution under the null hypothesis, and correspondingly the $p$-value can be calculated by \eqref{eq:normal_z_pcal}. But for the same reason as in SRL, \citet{fernandez2023three} suggests considering the finite sample distribution, where the $S_T$ follows a gamma distribution $\Gamma(T,1)$ under $\mathcal{H}_0$. For a text with a score $s_0$, its $p$-value is calculated through
    \begin{equation}\label{eq:ems_pval}
        p_0 = \mathbb{P}(S_T > s_0),\quad S_T\sim \Gamma(T, 1).
    \end{equation}
    In the experiments, our $p$-values of EMS are calculated from \eqref{eq:ems_pval}.
\end{itemize}

\subsection{Additional Experiment Results and Discussions}
\label{subapx:additional_exp}
In this subsection, we present the additional experiments omitted in the main body of the paper.

\subsubsection{Examples of Watermarked Responses}
\label{apxsub:Example_response}
We provide several examples of responses generated by the DualGA algorithm, highlighting each token with green or red coloring. Alongside the watermarked responses, we include the corresponding oracle (human) completions provided by the dataset, which are also color-coded using the same detection algorithm. The generated response does not significantly degrade in sentiment quality, while containing a much larger proportion of green words compared to the unwatermarked counterpart.

\begin{figure}[htbp]
  \centering

    \begin{tcolorbox}[colback=blue!5, colframe=blue!40!black, 
                      sharp corners] 
    \footnotesize
    \textbf{Prompt:} Were there any civilizations in the medieval period that actually lived inside of swamps?\\

\textbf{Watermarked Response:} \textcolor{green}{Yes, there were several civil}\textcolor{red}{izations}\textcolor{green}{in medieval period that lived }\textcolor{red}{inside or near swamps}\textcolor{green}{. Swamps were often }\textcolor{red}{seen }\textcolor{green}{as marginal or un}\textcolor{red}{product}\textcolor{green}{ive land that could be difficult to inhabit due to their wet and humid climate, poor }\textcolor{red}{drainage}\textcolor{green}{, and the }\textcolor{red}{presence }\textcolor{green}{of }\textcolor{red}{insects}\textcolor{green}{, sn}\textcolor{red}{akes}\textcolor{green}{, }\textcolor{red}{and }\textcolor{green}{other dangerous wildlife. }\textcolor{red}{However}\textcolor{green}{, some civilizations adapted} ... \\

\textbf{Human Response:} \textcolor{red}{Ven}\textcolor{green}{ice would }\textcolor{red}{be your major }\textcolor{green}{example}\textcolor{red}{! }\textcolor{green}{
}\textcolor{red}{
}\textcolor{green}{It }\textcolor{red}{was build }\textcolor{green}{on }\textcolor{red}{a mal}\textcolor{green}{arial }\textcolor{red}{sw}\textcolor{green}{amp }\textcolor{red}{in }\textcolor{green}{the }\textcolor{red}{late classical and }\textcolor{green}{early }\textcolor{red}{medieval era and }\textcolor{green}{is still }\textcolor{red}{to }\textcolor{green}{this }\textcolor{red}{day still built }\textcolor{green}{on }\textcolor{red}{a}\textcolor{green}{l}\textcolor{red}{ago}\textcolor{green}{on. Which is why }\textcolor{red}{it }\textcolor{green}{is s}\textcolor{red}{inking... }\textcolor{green}{Martin}\textcolor{red}{, John Jeffries}\textcolor{green}{, }\textcolor{red}{and Dennis Romano}\textcolor{green}{, }\textcolor{red}{ed}\textcolor{green}{s. }\textcolor{red}{Ven}\textcolor{green}{ice }\textcolor{red}{Reconsidered: }\textcolor{green}{The History }\textcolor{red}{and Civilization of }\textcolor{green}{an }\textcolor{red}{Italian City-}\textcolor{green}{State, 12}\textcolor{red}{9}\textcolor{green}{7}\textcolor{red}{–179}\textcolor{green}{7}\textcolor{red}{. J}\textcolor{green}{H}\textcolor{red}{U Press}\textcolor{green}{, }\textcolor{red}{}\textcolor{green}{200}\textcolor{red}{2}

\end{tcolorbox}
\end{figure}

\begin{figure}[htbp]
  \centering
  \begin{tcolorbox}[colback=blue!5, colframe=blue!40!black, 
                      sharp corners] 
    \footnotesize
\textbf{Prompt:} Did Robert E. Lee ever say, about the Civil War, that he wished he owned every slave so that he could free them and avoid war?\\

\textbf{Watermarked Response:} \textcolor{red}{No}\textcolor{green}{, }\textcolor{red}{there }\textcolor{green}{is no historical }\textcolor{red}{evidence to }\textcolor{green}{support this claim. Lee did not express such a}\textcolor{red}{sentiment}\textcolor{green}{, and the claim has likely been }\textcolor{red}{invent}\textcolor{green}{ed or  }\textcolor{red}{dist}\textcolor{green}{orted by modern partisans or }\textcolor{red}{ideolog}\textcolor{green}{ues trying to reframe the past to }\textcolor{red}{suit }\textcolor{green}{their own political goals. Lee himself opposed }\textcolor{red}{sla}\textcolor{green}{very and freed }\textcolor{red}{some }\textcolor{green}{of }\textcolor{red}{his }\textcolor{green}{own slaves during } ...\\

\textbf{Human Response:} \textcolor{red}{No, Robert E}\textcolor{green}{. }\textcolor{red}{Lee never }\textcolor{green}{said }\textcolor{red}{this. }\textcolor{green}{This is }\textcolor{red}{a mis-quote from an }\textcolor{green}{"inter}\textcolor{red}{view}\textcolor{green}{" that a reverend named }\textcolor{red}{John Leyburn claimed }\textcolor{green}{to }\textcolor{red}{have had }\textcolor{green}{with }\textcolor{red}{Robert E}\textcolor{green}{. }\textcolor{red}{Lee }\textcolor{green}{in 1869}\textcolor{red}{, }\textcolor{green}{several years after }\textcolor{red}{the war was }\textcolor{green}{over, }\textcolor{red}{and a year }\textcolor{green}{before }\textcolor{red}{Lee}\textcolor{green}{'s }\textcolor{red}{death. }\textcolor{green}{}\textcolor{red}{John Ley }\textcolor{green}{burn had been born in Lexington}\textcolor{red}{, Virginia}\textcolor{green}{, }\textcolor{red}{and was }\textcolor{green}{a }\textcolor{red}{Pres}\textcolor{green}{byterian pastor}\textcolor{red}{. }\textcolor{green}{At the outbreak }\textcolor{red}{of the }\textcolor{green}{Civil War} ...

    \end{tcolorbox}

\end{figure}

\subsubsection{Runtime Overhead of Watermarking Algorithms}\label{subapx:computational_table}

In Table~\ref{tab:watermark-runtime}, we compare the generation time of the watermarking algorithms considered in the main text (and detailed in Appendix \ref{subapx:algs_implement}). We use Llama‑3‑8B as the LM for generation. All algorithms are tested on a single NVIDIA L40S GPU, with batch size $=16$ and the maximum generation length $=300$ tokens. Our DualGA algorithm incurs a 9.7\% overhead relative to the unwatermarked baseline, indicating that the majority of the computation time is dominated by the underlying language model inference rather than the watermarking procedure. Moreover, DualGA demonstrates competitive runtime performance compared to other benchmark methods.

\begin{table}[t]
  \centering
  \begin{tabular}{lrr}
    \toprule
    \textbf{Algorithm} & \textbf{Time (s)} & \textbf{Overhead (\%)}\\
    \midrule
    No watermark & 704 & --\\
    SRL          & 745 & 5.8\\
    EMS          & 816 & 15.9\\
    UW           & 726 & 3.1\\
    ATW          & 2\,047 & 190.7\\
    DualGA (ours) & 772 & 9.7\\
    \bottomrule
  \end{tabular}
  \caption{End‑to‑end generation time for each watermarking algorithm (16 samples, max 300 tokens) and the percentage overhead relative to the un‑watermarked baseline.}
    \label{tab:watermark-runtime}
\end{table}

\subsubsection{Replication of Table \ref{Tab:exp_performance} Using Llama-7B}
\label{subapx:addi_detectVSdistort}

We evaluate the performance of the algorithms listed in Table \ref{Tab:exp_performance}, replacing Llama-3-8B with Llama-7B as the language model. The results are presented in Table \ref{Tab:exp_performance_ext}. Interestingly, despite the significant difference in generation quality between Llama-7B and Llama-3-8B, the performance of the watermarking algorithms remains largely consistent across the two LMs.

\begin{table}[h!]
  \centering
\resizebox{\textwidth}{!}{
  \begin{tabular}{ll|ccc|ccc}
    \toprule
    \multirow{3}{*}{Method} & \multirow{3}{*}{Configuration}&\multicolumn{3}{c|}{C4} &\multicolumn{3}{c}{LFQA}  \\
    & & \multicolumn{2}{c}{TPR $\uparrow$}& \multirow{2}{*}{KL$\downarrow$} & \multicolumn{2}{c}{TPR$\uparrow$}& \multirow{2}{*}{KL$\downarrow$} \\ 
    \cmidrule(lr){3-4} \cmidrule(lr){6-7}
& &FPR$<10^{-4}$ &FPR$<10^{-6}$& &FPR$<10^{-4}$ &FPR$<10^{-6}$& \\
    \midrule
\multirow{9}{*}{SRL}   
& $\delta=1, \gamma=0.1$&0.15 & 0.04 & 0.04 & 0.31 & 0.10 & 0.04 \\
 &  $\delta=1,\gamma=0.5$&0.49 & 0.18 & 0.07 & 0.56 & 0.22 & 0.07 \\
 & $\delta=1,\gamma=0.7$&0.22 & 0.04 & 0.05 & 0.33 & 0.11 & 0.05 \\
& $\delta=5,\gamma=0.1$&0.99 & 0.98 & 1.3 & 1.0 & 0.98 & 1.3 \\
   & $\delta=5,\gamma=0.5$&1.0 & 1.0 & 0.70 & 1.0 & 1.0 & 0.69 \\
   & $\delta=5,\gamma=0.7$&0.96 & 0.92 & 0.39 & 0.98 & 0.96 & 0.39 \\
& $\delta=10,\gamma=0.1$&1.0 & 0.99 & 1.7 & 1.0 & 0.98 & 1.5 \\
& $\delta=10,\gamma=0.5$&1.0 & 0.99 & 1.1 & 0.99 & 0.99 & 0.90 \\
& $\delta=10,\gamma=0.7$&0.99 & 0.96 & 0.62 & 0.96 & 0.94 & 0.51 \\
    \midrule
 \multirow{3}{*}{EMS} 
 & $\tau=0.5$&0.44 & 0.21 & 0.65 & 0.26 & 0.16 & 0.47 \\
 & $\tau=1$&0.94 & 0.87 & 0.98 & 0.94 & 0.94 & 0.91 \\
 & $\tau=1.5$&0.98 & 0.98 & 2.4 & 1.0 & 1.0 & 1.6 \\
    \midrule
\multirow{2}{*}{UW} 
 & $\delta$-reweight & 1.0 & 1.0 & 2.0 & 1.0 & 0.99 & 1.6 \\
 & $\gamma$-reweight & 1.0 & 0.99 & 1.3 & 0.99 & 0.99 & 1.1 \\
 \midrule
 \multirow{2}{*}{ATW} 
 & $\delta=1,\alpha=2$ & 0.34 & 0.13 & 0.05 & 0.41 & 0.27 & 0.04\\
 & $\delta=5,\alpha=2$ & 0.93 & 0.90 & 0.32 & 0.93 & 0.91 & 0.31 \\
\midrule 
 \multirow{4}{*}{DualGA} 
 & $\Delta=0.2$&0.90 & 0.70 & 0.15 & 0.91 & 0.72 & 0.13 \\
 & $\Delta=0.3$&0.97 & 0.96 & 0.33 & 0.98 & 0.96 & 0.30 \\
 & $\Delta=0.4$&1.0 & 0.98 & 0.61 & 1.0 & 1.0 & 0.55 \\
 & $\Delta=0.5$ &1.0 & 1.0 & 0.93 & 1.0 & 0.98 & 0.83 \\
    \bottomrule
  \end{tabular}}
  \caption{The detection ability (TPR) and model distortion (KL) of watermarking algorithms across different hyperparameter configurations, using Llama-7B as the LM. The TPR under different FPR thresholds measures the detection ability (the higher the better). The KL measures the distortion of the watermarked text (the lower the better).}
  \label{Tab:exp_performance_ext}
\end{table}

\subsubsection{Pareto Optimality and Stability of DualGA}
\label{subapx:stability_apx}

In Figure \ref{fig:LFQA_boxplot_zscore}, we include a performance comparison between SRL, EMS and our DualGA algorithms on the C4 dataset, employing $z$-scores (defined in Appendix \ref{subapx:detection_ps}) as a measure of detection ability since EMS does not divide the vocabulary into green/red lists (and thus no DG). Figure \ref{fig:LFQA_boxplot_zscore} uses boxplots to show the detection ability across algorithms, where the horizontal axis's location is computed by the mean of KL divergences over 500 samples. It highlights DualGA's stability in detection ability, as evidenced by its smaller dispersion in $z$-scores compared to benchmark algorithms.

\begin{figure}[!htb]
  \centering
  \includegraphics[width=1\linewidth]{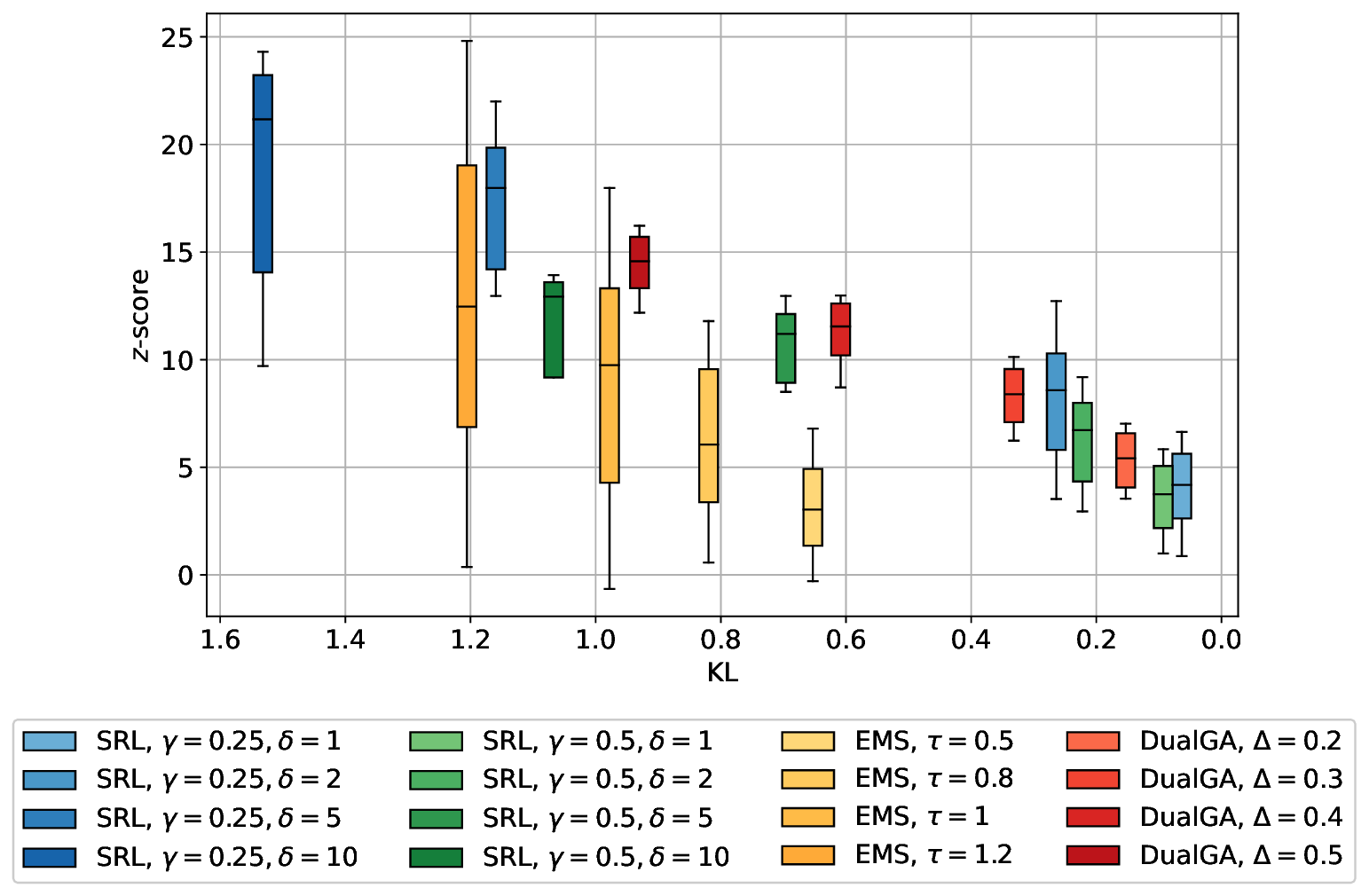}
  \caption{The dispersion in detection ability across different algorithms. Boxplots (showing the 0th, 10th, 50th, 90th, and 100th percentiles) are applied to represent the distributions of $z$-scores calculated from 500 samples of C4, illustrating the dispersion in detection ability ($z$-scores) among several algorithms. The x-axis is the averaged KL across 500 test prompts. }
  \label{fig:LFQA_boxplot_zscore}
\end{figure}

\subsubsection{Details of Watermark-Removal Attacks}
\label{subapx:robustness_attack}
We provide the implementation details of the watermark-removal attack considered in Section \ref{subsec:body_text_robust}. 

\begin{itemize}
    \item Deletion: A certain percentage of tokens, referred to as the Attack Rate, is randomly selected from the watermarked text and deleted to produce the text for detection.
    \item Insertion: We first randomly sample a position in the watermarked text and insert a randomly sampled token from the vocabulary list. This process is repeated a number of times equal to the Attack Rate multiplied by the original length of the watermarked text.
    \item Substitution:  Similar to the deletion approach, a proportion of tokens specified by the Attack Rate is randomly selected from the watermarked text. These tokens are then replaced with randomly chosen tokens from the vocabulary to create the text.
    
    \item Paraphrasing: To perform a paraphrasing attack with a controllable Attack Rate, we directly query GPT-4 via the OpenAI API using the prompt: ``Paraphrase the sentence, replacing only ... percent of the words'', where the specified percentage corresponds to the desired Attack Rate. 
    
    \item Entropy-based Substitution: Use Llama-3-8B to estimate the conditional entropy at each position and replace the token once the entropy exceeds a predefined threshold. This threshold is set to ensure the desired Attack Rate. The token is replaced with the most probable candidate from the red list.

    \item Color-based Substitution: This method first identifies the green words in a sentence. It then queries GPT-4 to generate paraphrases by replacing only those green words. The process is repeated 8 times, and the paraphrase with the highest number of green words successfully replaced by red words is selected.    
\end{itemize}

\subsubsection{Detectability with Sentence Length}

Figure \ref{fig:z_dynamic_w_t} shows the dynamics of the $z$-score for the unwatermarked sentence and watermarked sentences under varying substitution‐attack levels, as a function of sentence length $T_{\text{max}}$. The dynamic of $z$-score is shown for the responses of a single prompt, but the trend generally holds for most of the data samples. The result shows that the $z$-score of the unwatermarked sentence fluctuates within the 95\% confidence interval ($[-1.96,1.96]$). In contrast, the $z$-scores for watermarked sentences increase with $T_{\max}$, indicating a higher detection ability. This aligns with the intuition that longer watermarked texts are easier to detect and more resistant to watermark attacks. Specifically, the watermark remains detectable (with FPR$<5\%$) under a random substitution of less than $30\%$ of the total tokens, but loses detectability when more than $50\%$ of the tokens are attacked.

\begin{figure}[!htb]
  \centering
  \includegraphics[width=1\linewidth]{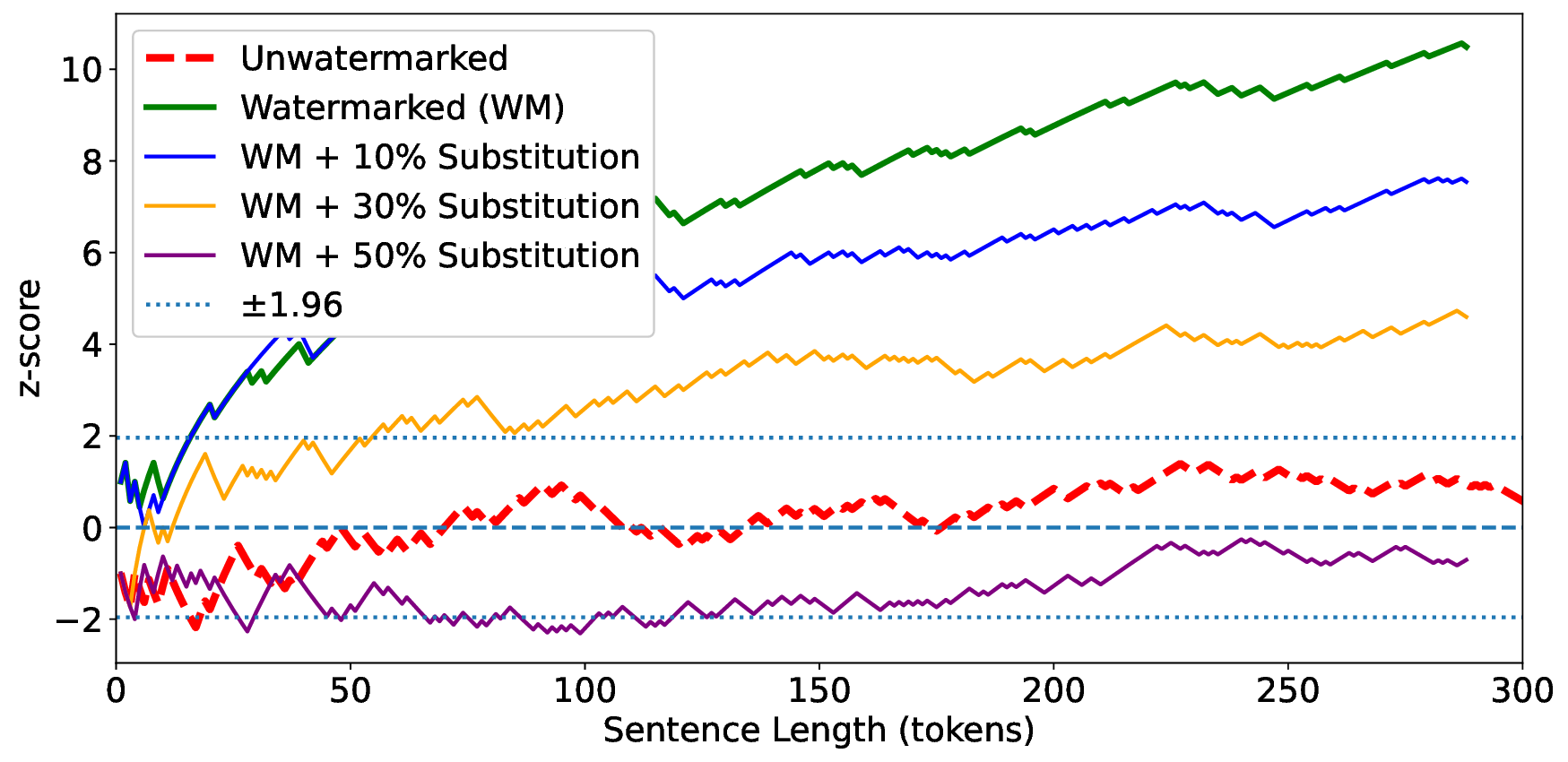}
  \caption{Evolution of the $z$-score with sentence length (in tokens) for the unwatermarked response and watermarked responses under varying substitution‐attack levels. As the sentence length increases, the watermarked text’s $z$-score rises, indicating a higher likelihood of detection. With up to $30\%$ of tokens randomly replaced, the watermark remains detectable beyond 50 tokens; however, it becomes undetectable once more than $50\%$ of tokens are substituted.}
  \label{fig:z_dynamic_w_t}
\end{figure}

\subsubsection{Repetition Detection via Dual-Variable Dynamics} \label{subsec:body_text_repetition_detection}

Most (if not all) of the existing watermarking algorithms rely on hashing-based rules for generating watermarks. For example, the green word-based watermarking algorithms generate the partition of the vocabulary for each token based on a hashing-based pseudo-random function. This type of hashing-based design can cause the issue of text repetition, as empirically observed by \citet{kuditipudi2023robust}. Specifically, if a sequence of tokens typically appears together and the hashing code for the final token in the sequence happens to skew the vocabulary distribution significantly towards the first token, this can lead to the sampling of a repetitive chunk, creating a cyclical pattern (illustrated in Figure \ref{fig:body_text_dual_lambda} (left)). While such cycles may emerge infrequently at the population level, their detection with the subsequent fixing is crucial to ensure quality on the individual (prompt) level. Traditional global metrics like entropy or KL-divergence are insufficient for identifying these repetitions, as they might happen only two or three times before the generating process exits the cycle.

\begin{figure}[htbp]
  \centering
  \begin{minipage}[b]{0.6\textwidth}
    \begin{tcolorbox}[colback=blue!5, colframe=blue!40!black, 
                      sharp corners] 
    \footnotesize
    \textbf{Prompt:} Why is the ``Canaanite shift'', or the shift of ``a'' to ``o'', more complete in Ashkenazi and Yemenite Hebrew than in Sephardi and modern Israeli Hebrew?\\

      \textbf{Generated Response:} The reason for the contrast between Mizrachi and European Jews ... pronunciation of the shahar or of the shetaf is \textcolor{red}{linguistics, and nothing else}, just the same for linguistics and nothing else for the Russian reflexive suffix *ynyj' is linguistics and nothing else, just the same for linguistics and nothing else for the Arabic dable (debt, lend); \textcolor{red}{linguistics, and nothing else}, just the same for linguistics and nothing else for the Russian reflexive ...

    \end{tcolorbox}
  \end{minipage}\hfill
  \begin{minipage}[b]{0.39\textwidth}
    \centering
    \includegraphics[width=\linewidth]{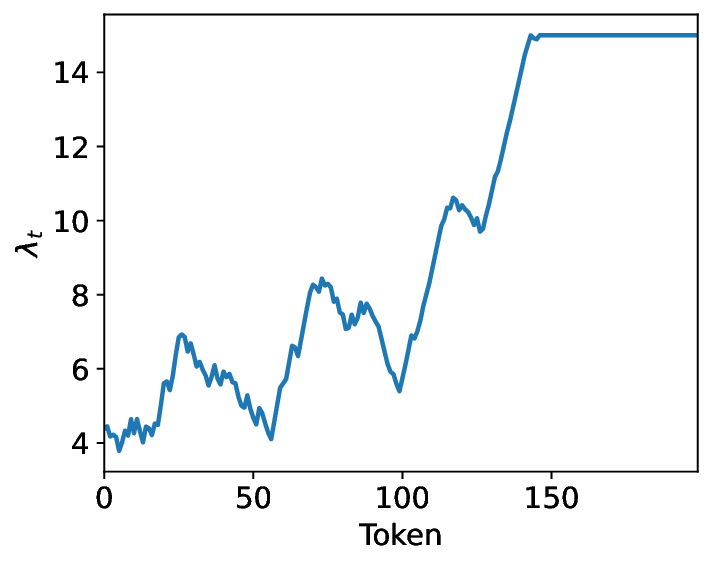} 
  \end{minipage}
  \caption{Left: Example of repeated token chunks generated by DualGA using Llama-3-8B. The initial non-repetitive text and the full extent of the repeated sequences are abbreviated with “...”. Red words indicate the start of each repeated chunk. Right: Abnormal rise in ${\lambda_t}$. Values of $\lambda_t$ that exceed 15 are truncated to 15 to ensure algorithmic stability.}
  \label{fig:body_text_dual_lambda}
\end{figure}

Figure \ref{fig:body_text_dual_lambda} (right) plots the dual variables $\lambda_t$ of our algorithm for the sequence generated on Figure \ref{fig:body_text_dual_lambda} (left). We can see the dual variable significantly jumps up when the repetition starts. This phenomenon is indeed general (we provide more examples in Figure \ref{fig:apx_rep_exps} ) and reasonable. The rationale is that, generally, a token sequence appearing with high certainty may pose a great challenge to the watermarking process: modest watermarking strength (smaller dual variable $\lambda_t$) may not suffice to override the certainty inherent in selecting the subsequent word. Therefore, the algorithm needs to increase the dual variable significantly. The repetitions can be viewed as a very deterministic generation that requires increasing the dual variable. In this light, the dual variable works as a monitor for the repetition phenomenon of the watermarking process.

Below we give more examples  of text repetition in Figure \ref{fig:apx_rep_exps}. The yellow blocks on the left show the generated text, while the record curves on the right show the dynamics of $\lambda_t$ with step $t$. We select from the group of text generated from DualGA with $\Delta=0.5$ whose $\{\lambda_t\}$ series ever reaches above $12$, and in all 4 examples that we pick out, text repetition is observed.

\begin{figure}[ht!]
    \begin{adjustbox}{valign=t}
    \begin{minipage}{0.5\textwidth}
        \centering
        \includegraphics[width=\textwidth]{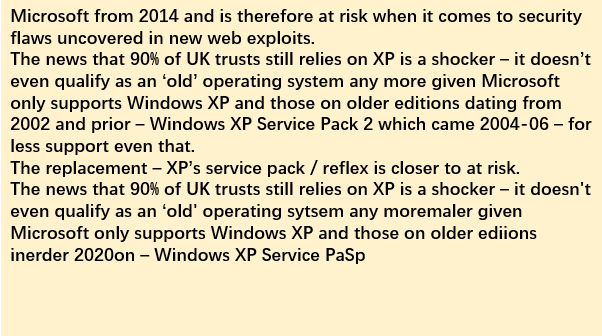} 
    \end{minipage}
    \end{adjustbox}
    \hfill
    \begin{adjustbox}{valign=t}
    \begin{minipage}{0.4\textwidth}
        \centering
        \includegraphics[width=\textwidth]{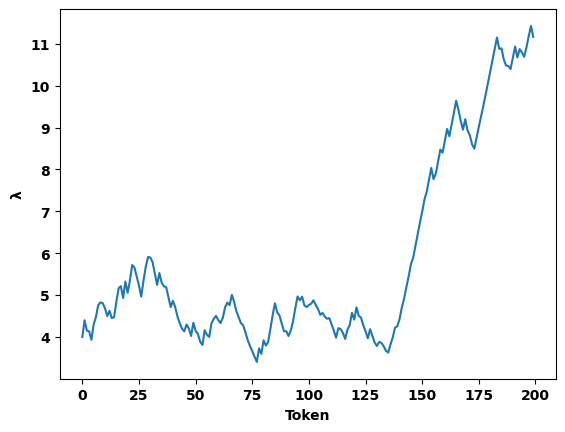}
    \end{minipage}
    \end{adjustbox}
    
        \begin{adjustbox}{valign=t}
    \begin{minipage}{0.5\textwidth}
        \centering
        \includegraphics[width=\textwidth]{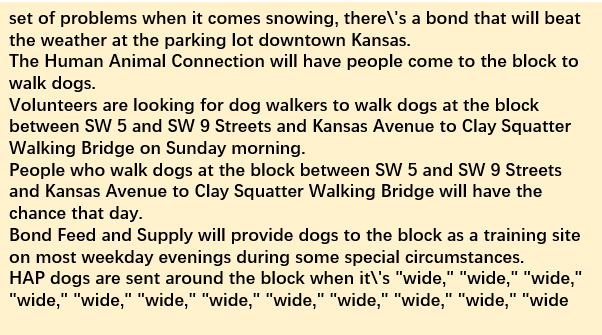} 
    \end{minipage}
    \end{adjustbox}
    \hfill
    \begin{adjustbox}{valign=t}
    \begin{minipage}{0.4\textwidth}
        \centering
        \includegraphics[width=\textwidth]{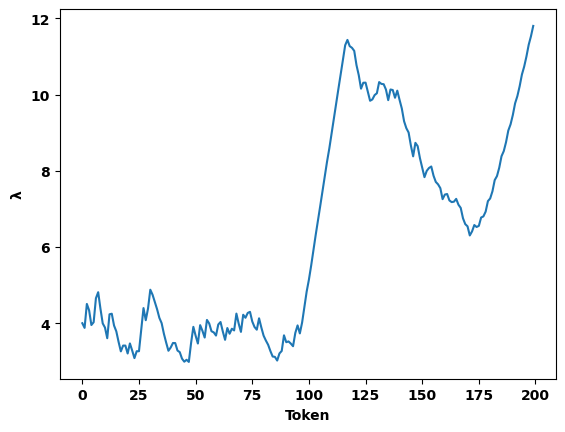}
    \end{minipage}
    \end{adjustbox}
        \begin{adjustbox}{valign=t}
    \begin{minipage}{0.5\textwidth}
        \centering
        \includegraphics[width=\textwidth]{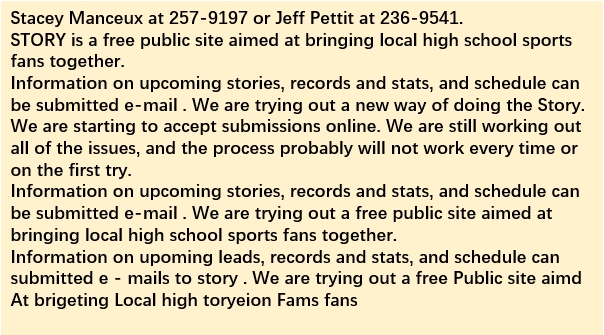} 
    \end{minipage}
    \end{adjustbox}
    \hfill
    \begin{adjustbox}{valign=t}
    \begin{minipage}{0.4\textwidth}
        \centering
        \includegraphics[width=\textwidth]{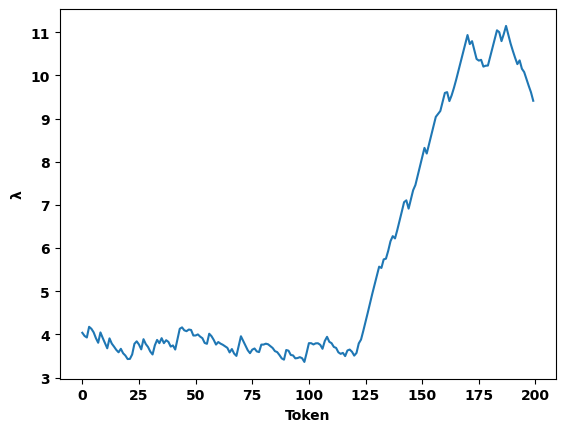}
    \end{minipage}
    \end{adjustbox}
    \begin{adjustbox}{valign=t}
    \begin{minipage}{0.5\textwidth}
        \centering
        \includegraphics[width=\textwidth]{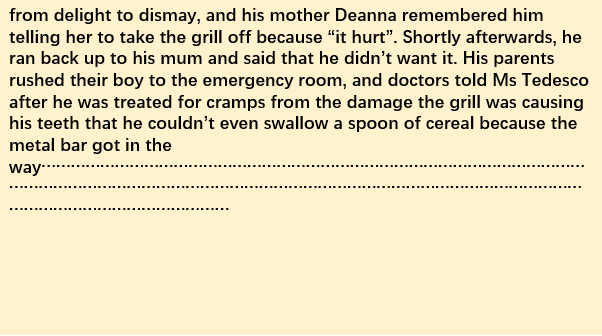}
    \end{minipage}
    \end{adjustbox}
    \hfill
    \begin{adjustbox}{valign=t}
    \begin{minipage}{0.4\textwidth}
        \centering
        \includegraphics[width=\textwidth]{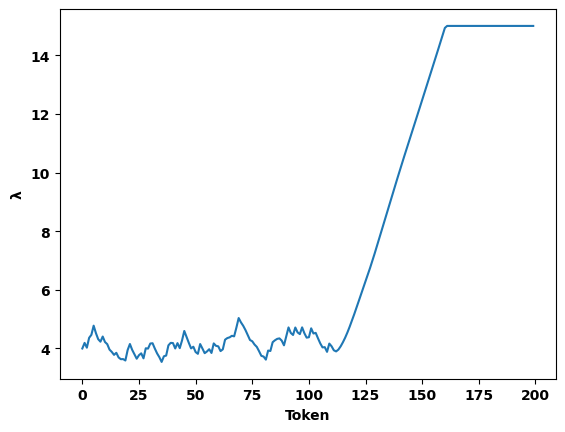}
    \end{minipage}
    \end{adjustbox}
    \caption{Examples for the repetitions (left) and the abnormal consistent rise in $\lambda_t$ (right).}
    \label{fig:apx_rep_exps}
\end{figure}

\subsubsection{Sensitivity Study on Hyperparameters}
\label{appsec:robust_hyper}
In addition to the robustness across different prompts discussed in Section~\ref{sec:Consistent_Performance_exp}, another aspect of DualGA’s consistent performance is its robustness to its hyperparameter settings. As briefly mentioned in the previous section, the $\Delta$ parameter of DualGA serves as a handle to balance the trade-off but does not degrade the Pareto optimality. The $\gamma$ is set by solving \eqref{eq:certainty_equivalent} in Appendix \ref{subapx:choose_green_ratio}. The only hyperparameters that may affect the Pareto optimality are the initial dual variable $\lambda_1$ and the step size $\eta$. Empirically, we select these parameters by testing on a small set of 16 samples and choosing the configuration that performs best. Across all our experiments, the resulting $\eta$ always falls within the range $[0.1, 2]$, and $\lambda_1$ is typically around $10 \times \Delta$. In Figure \ref{fig:box_4}, we examine the robustness of DualGA across various choices of $\lambda_1$ and $\eta$. Within a reasonable range of parameter values (specifically, $\lambda_1 \in \{1, 2, 5\}$ and $\eta \in \{0.1, 1, 5\}$), the algorithm demonstrates stable performance and maintains the distortion levels (KL) close to optimal. However, for more extreme parameter settings, such as $\lambda_1 = 10$ and $\eta = 10$, the KL is noticeably suboptimal.

\begin{figure}[!htb]
\centering
\includegraphics[width=0.8\linewidth]{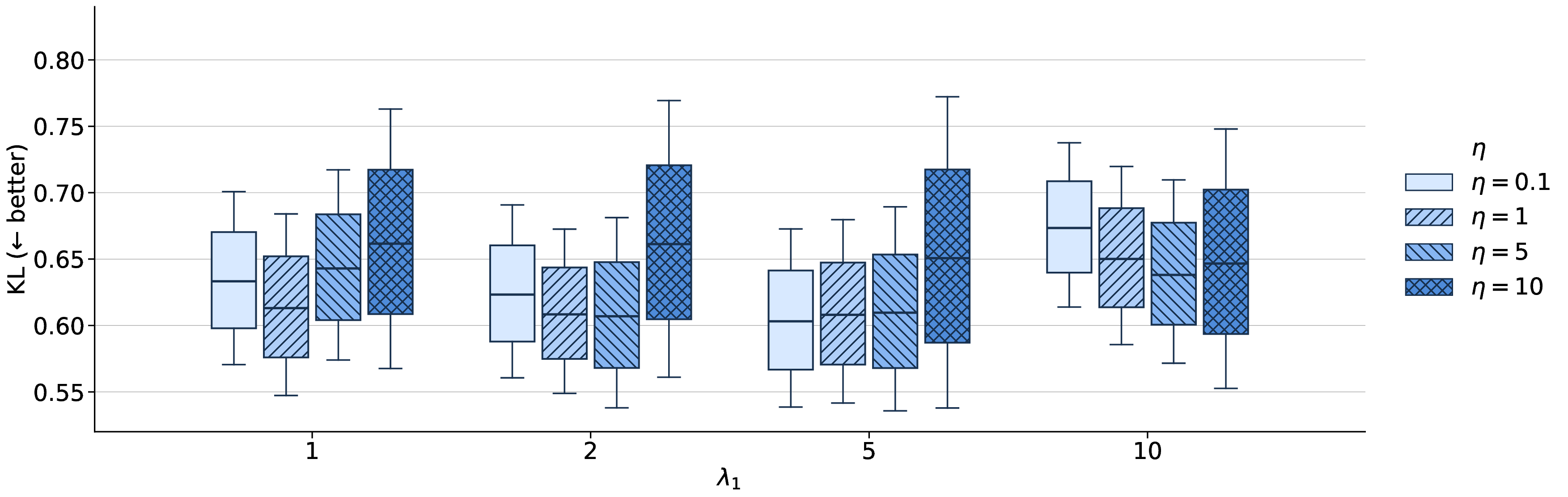}
\caption{Distributions of the distortion level (KL) under various hyperparameter settings for DualGA, evaluated on LFQA using Llama-3-8B and with $\Delta = 0.4$. Each box plot represents the KL distribution across different configurations of the initial dual variable $\lambda_1$ and the step size $\eta$. Boxes show the median (line) and the 25th-75th quantiles range; whiskers indicate the 10th–90th percentiles. The algorithm exhibits stable performance for reasonable hyperparameter choices, although for extreme values (e.g., $\eta = 10$ or $\lambda_1 = 10$), the resulting KL divergence remains noticeably larger than the optimal values.
}\label{fig:box_4}
\end{figure}

\section{Possibility for Other Variants of Watermarking}
\begin{algorithm}[ht!]
\caption{Dual Gradient Ascent on Autocovariance}
    \label{alg:DualGA_autocov}
    \begin{algorithmic}[1] 
    \Require Original LM $\bm{p}$, prompt $x$, pseudorandom function $f_{\gamma}$ to determine the green list with ratio $\gamma$, random seed $\mathcal{K}$, DG constraint $\Delta$, step size $\eta$, initial dual variable $\lambda_1$
    \Ensure $y_1, \dots, y_{T_{\text{max}}}$
    \State Initialize $t\gets1$, $y_0 \gets \mathcal{K}$, color $y_0$ with \emph{red}
    \While{$t \leq T_{\text{max}}$ \& $y_{t-1} \neq v_{\text{term}}$}
        \State Decide the green list via $f(y_{t-1})$ and observe the green list probability $G_t$ and the red list probability $R_t$
        \[G_t \gets \sum_{k \text{ green at }t} p_{t,k},\quad R_t \gets 1-G_t\]
        \State Set $\delta_t$ so that the Lagrangian is maximized
        \[\delta_t \gets \lambda_t\]
        \State Compute the logit vector generated by the original LM $\bm{p}$. Based on the color of $y_{t-1}$:
        \If{$y_{t-1}$ is \emph{red}}
            \State Produce probability vector $\bm{q}^{\delta_t}_{t}$ by
            \[
            q^{\delta_t}_{t,k} \gets 
            \frac{\exp(l_{t,k} + \delta_t \cdot \mathbbm{1}\{k \text{ green at }t\})}{\sum_{k^\prime=1}^{|\mathcal{V}|} \exp(l_{t,k^\prime} + \delta_t \cdot \mathbbm{1}\{k^\prime \text{ green at }t\})}
            \]
          \Else
             \State Produce probability vector $\bm{q}^{\delta_t}_{t}$ by
            \[
            q^{\delta_t}_{t,k} \gets 
            \frac{\exp(l_{t,k} + \delta_t \cdot \mathbbm{1}\{k \text{ red at }t\})}{\sum_{k^\prime=1}^{|\mathcal{V}|} \exp(l_{t,k^\prime} + \delta_t \cdot \mathbbm{1}\{k^\prime \text{ red at }t\})}
            \]
          \EndIf
          
        \State Sample out a token and record its color
        \[y_t \sim \bm{q}^{\delta_t}_{t}\]
        \State Compute the gradient of the dual function
        \[\mathrm{gd}_t \gets \Delta + \mathrm{DA}_t(\delta_t)\]
        \State Update the dual variable via gradient ascent
        \[\lambda_{t+1} \gets \lambda_{t} + \eta \cdot \mathrm{gd}_t\]
        \State $t\gets t+1$
    \EndWhile
    \If{$t<T_{\text{max}}$}
        \State Set all remaining $y_{t+1}, \dots, y_{T_{\text{max}}}$ to be $v_{\text{term}}$
    \EndIf
    \end{algorithmic}
\end{algorithm}

In this section, we introduce a variant for the design of watermark algorithms other than increasing the green list probability only. We hope this section serves as an inspiration for further developments of watermarking schemes. The red-green list scheme guarantees that under the null hypothesis, variables $N_t \coloneqq \mathbbm{1}\{y_t \in f_{\gamma}(y_{t-1})\}-\gamma$'s are i.i.d. and mean-zero random variables. In other words, a violation of this i.i.d. assumption in any sense can serve as a watermarking scheme. Specifically, for any statistic $S_T=S(\{N_t\}_{t=1}^{T})$ and a corresponding watermarking paradigm such that for watermarked texts the distribution of $S_T$ deviates from the distribution of $S(\{N_t\}_{t=1}^{T})$ under the null hypothesis, we can also derive a reasonable watermarking algorithm. In both SRL (Algorithm \ref{alg:kirchen}) and DualGA (Algorithm \ref{alg:DualGA}), we construct the statistic by $S_T=\sum_{t=1}^T N_t$. Other variants of watermarking algorithms exploiting different statistics could also be developed, which may serve well in certain circumstances.

To give an example, consider using the sample autocovariance as the score statistic
$$
S_T = \dfrac{1}{T}\sum_{t=1}^{T-1}(N_t-\overline{N})(N_{t+1}-\overline{N}), \quad \overline{N}\coloneqq \dfrac{1}{T}\sum_{t=1}^T N_t,
$$
which, by Theorem 7.2.1 in \citet{BrockwellTimSeries}, has the following asymptotic property
$$
\lim_{T\to\infty}\dfrac{\sqrt{T}\cdot S_T}{\gamma (1-\gamma)} \stackrel{\text{d}}{\Longrightarrow} \mathcal{N}(0, 1).
$$
Define $Z_T\coloneqq \dfrac{\sqrt{T}\cdot S_T}{\gamma (1-\gamma)}$. The above property says that under the null hypothesis (where $N_t$'s are i.i.d. and mean-zero), the statistic $Z_T$ asymptotically has a standard normal distribution. In light of this, a watermarking algorithm is designed to induce the largest deviation of $Z_T$ from a standard normal distribution, at the least cost of model distortion (measured in KL-divergence). One possible approach is to greedily minimize the stepwise autocovariance. Once we observe that $y_t$ is realized into the green list, we add $\delta_{t+1}$ to the red list logits of the next token and encourage the model to pick a different color; on the contrary, once we observe that $y_t$ is realized as red, we add $\delta_{t+1}$ to the green list logits for the next token.

For the test statistic, we define $S_T$ as $S_T \coloneqq \dfrac{1}{T}\sum_{t=1}^{T-1}N_t\cdot N_{t+1}$, which satisfies the asymptotically normal property under the null hypothesis. When applying watermarks, our objective is to minimize the expectation of $S_T$, which can be decomposed into stepwise minimizations of \emph{difference of the autocovariance} $\mathrm{DA}_t$, defined by
$$
\begin{aligned}
    -\mathrm{DA}_t(\bm{q}_t, y_{t-1}) \coloneqq & \gamma\cdot \mathbbm{1}\{y_{t-1} = \text{red}\}\cdot (\sum_{k = \text{green},  k \in \mathcal{V}} q_{t,k} - \sum_{k = \text{green},  k \in \mathcal{V}} p_{t,k}) \\
    & + (1-\gamma)\cdot \mathbbm{1}\{y_{t-1} = \text{green}\}\cdot (\sum_{k = \text{red},  k \in \mathcal{V}} q_{t,k} - \sum_{k = \text{red},  k \in \mathcal{V}} p_{t,k}).
\end{aligned}
$$
The new algorithm, DualGA-AutoCov is presented as Algorithm \ref{alg:DualGA_autocov}. The algorithm is evaluated based on its detection ability and distortion level on the C4 dataset. The comparison of DualGA-AutoCov between DualGA and other benchmarks (Figure \ref{fig:ACF_combine} (left)) shows that DualGA-AutoCov has a similar trade-off pattern to SRL. Although the average $p$-value of DualGA-AutoCov is larger than the original DualGA at the same distortion extent (when KL$=0.5$), it performs consistently better than EMS under different configurations. In the right $3$ figures of Figure \ref{fig:ACF_combine}, the robustness of DualGA-AutoCov is tested under the three common attacks given in Section \ref{subapx:robustness_attack}. The new DualGA-AutoCov remains detectable at a significance level of $10^{-4}$ when the attack rate is smaller than $0.3$. This observation validates the reasonable robustness of this new formulation under mild attacks.

\begin{figure}[ht!]
     \centering
         \includegraphics[width=\textwidth]{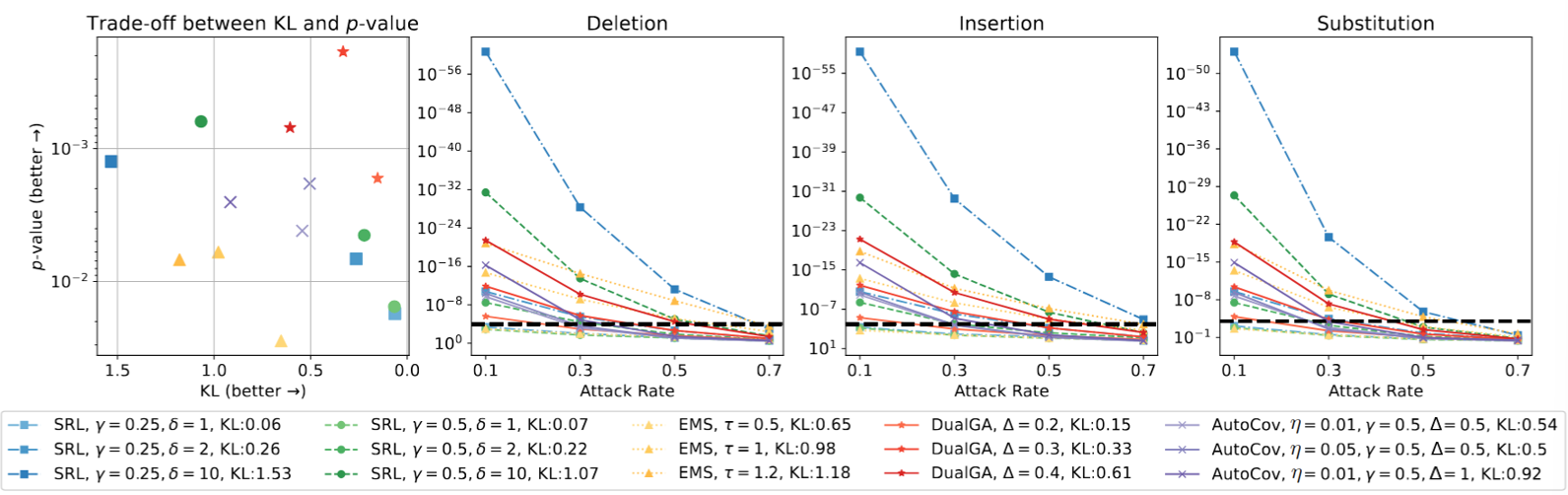}
        \caption{Evaluation of DualGA-AutoCov algorithm (abbreviated as AutoCov in the legends) on the C4 dataset. The leftmost figure gives the population-level distortion (KL) and detection ability ($p$-value) of DualGA-AutoCov. At the same distortion extent, a smaller $p$-value marks a better detection ability. The three figures on the right show
        the robustness of different watermarking algorithms under the three types of attack. The attack rate denotes the proportion of the attacked tokens and the black dashed line represents a median $p$-value of $10^{-4}$ across the generated samples. Together the above figures validate that DualGA-AutoCov is no worse than SRL in terms of the detection-distortion trade-off, and is reasonably robust under mild attacks. 
        }
        \label{fig:ACF_combine}
\end{figure}

\end{document}